\documentclass[12pt]{article}
\usepackage{fullpage}
\usepackage[round, numbers]{natbib}

\usepackage{type1cm}        
\usepackage{makeidx}         
\usepackage{graphicx}        
\usepackage{multicol}        
\usepackage[bottom]{footmisc}

\usepackage{newtxtext}       %
\usepackage[varvw]{newtxmath}       
\usepackage{algorithm,algorithmic}

\usepackage{amsthm}
\newtheorem{theorem}{Theorem}
\newtheorem{lemma}[theorem]{Lemma}
\usepackage{hyperref}
\setcitestyle{square}
\usepackage{booktabs}
\usepackage{multirow}
\usepackage{makecell}

\usepackage[skip=1ex]{caption}
\usepackage[skip=0ex]{subcaption}

\makeindex             

\DeclareMathOperator{\sign}{sign}
\DeclareMathOperator*{\argmin}{arg\,min}
\DeclareMathOperator*{\argmax}{\arg\max}
\newcommand{\const}{\mathrm{Const.}}
\newcommand{\bx}{\boldsymbol{x}}
\newcommand{\bw}{\boldsymbol{w}}
\newcommand{\bc}{\boldsymbol{c}}
\newcommand{\bl}{\boldsymbol{l}}
\newcommand{\bI}{\boldsymbol{I}}
\newcommand{\bW}{\boldsymbol{W}}
\newcommand{\by}{\boldsymbol{y}}

\newcommand{\balpha}{\boldsymbol{\alpha}}
\newcommand{\bbeta}{\boldsymbol{\beta}}
\newcommand{\bphi}{\boldsymbol{\phi}}
\newcommand{\bPhi}{\boldsymbol{\Phi}}
\newcommand{\bpsi}{\boldsymbol{\psi}}
\newcommand{\bPsi}{\boldsymbol{\Psi}}

\newcommand{\tr}{\mathrm{tr}}
\newcommand{\te}{\mathrm{te}}

\newcommand{\bR}{\mathbb{R}}
\newcommand{\bE}{\mathbb{E}}
\newcommand{\cX}{\mathcal{X}}
\newcommand{\cY}{\mathcal{Y}}
\newcommand{\cy}{\tilde{y}}

\newcommand{\cD}{\mathcal{D}}
\newcommand{\cF}{\mathcal{F}}
\newcommand{\cG}{\mathcal{G}}

\newcommand{\fR}{\mathfrak{R}}

\newcommand{\nparams}{b}

\newcommand{\T}{{\hspace{-0.25ex}\top\hspace{-0.25ex}}}

\newcommand{\dif}{\mathrm{d}}
\newcommand{\cA}{\mathcal{A}}

\newcommand{\cM}{\mathcal{M}}
\newcommand{\cH}{\mathcal{H}}
\newcommand{\cS}{\mathcal{S}}
\newcommand{\cZ}{\mathcal{Z}}
\newcommand{\cL}{\mathcal{L}}

\newcommand{\bz}{\boldsymbol{z}}
\newcommand{\bK}{\boldsymbol{K}}
\newcommand{\bk}{\boldsymbol{k}}

\newcommand{\mf}{\boldsymbol{f}_\theta}

\newcommand{\Str}{\cS^\tr}

\newcommand{\Ste}{\cS^\te}
\newcommand{\Ztr}{\cZ^\tr}

\newcommand{\Zte}{\cZ^\te}
\newcommand{\Ltr}{\cL^\tr}

\newcommand{\Lte}{\cL^\te}

\makeatletter
\newcommand{\specificthanks}[1]{\@fnsymbol{#1}}
\makeatother

\begin{document}

\title{Rethinking Importance Weighting for Transfer Learning}



\author{\normalsize
Nan Lu\thanks{The University of Tokyo, email: lu@ms.k.u-tokyo.ac.jp}
\and\normalsize Tianyi Zhang\thanks{The University of Tokyo / RIKEN, email: zhang@ms.k.u-tokyo.ac.jp}
\and\normalsize Tongtong Fang\thanks{The University of Tokyo, email: fang@ms.k.u-tokyo.ac.jp}
\and\normalsize Takeshi Teshima\thanks{The University of Tokyo / RIKEN, email: teshima@ms.k.u-tokyo.ac.jp}
\and\normalsize Masashi Sugiyama\thanks{RIKEN/ The University of Tokyo, email: sugi@k.u-tokyo.ac.jp}
}
\date{}

%
%
\maketitle

\abstract{
A key assumption in supervised learning is that training and
test data follow the \emph{same} probability distribution.
However, this fundamental assumption is not always satisfied in
practice, e.g., 
due to changing environments, sample selection bias, privacy concerns, or high labeling costs.
\emph{Transfer learning} (TL) relaxes this assumption and allows us to
learn under distribution shift.
Classical TL methods typically rely on \emph{importance-weighting}---a
predictor is trained based on the training losses weighted according to
the importance (i.e., the test-over-training density ratio).
However, as real-world machine learning tasks are becoming
increasingly complex, high-dimensional, and dynamical, novel
approaches are explored to cope with such challenges recently.
In this article, after introducing the foundation of TL based on
importance-weighting, we review recent advances based on \emph{joint} and
\emph{dynamic} importance-predictor estimation.
Furthermore, we introduce a method of \emph{causal mechanism transfer}
that incorporates causal structure in TL.
Finally, we discuss future perspectives of TL research.
}

\section{Introduction}
\label{sec:intro}
Supervised learning has been successfully applied to a wide variety of fields \cite{goodfellow2016deep}.
The vast majority of supervised learning methods follow the canonical framework of \emph{empirical risk minimization}~(ERM) \cite{wahba1990spline,bishop1995neural,Vapnik1998,scholkopf2001learning,hastie2009elements,duda2012pattern,sugiyama2015introduction}, assuming that the samples used for training and the samples used for testing follow the \emph{same} probability distribution.
However, this assumption may not be fulfilled in many real-world scenarios.
For example, due to changing environments, autonomous cars trained in good weather also need to work well in bad weather \cite{yu2020bdd100k,koh2021wilds};
due to sample selection bias, the training set may be gender or race imbalanced and the test set is balanced \cite{berk1983introduction,zadrozny2004learning};
and due to privacy concerns or high labeling costs, the labels of the training data may be corrupted or noisy versions of the ground truth \cite{JMLR:v18:16-315,han2020survey}.
Such distribution mismatches may significantly degrade the prediction performance of models trained by standard supervised learning \cite{quionero2009dataset,sugiyama2012machine}.
\clearpage

Fortunately, \emph{transfer learning} provides us useful tools to learn under differing distributions.
Classical transfer learning methods mainly rely on \emph{importance weighting} \cite{cochran2007sampling,fishman2013monte,kahn1953methods}, which handles the distribution mismatch in two steps \cite{shimodaira2000improving,zadrozny2004learning,sugiyama2005input,sugiyama2007covariate,huang2007correcting,sugiyama2008direct,sugiyama2008direct1,kanamori2009least,yamada2011relative}:
\begin{itemize}
\item \emph{importance estimation}: estimate the ratio between test and training densities, a.k.a. the \emph{importance};
\item \emph{importance-weighted ERM}: train a predictive model by weighting the training losses according to the importance in the ERM framework.
\end{itemize}
These classical methods work well, as if there is no distribution mismatch, if the form of data is simple, e.g., some linear model suffices \cite{sugiyama2012density}.
However, nowadays, as the data become increasingly complex, high-dimensional, and dynamical, new challenges arise in transfer learning.
\begin{enumerate}
    \item First, given very complex and high-dimensional data, estimating the test-over-training density ratio becomes significantly difficult.
    However, for the classical two-step methods, the error occured in the first importance estimation step directly propagates to the second importance-weighted training step, which degrades the performance of the learned predictor.
    Can we avoid this problem and use importance weighting to solve transfer learning problems in an end-to-end fashion?
    \item 
    Second, existing transfer learning researches mainly rely on certain assumptions on distribution shifts \cite{sugiyama2012machine,zhang2013domain}.
    However, the type of distribution shift can be unknown or multiple type of distribution shifts can present together.
    In practice, data from different domains may share a common data-generating mechanism, 
    e.g., the source and target distributions are induced by a common causal mechanism.
    Can we leverage such causal mechanisms to tackle the transfer learning problem?
\end{enumerate}

In this article, 
we answer these questions affirmatively and provide our solutions.
In Section~\ref{sec:formulate}, we formulate the ordinary supervised learning and transfer learning problems.
In Section~\ref{sec:iw} and Section~\ref{sec:aiw}, we review classical two-step importance weighting based methods for transfer learning.
In Section~\ref{sec:ubm}, we answer the first question by proposing a one-step \emph{joint} approach which integrates the importance estimation step and the importance-weighted ERM step by minimizing an upper bound of the test risk \cite{zhang2020one,zhang2021one}.
In Section~\ref{sec:diw}, we further explore the first question by proposing an end-to-end \emph{dynamic} approach that iterates between importance estimation and importance-weighted ERM and combines them in a seamless manner.
In Section~\ref{sec:causal}, we answer the second question by proposing a causal mechanism transfer approach that incorporates the causal structure in transfer learning.
We conclude this article and discuss future works in Section~\ref{sec:concl}.
\section{Problem Formulation}
\label{sec:formulate}
In this section, we introduce the problem setups for ordinary supervised learning and transfer learning.

\subsection{Ordinary Supervised Learning}\label{sec:formulate:osl}
Let us consider the \emph{supervised learning} problem of estimating an unknown input-output relationship from training samples.
Let $\cX\subset\bR^{d_\mathrm{in}}$ be the input feature space, and $\cY$ be the output label space where
$\cY\subset\bR$ for regression,
$\cY\coloneqq\{+1,-1\}$ for binary classification,
and $\cY\coloneqq[k]$ for multi-class classification.
Here, $d_\mathrm{in}$ denotes the input dimension,
$k$ $(\ge 3)$ denotes the number of classes,
and $[k]\coloneqq\{1,2,\ldots,k\}$.
Let $\cD_\tr\coloneqq\left\{\left(\bx_i^\tr,  y_i^\tr\right)\right\}_{i=1}^{n_\tr}$ be the training samples in the source domain, where $n_\tr$ is the training sample size, $\bx_i^\tr \in \cX$ and $y_i^\tr \in \cY$ are the training input and output drawn independently from a source domain density $p_\tr(\bx,y)$.
Any $p_\tr(\bx,y)$ can be decomposed in two ways, using either
\begin{itemize}
    \item the class conditional density $p_\tr(\bx\mid y)$ and the class prior probability $p_\tr(y)$, i.e., $p_\tr(\bx,y)=p_\tr(\bx\mid y)p_\tr(y)$;
    \item the marginal density $p_\tr(\bx)$ and the class posterior probability $p_\tr(y\mid\bx)$, i.e., $p_\tr(\bx,y)=p_\tr(\bx)p_\tr(y\mid \bx)$.
\end{itemize}
Let $(\bx^\te, y^\te)$ be a test sample in the target domain, where $\bx^\te \in \cX$ and $y^\te \in \cY$ are the test input and output drawn from a target domain density $p_\te(\bx,y)$.
Note that the test sample is not given 
in the training phase, but will be given in the test phase in the future.

The goal of supervised learning is to learn a predictor 
$f\colon\cX\rightarrow\bR^{d_\mathrm{out}}$
that minimizes the expected test error, also known as the \emph{risk}:
\begin{align}
    \label{eq: risk}
    R(f) \coloneqq \bE_{(\bx^\te, y^\te)\sim p_\te(\bx, y)}\left[\ell(f(\bx^\te), y^\te)\right].
\end{align}
Here, $d_\mathrm{out}=1$ for regression and binary classification,
and $d_\mathrm{out}=k$ for multi-class classification.
$\bE$ denotes the expectation, and 
$\ell$ denotes a real-valued \emph{loss function} that measures the discrepancy between the true output value $y$ and its predicted value $f(\bx)$.
Typically, the predicted output is given by 
\begin{align*}
    \widehat{y}=
    \begin{cases}
    f(\bx)& \mathrm{for~regression},\\
    \sign(f(\bx))& \mathrm{for~binary~classification},\\
    \argmax_{j\in [k]}(f(\bx))_j& \mathrm{for~multi\textit{-}class~classification},
    \end{cases}
\end{align*}
where $(f(\bx))_j$ is the $j$-th element of $f(\bx)$.
For regression problems, $\ell$ in \eqref{eq: risk} is often chosen as the \emph{squared loss}
\begin{align*}
    \ell_\mathrm{s}(f(\bx), y)=(\widehat{y}-y)^2
\end{align*}
for training and evaluation purposes.
For classification problems, $\ell$ in \eqref{eq: risk} is often chosen as the \emph{zero-one loss}
\begin{align*}
    \ell_\mathrm{01}(f(\bx), y)=I(\widehat{y}\neq y)
\end{align*}
for evaluation, where $I$ is the indicator function.
However, since $\ell_\mathrm{01}$ 
is discontinuous and therefore difficult to optimize, we often replace it with some \emph{surrogate loss} \cite{bendavid06jcss,bartlett06jasa}.
For example, $\ell_\mathrm{01}$ is replaced by the 
\emph{hinge loss} 
\begin{align*}
    \ell_{\mathrm{h}}(f(\bx), y)=\operatorname{max}(0, 1-yf(\bx))
\end{align*}
for binary classification,
and is replaced by the \emph{softmax cross-entropy loss} 
\begin{align*}
    \ell_{\mathrm{ce}}(f(\boldsymbol{x}), y)&=-\sum\nolimits_{j=1}^k I(y=j)\log\left(\frac{\exp((f(\boldsymbol{x}))_j)}{\sum_{i=1}^k \exp((f(\boldsymbol{x}))_i)}\right)\\
    &=\log\left(\sum\nolimits_{i=1}^k\exp((f(\boldsymbol{x}))_i)\right)-(f(\boldsymbol{x}))_y
\end{align*}
for multi-class classification.

As the target domain density $p_\te(\bx, y)$ in \eqref{eq: risk} remains unknown, in supervised learning theories, 
it is commonly assumed that the training samples and test samples follow the same probability density, i.e., $p_\tr(\bx,y)=p_\te(\bx,y)$ \cite{wahba1990spline,bishop1995neural,Vapnik1998,scholkopf2001learning,hastie2009elements,duda2012pattern,sugiyama2015introduction}.
With this assumption, \emph{empirical risk minimization}~(ERM) \cite{Vapnik1998} is a common practice in supervised learning that learns a predictor $f$ with training samples through the following procedures:
\begin{enumerate}
    \item Choose a loss function $\ell(\cdot)$, so that the risk $R(f)$ is defined;
    \item Choose a model $\cF$, so that our goal is $\min_{f\in\cF}R(f)$;
    \item Approximate $R(f)$ by replacing the expectation by the sample average (under $p_\tr(\bx,y)=p_\te(\bx,y)$):
    \begin{align}
        \label{eq:empirical-risk}%
        \widehat{R}(f) = \frac{1}{n_\tr}\sum_{i=1}^{n_\tr}\left[\ell(f(\bx_i^\tr), y_i^\tr)\right];
    \end{align}
    \item Minimize the empirical risk $\widehat{R}(f)$, with appropriate regularization, by a favorite optimization algorithm:
    \begin{align}
    \label{eq:empirical-risk-minimizer}%
    \widehat{f}=\argmin_{f\in\cF}\widehat{R}(f).
    \end{align}
\end{enumerate}
ERM defines a family of learning algorithms and thanks to its mathematically clean formulation, it enables us to conduct theoretical analysis on the prediction performance.
Indeed, the empirical risk estimator \eqref{eq:empirical-risk} is \emph{unbiased} and the empirical risk minimizer \eqref{eq:empirical-risk-minimizer} is \emph{consistent}, i.e., $\widehat{f}$ in \eqref{eq:empirical-risk-minimizer} will converge to the optimal predictor $f^*=\argmin_{f\in\cF}R(f)$ as $n_\tr\to\infty$ \cite{Vapnik1998}.

\subsection{Transfer Learning}
In many real-world problems, the training samples in the source domain and the test samples in the target domain are drawn from different densities, i.e., $p_\tr(\bx,y)\neq p_\te(\bx,y)$.
In that case, ERM is not generally consistent anymore.
Learning under such differing distributions is called \emph{transfer learning} or \emph{domain adaptation} \cite{pan2009survey,quionero2009dataset,sugiyama2012machine,yang2020transfer}.

Due to various assumptions on how the joint distribution shifts from the source domain to the target domain, several transfer learning scenarios have been considered in the literature:
\begin{itemize}
    \item \emph{Covariate shift} is a traditional learning setting, which assumes that $p_\tr(\bx)\neq p_\te(\bx)$
    and $p_\tr(y\mid\bx)=p_\te(y\mid\bx)$ \cite{shimodaira2000improving,huang2007correcting,sugiyama2007covariate,sugiyama2012machine,zhang2020one,zhang2021one}.
    Even though $p(y\mid\bx)$ does not change, $p_\tr(\bx)\neq p_\te(\bx)$ may break the consistency of ERM which occurs if the model $\cF$ is not expressive enough, so that fitting the training data does not fit the test data \cite{shimodaira2000improving}.
    Covariate shift is a primal target in Sections~\ref{sec:iw},~\ref{sec:aiw},~and \ref{sec:ubm}.
    
    \item \emph{Output noise} is a common learning setting, where $p_\tr(y\mid\bx)\neq p_\te(y\mid\bx)$ and $p_\tr(\bx)=p_\te(\bx)$ which is opposite to covariate shift \cite{angluin1988learning,han2020survey}.
    In this scenario, we consider a label corruption process $p(\cy\mid y,\bx)$ where $\cy$ denotes the corrupted label so that $p_\tr(\cy\mid\bx)=\sum_y p(\cy\mid y,\bx)\cdot p_\te(y\mid\bx)$, i.e., a label $y$ may flip to every corrupted label $\cy\neq y$ with probability $p(\cy\mid y,\bx)$.
    Such output noise is extremely detrimental to ERM training, since an over-parameterized $f$ is able to fit any training data even with random labels \cite{zhang2016understanding}.
    
    \item \emph{Class-prior/target shift} is also a frequently encountered setting in practice, where $p_\tr(y)\neq p_\te(y)$ while $p_\tr(\bx\mid y)=p_\te(\bx\mid y)$ \cite{japkowicz2002class,he2009learning,zhang2013domain,huang2016learning,buda2018systematic,lipton2018detecting}.
    Under this shift, $f$ will emphasize over-represented classes and neglect under-represented classes, which may raise fairness issues \cite{cao2019learning}.
    
    \item \emph{Class-conditional shift} refers to the setting that $p_\tr(\bx\mid y)\neq p_\te(\bx\mid y)$ and $p_\tr(y)=p_\te(y)$ \cite{zhang2013domain}.
    Since the estimation of $p_\te(\bx\mid y)$ in this setting is in general ill-posed, a further assumption that there exists a transformation $\tau$ such that the conditional densities agree, i.e., $p_\tr(\tau(\bx)\mid y)=p_\te(\tau(\bx)\mid y)$, is often needed to solve the class-conditional shift problems \cite{gong2016domain,yu2020label}.
    
    
    \item \emph{Full-distribution shift} is the most general setting where we consider $p_\tr(\bx,y)\neq p_\te(\bx,y)$ without any further information.
    The problem is extremely challenging and the use of validation data from the target domain is essential \cite{fang2020rethinking}.
    In principle, importance weighting is also applicable if validation data from the target domain is available.
    We discuss the details in Section~\ref{sec:diw}.
    \item \emph{Independent-component shift} also considers $p_\tr(\bx,y)\neq p_\te(\bx,y)$ but with an additional prior knowledge about the underlying generative mechanism. In particular, we consider the case where the distribution shift is attributed to a latent distribution shift of the \emph{independent components} (ICs) and the mixing function to generate the data from the ICs is identical across the domains. We discuss this case in Section~\ref{sec:causal}.
\end{itemize}
\section{Importance Weighting for Covariate Shift Adaptation}
\label{sec:iw}
In this section, we briefly introduce the \textit{importance weighting} for covariate shift adaptation, 
including importance-weighted ERM and some representative approaches for direct importance estimation.

\subsection{Importance-Weighted ERM}
\label{subsec:iwerm}
As discussed in Section~\ref{sec:formulate}, covariate shift may cause the failure of ERM. 
Importance weighting \cite{quionero2009dataset, sugiyama2012machine, shimodaira2000improving, sugiyama2007covariate} is a standard method for mitigating this problem. 
The essential idea of importance weighting under covariate shift is as follows:
\begin{align*}
    \bE_{(\bx^\te, y^\te)\sim p_\te(\bx, y)}\left[\ell(f(\bx^\te), y^\te)\right] = \bE_{(\bx^\tr, y^\tr)\sim p_\tr(\bx, y)}\left[w(\bx^\tr)\ell(f(\bx^\tr), y^\tr)\right],
\end{align*}
where $w(\bx) = \frac{p_\te(\bx)}{p_\tr(\bx)}$ is referred to as the \textit{importance}. This implies the expectation for any loss function $\ell(\cdot)$ over the test distribution is the importance-weighted expectation of the loss over the training distribution. This is how covariate shift is mitigated by importance weighting. 

\textit{Importance-weighted ERM} (IWERM) is accordingly defined as
\begin{align}
    \label{eq: weighted_erm}
    \min_{f\in\cF}\frac{1}{n_\tr}\sum_{i=1}^{n_\tr}\left[w(\bx_i^\tr)\ell(f(\bx_i^\tr), y_i^\tr)\right].
\end{align} 
For any fixed $f\in\cF$, the importance weighted empirical risk is an unbiased estimator of the test risk and its minimizer is consistent \cite{shimodaira2000improving}.

\subsection{Direct Importance Estimation}
\label{subsec:w_estimation}
In the importance-weighted learning objective \eqref{eq: weighted_erm}, the importance remains unknown.
Next, we review approaches for directly estimating the importance $w(\bx)$.
We assume that test input data $\{\bx_j^\te\}_{j=1}^{n_{\te}}$ drawn independently from $p_\te(\bx)$ are available for importance estimation, in addition to the training input data $\{\bx_i^\tr\}_{i=1}^{n_{\tr}}$.

\subsubsection{Kernel Mean Matching (KMM)}
\label{subsubsec: kmm}
\textit{Kernel mean matching}~(KMM) \cite{huang2007correcting} aims to find $w(\bx)$ that minimizes the mean discrepancy of nonlinear transformations of data drawn from $p_\tr$ and $p_\te$ in a \textit{reproducing kernel Hilbert space} (RKHS), that is, for an RKHS $\cH$,
\begin{equation*}
    \begin{split}
    &\min_{w(\bx)}\left\|\bE_{\bx^\te\sim p_\te(\bx)}\left[K(\bx^\te, \cdot)\right]-\bE_{\bx^\tr\sim p_\tr(\bx)}\left[w(\bx^\tr)K(\bx^\tr, \cdot)\right]\right \|_\cH ^{2}\\
    &\mathrm{subject~to~} \bE_{\bx^\tr\sim p_\tr(\bx)}\left[w(\bx^\tr)\right]=1 \mathrm{~and~} w(\bx)\geq 0 \mathrm{~for~all~} \bx\in\mathcal{X},
    \end{split}
\end{equation*}
where $\left\| \cdot \right\|_\cH$ denotes the norm in $\cH$ and $K(\bx, \bx')$ is a kernel function \cite{10.5555/345662}.
An empirical version of the above optimization problem can be written as
\begin{equation*}
\begin{split}
    &\min_{\left\{ w_i\right\}_{i=1}^{n_\tr}}\left[\frac{1}{2}\sum_{i, i'=1}^{n_\tr}w_iw_i'K(\bx_i^\tr, \bx_{i'}^\tr)-\sum_{i=1}^{n_\tr}w_i\kappa_i \right]\\
    &\mathrm{~subject~to~} \frac{1}{n_\tr}\left|\sum_{i=1}^{n_\tr}w_i-n_\tr\right|\leq\epsilon \mathrm{~and~} 0\leq w_1, \ldots, w_{n_\tr}\leq B, 
\end{split}
\end{equation*}
where $\kappa_i=\frac{n_\tr}{n_\te}\sum_{j=1}^{n_\te}K(\bx_i^\tr, \bx_j^\te)$, and $B \geq 0$ and $\epsilon \geq 0$ are parameters for controlling the strength of regularization.

\subsubsection{Least-Squares Importance Fitting (LSIF)}
\label{subsubsec: lsif}
\emph{Least-squares importance fitting} (LSIF) \cite{kanamori2009least} adopts the squared loss for importance fitting. We model the importance $w(\bx)$ using a \emph{linear-in-parameter model}:
\begin{align*}
  g(\bx)=\sum\nolimits_{l=1}^\nparams\beta_l \psi_l(\bx)=\bbeta^\top\bpsi(\bx),
\end{align*}
where $\bbeta=(\beta_1,\ldots,\beta_\nparams)^\top\in\bR^\nparams$ are the parameters to be learned,
\begin{align*}
\bpsi(\bx)=(\psi_1(\bx),\ldots,\psi_\nparams(\bx))^\top\in\bR^\nparams
\end{align*}
is a vector of \emph{basis functions}, and $\nparams$ denotes the number of basis functions.
As basis functions, Gaussian kernels are practically useful \cite{10.5555/345662},
where $\psi_l(\bx):=\exp(-\frac{\left\|\bx-\bx'\right\|^2}{2\sigma^2})$ for some $\bx'\in\bR^{d_\mathrm{in}}$ and $\sigma>0$.
The parameters $\bbeta$ in $g(\bx)$ are determined by minimizing the following squared-error $J$:
\begin{align}
\label{squared-error-obj}
    J(\bbeta) & =\frac{1}{2}\bE_{\bx^\tr\sim p_\tr(\bx)}\left[(g(\bx^\tr)-w(\bx^\tr))^2\right] \notag\\
    & = \frac{1}{2}\bE_{\bx^\tr\sim p_\tr(\bx)}\left[g^2(\bx^\tr)\right] - \bE_{\bx^\tr\sim p_\tr(\bx)}\left[g(\bx^\tr)w(\bx^\tr)\right] + \frac{1}{2}\bE_{\bx^\tr\sim p_\tr(\bx)}\left[w^2(\bx^\tr)\right], \notag\\
    & = \underbrace{\frac{1}{2}\bE_{\bx^\tr\sim p_\tr(\bx)}\left[g^2(\bx^\tr)\right] - \bE_{\bx^\te\sim p_\te(\bx)}\left[g(\bx^\te)\right]}_{:=J'(\bbeta)} + \underbrace{\frac{1}{2}\bE_{\bx^\tr\sim p_\tr(\bx)}\left[w^2(\bx^\tr)\right]}_{:=C}.
\end{align}
The last term $C$ is a constant with respect to $\bbeta$ and thus it can be safely ignored.
Approximating the expectations in $J'(\bbeta)$, we have
\begin{align*}
    \widehat{J'}(\bbeta) & = \frac{1}{2n_\tr}\sum\nolimits_{i=1}^{n_\tr}g^2(\bx_i^\tr)-\frac{1}{n_\te}\sum\nolimits_{j=1}^{n_\te}g(\bx_j^\te) \\
    & = \frac{1}{2}\sum\nolimits_{l,l'=1}^\nparams\beta_l\beta_{l'}\left(\frac{1}{n_\tr}\sum\nolimits_{i=1}^{n_\tr}\psi_{l}(\bx_i^\tr)\psi_{l'}(\bx_i^\tr)\right)
    - \sum_{l=1}^\nparams\beta_l\left(\frac{1}{n_\te}\sum_{j=1}^{n_\te}\psi_{l}(\bx_j^\te)\right) \\
    & = \frac{1}{2}\bbeta^{\T}\widehat{\boldsymbol{H}}\mathbb{\bbeta}-\widehat{\boldsymbol{h}}^{\T}\bbeta,
\end{align*}
where the $(l, l')$-th element of $\widehat{\boldsymbol{H}}$ is $\frac{1}{n_\tr}\sum_{i=1}^{n_\tr}\psi_{l}(\bx_i^\tr)\psi_{l'}(\bx_i^\tr)$ and $l$-th element of $\widehat{\boldsymbol{h}}$ is $\frac{1}{n_\te}\sum_{j=1}^{n_\te}\psi_{l}(\bx_j^\te)$.

Since the importance $w(\bx)$ is non-negative by definition, we may practically solve the following optimization problem:
\begin{align}
    \label{eq: lsif}
    \min_\beta \left[ \frac{1}{2}\bbeta^{\T}\widehat{\boldsymbol{H}}\bbeta-\widehat{\boldsymbol{h}}^{\T}\bbeta+\lambda \textbf{1}_\nparams^\T\bbeta \right]
    \mathrm{~subject~to~} \bbeta\geq\textbf{0}_\nparams,
\end{align} 
where $\textbf{1}_\nparams$ and $\textbf{0}_\nparams$ are the $\nparams$-dimensional vectors with all ones and zeros. $\bbeta\geq\textbf{0}_\nparams$ is applied in the element-wise manner. The term $\lambda \textbf{1}_\nparams^\T\bbeta$ is for regularization and a non-negative $\lambda$ is the parameter controlling the regularization effect. Note that \eqref{eq: lsif} is a convex quadratic programming problem and the unique global optimal solution can be obtained by using some standard optimization packages. 

\emph{Unconstrained LSIF} (uLSIF) \cite{kanamori2009least} is an approximation method for LSIF by removing the non-negativity constraint in \eqref{eq: lsif}. This leads to the following unconstrained optimization problem:
\begin{align*}
        \min_{\beta \in \bR^{\nparams}} \left[ \frac{1}{2}\bbeta^{\T}\widehat{\boldsymbol{H}}\bbeta-\widehat{\boldsymbol{h}}^{\T}\bbeta+\frac{\lambda}{2} \bbeta^\T\bbeta \right]. 
\end{align*}
Here $\frac{\lambda}{2} \bbeta^\T\bbeta$ is introduced for regularization rather than $\lambda\textbf{1}_{\nparams}^\T\bbeta$ because the linear term does not work as a regularizer without the nonnegativity constraint. The solution can be analytically computed as
$
    \tilde{\bbeta} = (\widehat{\boldsymbol{H}}+\lambda\boldsymbol{I}_{\nparams})^{-1}\widehat{\boldsymbol{h}},
$
where $\boldsymbol{I}_{\nparams}$ is the ${\nparams}$-dimensional identity matrix.

To compensate for the approximation error that some of the learned $\bbeta$ can be negative, the solution is modified as
$
 \widehat{\bbeta} = \max(\textbf{0}_{\nparams}, \tilde{\bbeta}),
$
where the $\max$ operation is applied in the element-wise manner for a pair of vectors. 

\section{Adaptive Importance Weighting for Covariate Shift Adaptation
}
\label{sec:aiw}
Although IWERM gives a consistent predictor as shown in Section~\ref{sec:iw}, it can also produce an unstable predictor due to high variance of the importance weights for the training samples, which indicates that IWERM may not be the best method in the finite-sample case \cite{shimodaira2000improving}. 
Therefore, in practice, it is preferable to slightly stabilize IWERM during training. 
In this section, we introduce two such adaptive methods for covariate shift adaptation.

\subsection{Exponentially-flattened IWERM (EIWERM)}
\label{subsec:eiwerm}
The first method flattening the importance weights in IWERM is \emph{exponentially-flattened IWERM} (EIWERM) \cite{shimodaira2000improving}:
\begin{align*}
    \min_{f\in\cF}\frac{1}{n_\tr}\sum_{i=1}^{n_\tr}\left[(w(\bx_i^\tr))^\gamma\ell(f(\bx_i^\tr), y_i^\tr)\right],
\end{align*}
where $\gamma$ $(0\le\gamma\le1)$ is called the \emph{flattening parameter}.

The flattening parameter controls stability and consistency of the predictor: $\gamma=0$ corresponds to ordinary ERM (the uniform weight, which yields a stable but inconsistent predictor), and $\gamma=1$ corresponds to IWERM (the importance weight, which yields a consistent but unstable predictor). An intermediate value of $\gamma$ would provide the optimal control of the trade-off between stability and consistency.

\subsection{Relative IWERM (RIWERM)}
\label{riwerm}
A potential drawback of EIWERM when the importance $w(\bx)$ is replaced with an estimate $\widehat{w}(\bx)$ --- its exponent $(\widehat{w}(\bx))^\gamma$ is not necessarily a good estimate of the original exponent $(w(\bx))^\gamma$. Indeed, estimation of $(w(\bx))^\gamma$ may be hard/easy if $\gamma$ is large/small, since the target function $(w(\bx))^\gamma$  is more/less fluctuated.
 To cope with this problem, \emph{relative IWERM} (RIWERM) \cite{yamada2011relative} was proposed which learned a flattened importance weight directly. More specifically, it estimates the $\eta$-relative importance weight $w_\eta(\bx)$ $(0\le\eta\le 1)$ defined as
\begin{align*}
    w_{\eta}(\bx) \coloneqq \frac{p_\te(\bx)}{\eta p_\te(\bx) + (1-\eta)p_\tr(\bx)}.
\end{align*}
Consequently, RIWERM minimizes the following weighted loss in the same way as uLSIF:
\begin{align*}
    \min_{f\in\cF}\frac{1}{n_\tr}\sum_{i=1}^{n_\tr}\left[w_{\eta}(\bx_i^\tr)\ell(f(\bx_i^\tr), y_i^\tr)\right].
\end{align*}

Similarly to the role of flattening parameter $\gamma$ in EIWERM, the parameter $\eta$ also controls stability and consistency of the predictor: $\eta=0$ corresponds to IWERM and $\eta=1$ corresponds to the ordinary ERM.

To estimate the relative importance $w_{\eta}(\bx)$, \emph{relative uLSIF} (RuLSIF) was proposed and it learns a relative importance model $g_{\eta}(\bx)$ by minimizing the following expected squared-error $J_\eta$:
\begin{align*}
    J_\eta(g_\eta) & =\frac{1}{2}\bE_{\eta p_\te(\bx) + (1-\eta)p_\tr(\bx)}\left[(g_\eta(\bx)-w_\eta(\bx))^2\right] \\
    & = \frac{\eta}{2}\bE_{p_\te(\bx)}\left[g^2_\eta(\bx)\right] + \frac{1-\eta}{2}\bE_{p_\tr(\bx)}\left[g^2_\eta(\bx)\right] - \bE_{p_\te(\bx)}\left[(g_\eta(\bx)\right] + \text{Const}.
\end{align*}
The analytical solution minimizing the empirical version of the objective $J_\eta$ can be obtained similarly to uLSIF (see Section~\ref{subsubsec: lsif}) when the linear-in-parameter model is used and the details can be found in \cite{yamada2011relative}.
\newcommand{\tabincell}[2]{\begin{tabular}{@{}#1@{}}#2\end{tabular}}

\section{Upper Bound Minimization for Covariate Shift Adaptation}
\label{sec:ubm}
As we have seen in Sections~\ref{sec:iw} and~\ref{sec:aiw}, covariate shift adaptation methods usually need an intermediate step for estimating the importance. A natural question is whether we can have a more direct way to solve the covariate shift adaptation problem.
In this section, we answer it affirmatively by introducing
a one-step approach for covariate shift adaptation \cite{zhang2020one,zhang2021one} which integrates the importance estimation step and the importance-weighted ERM step by minimizing an upper bound of the test risk.


\subsection{One-step Approach}
\label{subsec:osa}
First, we derive an upper bound of the test risk \eqref{eq: risk}, which is the key of the one-step approach.

\begin{theorem}\cite{zhang2020one}%
\label{thm: upper bound}
~
Let $w(\bx)$ be the importance $p_\te(\bx)/p_\tr(\bx)$ and $\cF\subseteq \{f\colon\cX\to\bR\}$ be a given hypothesis set.
Suppose that there is a constant $m \in \bR$ such that $\ell(f(\bx), y)\leq m$ holds for every $f\in\cF$ and every $(\bx, y)\in\cX\times\cY$.
Then, for any $f\in\cF$ and any measurable function $g\colon\cX\to\bR$, the test risk is bounded as follows under covariate shift:
\begin{align}
    \label{eq: upper bound1}
    \frac{1}{2} R^2(f) \le J(f, g) &\coloneqq \left(\bE_{(\bx^\tr, y^\tr)\sim p_\tr(\bx, y)}\left[g(\bx^\tr)\ell(f(\bx^\tr), y^\tr)\right]\right)^2\notag\\
    &\phantom{\coloneqq}\ + m^2\bE_{\bx^\tr\sim p_\tr(\bx)}\left[\left(g(\bx^\tr)-w(\bx^\tr)\right)^2\right].
\end{align}
Furthermore, if $g$ is non-negative and $\ell_\mathrm{UB}$ bounds $\ell$ from above,
we have
\begin{align}
    \label{eq: upper bound2}
    J(f, g) \le J_\mathrm{UB}(f, g)
    &\coloneqq \left(\bE_{(\bx^\tr, y^\tr)\sim p_\tr(\bx, y)}\left[g(\bx^\tr)\ell_\mathrm{UB}(f(\bx^\tr), y^\tr)\right]\right)^2\notag\\
    &\phantom{\coloneqq}\ + m^2\bE_{\bx^\tr\sim p_\tr(\bx)}\left[\left(g(\bx^\tr)-w(\bx^\tr)\right)^2\right].
\end{align}
\end{theorem}

\begin{proof}
According to the Cauchy-Schwarz inequality, we have
\begin{align*}
    \frac{1}{2}R^2(f)
    &=\frac{1}{2}\left(\bE_{p_\tr(\bx, y)}\left[w(\bx^\tr)\ell(f(\bx^\tr), y^\tr)\right]\right)^2\\
    &\leq\left(\bE_{p_\tr(\bx, y)}\left[g(\bx^\tr)\ell(f(\bx^\tr), y^\tr)\right]\right)^2\\
    &\phantom{\leq}\ 
    +\left(\bE_{p_\tr(\bx, y)}\left[\left(w(\bx^\tr)-g(\bx^\tr)\right)\ell(f(\bx^\tr), y^\tr)\right]\right)^2\\
    &\leq\left(\bE_{p_\tr(\bx, y)}\left[g(\bx^\tr)\ell(f(\bx^\tr), y^\tr)\right]\right)^2\\
    &\phantom{\leq}\ 
    +\bE_{p_\tr(\bx)}\left[\left(g(\bx^\tr)-w(\bx^\tr)\right)^2\right]\bE_{p_\tr(\bx, y)}\left[\ell^2(f(\bx^\tr), y^\tr)\right]\\
    &\leq\left(\bE_{p_\tr(\bx, y)}\left[g(\bx^\tr)\ell(f(\bx^\tr), y^\tr)\right]\right)^2
    +m^2\bE_{p_\tr(\bx)}\left[\left(g(\bx^\tr)-w(\bx^\tr)\right)^2\right],
\end{align*}
This proves \eqref{eq: upper bound1}, and based on this, \eqref{eq: upper bound2} is obvious.
\end{proof}

For classification problems, the test risk $R(f)$ is typically defined by the zero-one loss, and thus the boundedness assumption of the loss function in Theorem~\ref{thm: upper bound} holds with $m=1$.
For regression problems, the squared loss is a typical choice, but it violates the boundedness assumption.
Instead, we define $R(f)$ using Tukey's bisquare loss \cite{beaton1974fitting} (see Fig.~\ref{fig:tukey}).

\begin{figure*}[bt]
  \centering
  \includegraphics[scale=0.5]{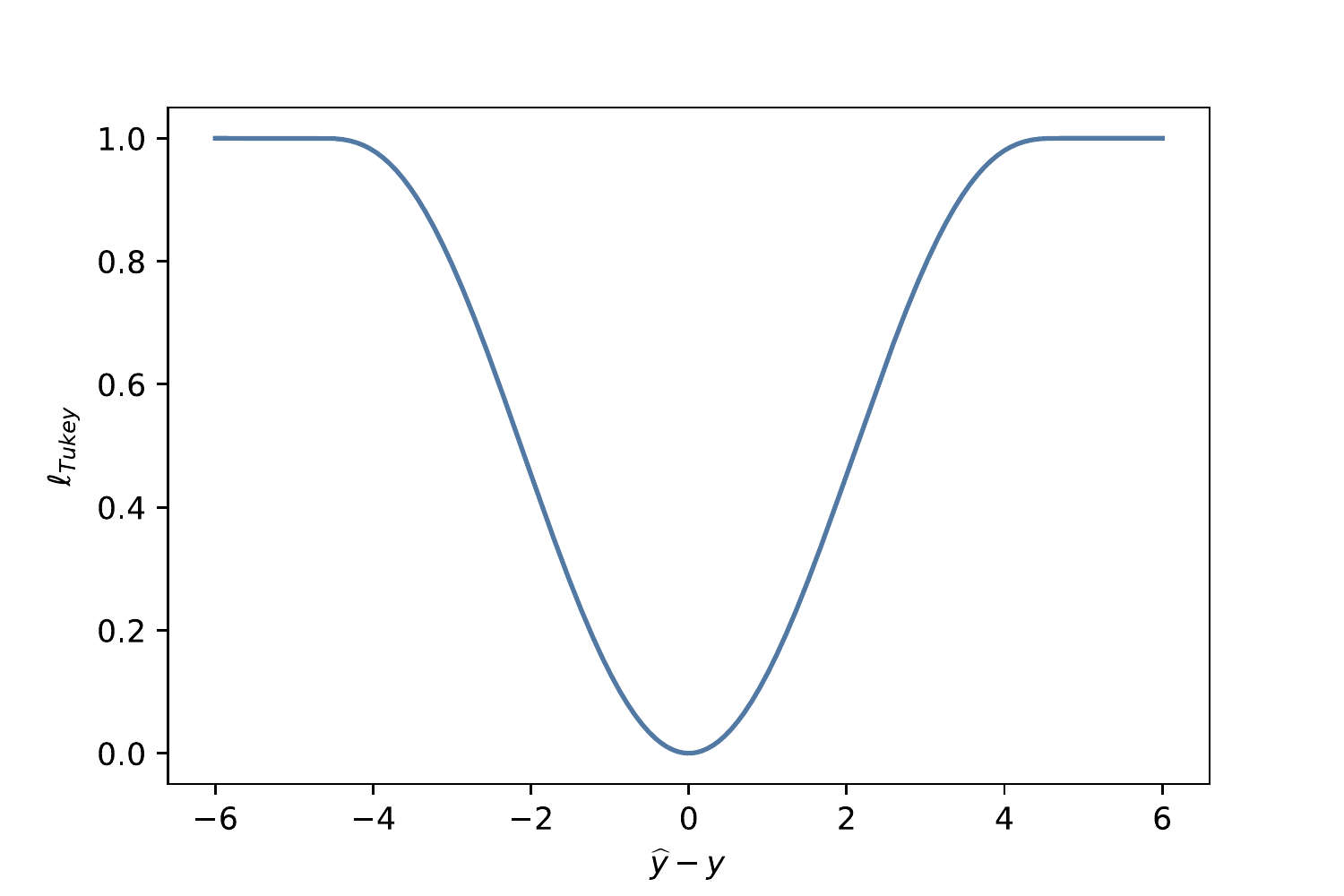}
  \caption{Tukey's loss defined as 
  $\ell_{\rm Tukey}(\widehat{y}, y) \coloneqq \min\left(1-\left[1-(\widehat{y}-y)^2/\rho^2\right]^{3}, 1\right) \le 1$.
  It has been widely used in the context of robust statistics.
  The hyper-parameter $\rho > 0$ is usually set to $4.685$ for this loss function,
  and it provides an asymptotic efficiency $95\%$ of that of least squares for Gaussian noise \cite{andersen2008modern}.
  Here, we rescale the standard Tukey's bisquare loss for convenience, which does not change the minimizer of the test risk.}
  \label{fig:tukey}
\end{figure*}

~\\
\noindent
\emph{Remark:} The two-step approach that first applies uLSIF to estimate the importance weights and then employs IWERM with the estimated importance is equivalent to minimizing the second term of the above upper bounds first and then minimizing the first term, which leads to a sub-optimal solution to the upper-bound minimization. 
~\\

Instead of estimating the unknown $w(\bx)$ for minimizing $R(f)$ as in the previous two-step approaches,
we propose a one-step approach that minimizes the upper bound $J(f, g)$ or $J_\mathrm{UB}(f, g)$ based on Theorem~\ref{thm: upper bound}.

For classification problems, $J(f, g)$ is defined using the zero-one loss, with which training will not be tractable \cite{bendavid06jcss}.
Fortunately, the latter part of Theorem~\ref{thm: upper bound} allows us to minimize $J_\mathrm{UB}(f, g)$ instead,
with $\ell_\mathrm{UB}$ being any (sub-)differentiable approximation that bounds the zero-one loss from above
so that we can apply any gradient method
such as stochastic gradient descent \cite{robbins1951stochastic}.
Examples of such $\ell_\mathrm{UB}$ include the hinge loss and the squared loss in binary classification and the softmax cross-entropy loss in multi-class classification.
For regression problems, Tukey's loss is already sub-differentiable,
but we can use the squared loss that bounds Tukey's loss which makes the optimization problem simpler as described later.
This is again justified by Theorem~\ref{thm: upper bound} with the squared loss used as the upper-bounding loss $\ell_\mathrm{UB}$.

Although the second expectation in $J_\mathrm{UB}(f, g)$ contains an unknown term $w(\bx)$, it can be estimated from the samples on hand, up to addition by a constant $C$ as shown in~\eqref{squared-error-obj}.
Since the true distributions are unknown, we minimize its empirical version $\widehat{J}_\mathrm{UB}(f,g)$ with respect to $f$ and non-negative $g$ in some given hypothesis sets $\cF$ and $\cG_+$:
\begin{align}
    \label{eq: empirical}
    \widehat{J}_\mathrm{UB}(f, g; S)
    &\coloneqq \left(\frac{1}{n_\tr}\sum_{i=1}^{n_\tr}g(\bx^\tr_i)\ell_\mathrm{UB}(f(\bx^\tr_i), y^\tr_i)\right)^2\notag\\
    &\phantom{\coloneqq} +m^2\left(\frac{1}{n_\tr}\sum_{i=1}^{n_\tr}g^2(\bx^\tr_i)-\frac{2}{n_\te}\sum_{j=1}^{n_\te}g(\bx^\te_j)+C\right),
\end{align}
where $S\coloneqq\left\{\left(\bx_i^\tr, y_i^\tr\right)\right\}_{i=1}^{n_\tr}\cup\big\{\bx_j^\te\big\}_{j=1}^{n_\te}$ is the set of sample points.
Notice that constant $C$ can be safely ignored in the minimization.

In Algorithm~\ref{alg:alternating}, we present an alternating minimization algorithm that can be employed when $f(\bx)$ and $g(\bx)$ are linear-in-parameter models used in Section~\ref{subsubsec: lsif}, i.e.,
\begin{equation}
    \label{linear model}
    f(\bx)=\balpha^\top \bphi(\bx)
    \quad \text{and}\quad 
    g(\bx)=\bbeta^\top \bpsi(\bx),
\end{equation}
where $\balpha\in\bR^{b_f}$ and $\bbeta\in\bR^{b_g}$ are parameters, and $\bphi$ and $\bpsi$ are $b_f$-dimensional and $b_g$-dimensional vectors of basis functions.

First, we minimize the objective~\eqref{eq: empirical} with respect to $g$ while fixing $f$.
This step has an analytic solution as shown in Algorithm~\ref{alg:alternating}, Line~6,
where $\bPhi_\tr=(\bphi(\bx^\tr_1),\ldots,\bphi(\bx^\tr_{n_\tr}))^\top$,
$\bPsi_\tr=(\bpsi(\bx^\tr_1),\ldots,\bpsi(\bx^\tr_{n_\tr}))^\top$,
$\bPsi_\te = (\bpsi(\bx^\te_1),\ldots,\bpsi(\bx^\te_{n_\te}))^\top$,
$\boldsymbol{1}=(1,\ldots,1)^\top$,
and $\bI$ is the identity matrix.

Next, we minimize the objective~\eqref{eq: empirical} with respect to $f$ while fixing $g$.
In this step, we can safely ignore the second term and remove the square operation of the first term in the objective~\eqref{eq: empirical}
to reduce the problem into weighted empirical risk minimization (cf.\@ Algorithm~\ref{alg:alternating}, Line~12)
by forcing $g$ to be non-negative with a rounding up technique \cite{kanamori2009least} as shown in Algorithm~\ref{alg:alternating}, Line~7.
For regression problems, the method of iteratively reweighted least squares (IRLS) \cite{beaton1974fitting} can be used for optimizing Tukey's bisquare loss.
In practice, we suggest using the squared loss as a convex approximation of Tukey's loss to obtain a closed-form solution as shown in Algorithm~\ref{alg:alternating}, Line~10 for reducing computation time, and we compare their performance in the experiments.
For classification with linear-in-parameter models using the hinge loss, then the weighted support vector machine \cite{yang2007weighted} can be used.
After this step, we go back to the step for updating $g$ and repeat the procedure.

\begin{algorithm}[t]
\caption{Alternating Minimization with Linear-in-parameter Models}
\label{alg:alternating}
\begin{algorithmic}[1]
\linespread{1.1}\selectfont
\STATE $\balpha_0\gets$ an arbitrary $b_f$-dimensional vector
\STATE $\lambda\gets$ a positive $\ell_2$-regularization parameter
\STATE $\mu\gets$ a positive $\ell_2$-regularization parameter
\FOR{$t=0, 1, \ldots, T-1$}
\STATE $\bl_t\gets(\ell_\mathrm{UB}(\balpha_t^\top\bphi(\bx^\tr_1),y^\tr_1),\ldots,\ell_\mathrm{UB}(\balpha_t^\top\bphi(\bx^\tr_{n_\tr}), y^\tr_{n_\tr}))^\top$
\STATE $\bbeta_{t+1}\leftarrow\left(\frac{1}{n_\tr}\bPsi_\tr^\top\bPsi_\tr+\frac{1}{m^2 n_\tr^2}\bPsi_\tr^\top\bl_t\bl_t^\top\bPsi_\tr+\frac{1}{m^2}\lambda\bI\right)^{-1}\frac{1}{n_\te}\bPsi_\te^\top\boldsymbol{1}$
\STATE $\bbeta_{t+1}\leftarrow\max(\bbeta_{t+1},\boldsymbol{0})$\
\STATE $w_i^{t+1} \gets \bbeta_{t+1}^\top\bpsi(\bx_i^\tr)$, $i=1,\ldots,n_\tr$
\IF{$\ell_\mathrm{UB}$ is the squared loss}
  \STATE $\balpha_{t+1}\leftarrow\left(\bPhi_\tr^\top\bW_{t+1}\bPhi_\tr+\mu n_\tr\bI\right)^{-1}\bPhi_\tr^\top\bW_{t+1}\by_\tr$,\\
  where $\bW_{t+1}=\operatorname{diag}(w_1^{t+1},\ldots,w_{n_\tr}^{t+1})$ and $\by_\tr=(y_1^\tr,\ldots,y_{n_\tr}^\tr)^\top$
\ELSE
  \STATE $\balpha_{t+1}\leftarrow\argmin_{\balpha}\frac{1}{n_\tr}\sum_{i=1}^{n_\tr}w_i^{t+1}\ell_\mathrm{UB}(\balpha_t^\top\bphi(\bx^\tr_i),y^\tr_i)+\mu\balpha^\top\balpha$
\ENDIF
\ENDFOR
\end{algorithmic}
\end{algorithm}

On the other hand, when $f(\bx)$ and $g(\bx)$ are modeled as neural networks, we present a gradient-based alternating minimization algorithm described in Algorithm~\ref{alg:gradient} that is more convenient for training neural networks.

\begin{algorithm}[t]
\caption{Gradient-based Alternating Minimization}
\label{alg:gradient}
\begin{algorithmic}[1]
\linespread{1.1}\selectfont
\STATE $\cD^{\tr}, \cD_{\mathrm{in}}^{\te} \gets \left\{\left(\bx_i^\tr, y_i^\tr\right)\right\}_{i=1}^{n_\tr}, \big\{\bx_j^\te\big\}_{j=1}^{n_\te}$
\STATE $\cA\gets$ a gradient-based optimizer
\STATE $f\gets$ an arbitrary classifier
\FOR{$\mathrm{round}=0, 1, \ldots, \mathrm{numOfRounds}-1$}
\FOR{$\mathrm{epoch}=0, 1, \ldots, \mathrm{numOfEpochsForG}-1$}
\FOR{$k=0, 1, \ldots, \mathrm{numOfMiniBatches}-1$}
\STATE $\cD_k^{\tr}, \cD_{\mathrm{in},k}^{\te}\gets\mathrm{sampleMiniBatch}(\cD^{\tr}, \cD_{\mathrm{in}}^{\te})$
\STATE $g\gets\cA(g, \nabla_g\widehat{J}_\mathrm{UB}(f, g; \cD_k^{\tr}\cup\cD_{\mathrm{in},k}^{\te}))$
\ENDFOR
\ENDFOR
\FOR{$\mathrm{epoch}=0, 1, \ldots, \mathrm{numOfEpochsForF}-1$}
\FOR{$k=0, 1, \ldots, \mathrm{numOfMiniBatches}-1$}
\STATE $\cD_k^{\tr}\gets\mathrm{sampleMiniBatch}(\cD^{\tr})$
\STATE $w_i\gets\max(g(\bx_i), 0)$, $\forall(\bx_i, \cdot)\in\cD_k^{\tr}$
\STATE $w_i\gets w_i/\sum_j w_j$, $\forall i$
\STATE $L_k\gets\sum_{(\bx_i, y_i)\in\cD_k^{\tr}}w_i\ell_\mathrm{UB}(f(\bx_i), y_i)$
\STATE $f\gets\cA(f, \nabla_{f}L_k)$
\ENDFOR
\ENDFOR
\ENDFOR
\end{algorithmic}
\end{algorithm}

\subsection{Theoretical Analysis}
\label{subsec:osa-ta}
In what follows, we establish a generalization error bound for the proposed method in terms of the \emph{Rademacher complexity} \cite{koltchinskii_rademacher_2001} in regression and binary classification. Due to the space limitation, we omit the proofs here --- they can be found in \cite{zhang2020one,zhang2021one}.

\begin{lemma}
\label{lemma: uniform bound}
Assume that
(a) there exist some constants $M\ge m$ and $L>0$ such that $\ell_\mathrm{UB}(f(\bx), y)\leq M$ holds for every $f\in\cF$ and every $(\bx, y)\in\cX\times\cY$ and $y\mapsto\ell_\mathrm{UB}(y, y')$ is $L$-Lipschitz for every fixed $y'\in\cY$;\footnote{This assumption is valid when $\sup _{f \in \cF}\|f\|_{\infty}$ and $\sup_{y\in\cY}|y|$ are bounded.}
(b) there exists some constant $G\geq 1$ such that $g(\bx)\leq G$ for every $g\in\cG_+$ and every $\bx\in\cX$. Let $\cG=\cG_+\cup-\cG_+$. 
Then for any $\delta>0$, with probability at least $1-\delta$ over the draw of $S$, the following holds for all $f\in\cF, g\in\cG_+$ uniformly:
\begin{align}
    \label{eq: uniform bound}
    J_\mathrm{UB}(f, g) \leq &\widehat{J}_\mathrm{UB}(f, g; S) + 8MG\left(M+G\right)\left(L\fR^\tr_{n_\tr}(\cF)+\fR^\tr_{n_\tr}(\cG)\right)\notag\\
    &+ 4M^2\fR^\te_{n_\te}(\cG) + 5M^2G^2\sqrt{\frac{\log\frac{1}{\delta}}{2}}\left(\frac{1}{\sqrt{n_\tr}}+\frac{1}{\sqrt{n_\te}}\right),
\end{align}
where $\fR^\tr_{n_\tr}(\cF)$ and $\fR^\tr_{n_\tr}(\cG)$ are the Rademacher complexities of $\cF$ and $\cG$, respectively, for the sampling of size $n_\tr$ from $p_\tr(\bx)$, and $\fR^\te_{n_\te}(\cG)$ is the Rademacher complexity of $\cG$ for the sampling of size $n_\te$ from $p_\te(\bx)$.
\end{lemma}

Combining \eqref{eq: upper bound1}, \eqref{eq: upper bound2}, and \eqref{eq: uniform bound}, we obtain the following theorem.

\begin{theorem}
\label{thm: uniform generalization bound}
Suppose that the assumptions in Lemma~\ref{lemma: uniform bound} hold.
Then, for any $\delta>0$, with probability at least $1-\delta$ over the draw of $S$, the test risk can be bounded as follows for all $f\in\cF$ uniformly:
\begin{align*}
    \frac{1}{2}R^2(f) \leq &\min_{g\in\cG_+}\widehat{J}_\mathrm{UB}(f, g; S) + 8MG\left(M+G\right)\left(L\fR^\tr_{n_\tr}(\cF)+\fR^\tr_{n_\tr}(\cG)\right)\notag\\
    &+ 4M^2\fR^\te_{n_\te}(\cG) + 5M^2G^2\sqrt{\frac{\log\frac{1}{\delta}}{2}}\left(\frac{1}{\sqrt{n_\tr}}+\frac{1}{\sqrt{n_\te}}\right).
\end{align*}
\end{theorem}

Theorem~\ref{thm: uniform generalization bound} implies that minimizing $\widehat{J}_\mathrm{UB}(f, g)$, as the proposed method does, amounts to minimizing an upper bound of the test risk.
Furthermore, the following theorem shows a generalization error bound for the minimizer obtained by the proposed method.

\begin{theorem}
\label{thm: generalization error bound}
Let $(\widehat{f}, \widehat{g}) = \argmin_{(f,g)\in\cF\times\cG_+}\widehat{J}_\mathrm{UB}(f, g)$.
Then, under the assumptions of Lemma~\ref{lemma: uniform bound}, for any $\delta>0$, it holds with probability at least $1-\delta$ over the draw of $S$ that
\begin{align*}
    \frac{1}{2}R^2(\widehat{f}) \leq &\min_{f\in\cF, g\in\cG_+}J_\mathrm{UB}(f, g) + 8MG\left(M+G\right)\left(L\fR^\tr_{n_\tr}(\cF)+\fR^\tr_{n_\tr}(\cG)\right) \notag\\
    &+ 4M^2\fR^\te_{n_\te}(\cG)+ 10M^2G^2\sqrt{\frac{\log\frac{1}{\delta}}{2}}\left(\frac{1}{\sqrt{n_\tr}}+\frac{1}{\sqrt{n_\te}}\right) + \frac{M^2G^2}{n_\tr}. \label{eq: gen error bound}
\end{align*}
\end{theorem}

If we use linear-in-parameter models with bounded norms, then $\fR^\tr_{n_\tr}(\cF)=O(1/\sqrt{n_\tr})$, $\fR^\tr_{n_\tr}(\cG)=O(1/\sqrt{n_\tr})$, and $\fR^\te_{n_\te}(\cG)=O(1/\sqrt{n_\te})$ \cite{mohri2018foundations,shalev2014understanding}. Furthermore, if we assume that the approximation error of $\cG_+$ is zero, i.e., $r\in\cG_+$, then $\min_{f\in\cF, g\in\cG+}J_\mathrm{UB}(f, g)\leq J_\mathrm{UB}(f^*,r)=R_\mathrm{UB}^2(f^*)$, where $R_\mathrm{UB}$ is the test risk defined with $\ell_\mathrm{UB}$ and $f^*=\argmin_{f\in\cF}R_\mathrm{UB}(f)$.
Thus,
\begin{equation*}
    R(\widehat{f})\leq\sqrt{2}R_\mathrm{UB}(f^*)+O_p(1/\sqrt[4]{n_\tr}+1/\sqrt[4]{n_\te}).
\end{equation*}
When the best-in-class test risk $R_\mathrm{UB}(f^*)$ is small, this bound would theoretically guarantee a good performance of the proposed method.

\subsection{Experiments}
\label{subsec:osa-exp}
In this section, we examine the effectiveness of the proposed method via a toy illustration and experiments on regression and classification benchmark datasets.

\begin{figure*}[t]
  \centering
  \includegraphics[scale=0.5]{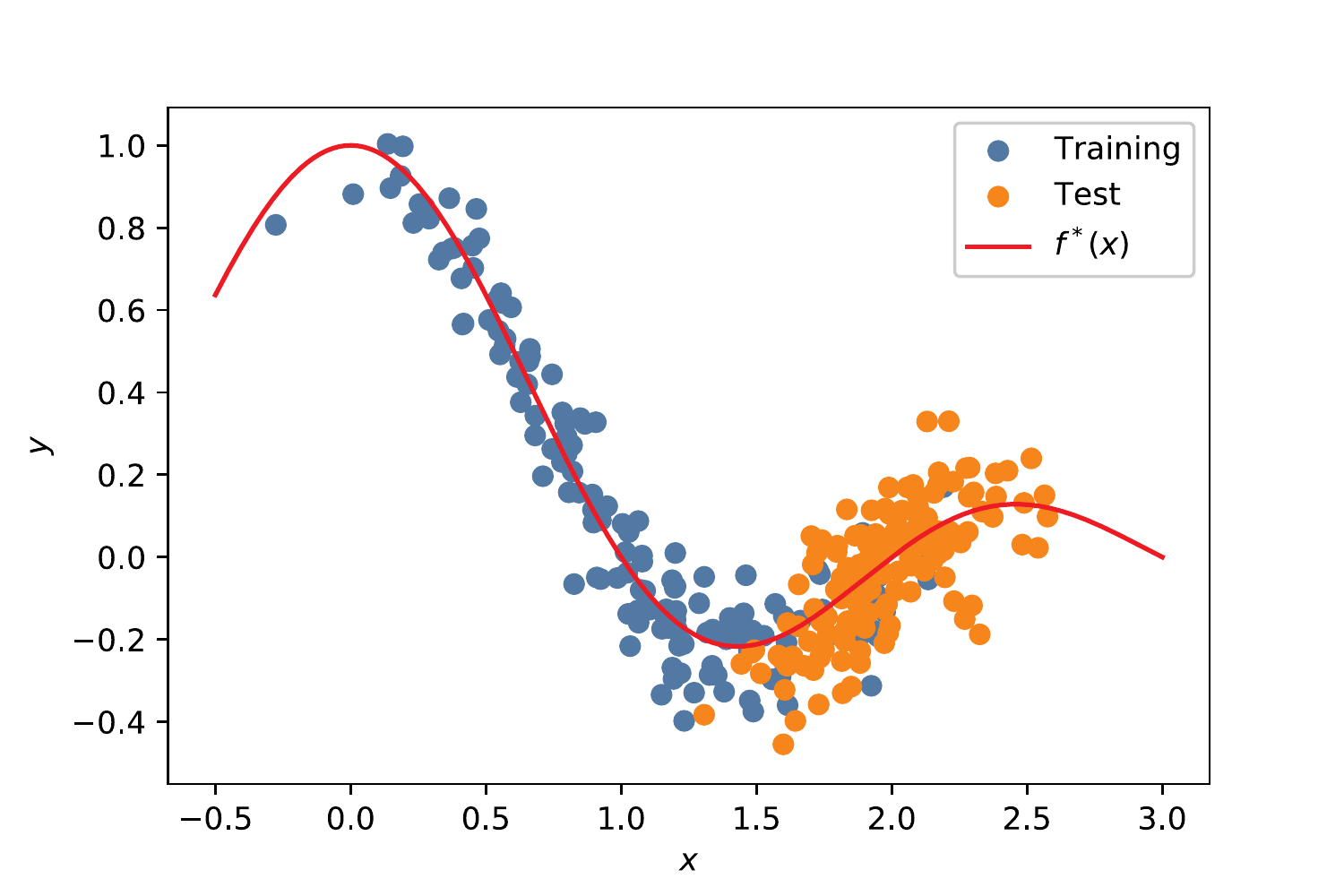}
  \caption{A toy regression example.
  The training input points (blue) are distributed on the left-hand side of the input domain
  and the test input points (orange) are distributed on the right-hand side.
  The two distributions share the same regression function $f^*$ (the red curve).}
  \label{fig:toy}
\end{figure*}

\subsubsection{Illustration with Toy Regression Datasets}
First, we report the results of experiments on a toy regression dataset. 

Let us consider a one-dimensional regression problem.
Let the training and test input densities be 
\[p_{\mathrm{tr}}(x)=N(x ; 1,(0.5)^{2}) \text{\quad and \quad  } p_{\mathrm{te}}(x)=N(x ; 2,(0.25)^{2}),\]
where $N(x ; \mu, \sigma^{2})$ denotes the Gaussian density with mean $\mu$ and variance $\sigma^2$.
Consider the case where the output labels of examples are generated by
\begin{equation*}
    y=f^*(x)+\epsilon \quad \text{with} \quad f^*(x)=\operatorname{sinc}(x),
\end{equation*}
and the noise $\epsilon$ following $N\left(0,(0.1)^{2}\right)$ is independent of $x$.
As illustrated in Fig.~\ref{fig:toy}, the training input points are distributed on the left-hand side of the input domain and the test input points are distributed on the right-hand side.
We sampled $n_{\tr}=150$ labeled i.i.d.\@ training samples $\left\{\left(x_i^\tr,  y_i^\tr\right)\right\}_{i=1}^{n_\tr}$ with each $x_i^\tr$ following $p_\tr(x)$ and $n_{\te}=150$ unlabeled i.i.d.\@ test samples $\{x_j^\te\}_{j=1}^{n_\te}$ following $p_\te(x)$ for learning the target function $f^*(x)$ in the experiment.

We compared our one-step approach with three baseline methods, which are the ordinary ERM, EIWERM with uLSIF, and RIWERM.
We used the linear-in-parameter models~\eqref{linear model} with the Gaussian kernels
(introduced in Section~\ref{subsubsec: lsif}) as basis functions for learning the input-output relation and the importance (or the relative importance) in all the experiments:
\[\phi_i(\bx)=\exp\bigg\{-\frac{\|\bx-\bc^f_i\|^2}{2\sigma_f^2}\bigg\}\quad\text{and}\quad \psi_i(\bx)=\exp\bigg\{-\frac{\|\bx-\bc^g_i\|^2}{2\sigma_g^2}\bigg\},\]
where $\sigma_f$ and $\sigma_g$ are the bandwidths of the Gaussian kernels, and $\bc_i^f$ and $\bc_i^g$ are the kernel centers randomly chosen from $\{\bx_j^\te\}_{j=1}^{n_\te}$ \cite{sugiyama2008direct,kanamori2009least}.
We set $b_f=b_g=50$ in all the experiments.

As suggested in Section~\ref{subsec:osa}, we used the squared loss in the one-step approach for computational efficiency. 
We also employed the IRLS algorithm for optimizing Tukey's bisquare loss in the one-step approach.
For better comparison, we reported the results of the baseline methods using both the squared loss and Tukey's bisquare loss.

The experimental results of the toy regression problem are summarized in Table~\ref{tab:toy}.
Note that when the target function $f^*$ is perfectly learned, the mean squared error is the variance of noise $\epsilon$, i.e., 0.01.
Therefore, our method significantly mitigates the influence of covariate shift. Since the IRLS algorithm is needed when using Tukey's bisquare loss, the training should take longer time than that when using the squared loss, and we confirmed this fact from the results in Table~\ref{tab:toy}.
\begin{table}[t]
\centering
\caption{Mean squared test errors averaged over 100 trials on the toy dataset.
The numbers in the brackets are the standard deviations.
The best method and comparable ones based on the \emph{paired t-test} at the significance level 5\% are described in bold face.
The computation time is averaged over 100 trials.
``Squared'' denotes the squared loss, while ``Tukey'' denotes Tukey's bisquare loss.}
\vspace{2mm}
\label{tab:toy}
\begin{tabular}{l|c|c}
\toprule[1pt]
Methods & MSE(SD) & Computation time (sec) \\
\midrule
ERM (squared) & 0.1453 (0.1812) & 0.04 \\
ERM (Tukey)   & 0.0760 (0.0733) & 0.09 \\
EIWERM (squared) & 0.0198 (0.0151) & 0.33 \\
EIWERM (Tukey) & 0.0161 (0.0106) & 0.76 \\
RIWERM (squared) & 0.0162 (0.0100) & 0.47 \\
RIWERM (Tukey) & 0.0149 (0.0073) & 0.84\\
one-step (squared) & 0.0131 (0.0036) & 0.73 \\
one-step (Tukey) & \textbf{0.0125} (\textbf{0.0021}) & 1.50 \\
\bottomrule[1pt]
\end{tabular}
\end{table}

\begin{table}[t]
\centering
\caption{Mean squared test errors/mean test misclassification rates averaged over 100 trials on regression/binary classification benchmark datasets.
The numbers in the brackets are the standard deviations.
All the error values are normalized so that the mean error by ``ERM'' will be one.
For each dataset, the best method and comparable ones based on the \emph{paired t-test} at the significance level 5\% are described in bold face.
The upper half are regression datasets and the lower half are binary classification datasets.}
\vspace{2mm}
\label{tab:benchmark}
\begin{tabular}{c|cccc}
\toprule[1pt]
Dataset & ERM & EIWERM & RIWERM & one-step \\
\midrule
auto & 1.00 (0.22) & 1.08 (0.25) & 1.08 (0.23) & \textbf{0.99 (0.21)} \\
bike & 1.00 (0.10) & 0.97 (0.10) & 0.98 (0.10) & \textbf{0.95 (0.08)} \\
parkinsons & 1.00 (0.28) & 0.93 (0.17) & 0.92 (0.16) & \textbf{0.76 (0.05)} \\
wine & 1.00 (0.22) & 0.95 (0.12) & 0.95 (0.14) & \textbf{0.90 (0.07)} \\
\midrule
australian & 31.62 (17.88) & 30.70 (16.35) & 29.82 (14.83) & \textbf{25.57 (12.74)} \\
breast & 22.13 (10.36) & 21.82 (11.20) & 22.00 (13.38) & \textbf{16.55 (9.09)} \\
diabetes & 43.35 (9.56) & 41.67 (8.66) & 43.26 (8.42) & \textbf{38.57 (6.36)} \\
heart & 34.91 (12.45) & \textbf{32.06 (11.05)} & \textbf{33.39 (12.24)} & \textbf{31.39 (10.36)} \\
sonar & \textbf{39.03 (6.69)} & \textbf{38.77 (6.37)} & \textbf{38.83 (7.15)} & \textbf{37.69 (7.17)} \\
\bottomrule[1pt]
\end{tabular}
\end{table}

\subsubsection{Experiments on Regression and Binary Classification Benchmark Datasets}
Below, we report the results of experiments on regression benchmark datasets from UCI\footnote{\url{https://archive.ics.uci.edu/ml/datasets.php}} and binary classification benchmark datasets from LIBSVM\footnote{\url{https://www.csie.ntu.edu.tw/\textasciitilde cjlin/libsvmtools/datasets/}}.

We considered experimental settings with both synthetically created covariate shift and naturally occurring covariate shift. 
To perform train-test split for the datasets with naturally occurring covariate shift, we followed \cite{ahmed2014dataset}, \cite{pmlr-v51-chen16d}, and \cite{NIPS2006_a74c3bae} to separate the auto mpg dataset, the bike sharing dataset, the parkinsons dataset, and the wine quality dataset based on different origins, different semesters, different age ranges, and different types, respectively.
In the rest of the datasets, we synthetically introduced covariate shift in the following way similarly to \cite{cortes2008sample}.
First, we used Z-score normalization to preprocess all the input samples.
Then, an example $(\bx, y)$ was assigned to the training dataset with probability $\exp(v)/(1+\exp(v))$ and to the test dataset with probability $1/(1+\exp(v))$, where $v=16\bw^\top\bx/\sigma$, $\sigma$ is the standard deviation of $\bw^\top\bx$, and $\bw\in\bR^d$ is some given projection vector.
To ensure that the methods are tested in challenging covariate shift situations, we randomly sampled projection directions and chose the one such that the classifier trained on the training dataset generalizes worst to the test dataset for train-test split.

By following the above procedure, we split the datasets into training datasets and test datasets (with some randomness in synthetic cases).
Then we sampled a certain number (depending on the size of the dataset) of training samples and test input samples for training. We used the rest of test samples for evaluating the performance.
We ran 100 trials for each dataset.\footnotemark
\footnotetext{It means that we conduct the experiment for each dataset 100 times with different random draws of training and test samples.}
As discussed in Section~\ref{subsec:osa}, we used the squared loss as the surrogate loss function for all the methods including the one-step approach in the experiments.

The experimental results on benchmark datasets are summarized in Table~\ref{tab:benchmark}.
The table shows the proposed one-step approach outperforms or is comparable to the baseline methods with the best performance, which suggests that it is a promising method for covariate shift adaptation.

\subsubsection{Multi-class Classification Experiments with Neural Networks}
Finally, we designed a covariate shift setting and conducted experiments on the Fashion-MNIST \cite{xiao2017} and Kuzushiji-MNIST \cite{clanuwat2018deep} benchmark datasets for image classification using convolutional neural networks (CNNs).

Based on the fact that the labels of the images from those datasets are invariant to rotation transformation, we introduced covariate shift to the image datasets in the following way: we rotated each image $I_i$ in the training sets by angle $\theta_i$, where $\theta_i/180^\circ$ was drawn from a beta distribution $\mathrm{Beta}(a, b)$, and rotated each image $J_i$ in the test sets by angle $\phi_i$,  where $\phi_i/180^\circ$ was drawn from another beta distribution $\mathrm{Beta}(b, a)$. 
The parameters $a$ and $b$ control the shift level, and we tested three different levels in our experiments: $(a, b) = (2, 4)$, $(2, 5)$, and $(2, 6)$. 
Since our experiments were conducted in an inductive manner, we also rotated each image $I_i$ in the training sets by angle $\psi_i$, where $\psi_i/180^\circ$ was drawn from the beta distribution $\mathrm{Beta}(b, a)$ to obtain the unlabeled test images for training.

We used Algorithm~\ref{alg:gradient} as our implementation of the one-step appraoch.
We used the softmax cross-entropy loss as a surrogate loss and use 5-layer CNNs, which consist of 2 convolutional layers with pooling and 3 fully connected layers, to model the classifier $f$ and the weight model $g$. 
In order to learn useful weights, we pretrained $g$ in a binary classification problem whose goal is to discriminate between $\left\{\bx_i^\tr\right\}_{i=1}^{n_\tr}$ and $\big\{\bx_j^\te\big\}_{j=1}^{n_\te}$ and froze the parameters in the first two convolutional layers. 
We trained $f$ and $g$ for 20 rounds for the one-step approach, where a round consists of 5 epochs of training $g$ followed by 10 epochs of training $f$: we trained the models for 300 epochs in total. Details can be found in \cite{zhang2021one}.

The experimental results summarized in Table~\ref{tab:deep} verify the effectiveness of our one-step approach in image classification problems with neural networks. 
Specifically, the table shows that the ordinary ERM performs poorly under covariate shift, the weighted methods all improve the performance, and the one-step approach further improves the performance especially under large covariate shift (i.e., the difference between shift parameters $a$ and $b$ is large).

\begin{table}[t]
\centering
\caption{Mean test classification accuracy averaged over 5 trials on image datasets with neural networks.
The numbers in the brackets are the standard deviations.
For each dataset, the best method and comparable ones based on the \emph{paired t-test} at the significance level 5\% are described in bold face.}
\vspace{2mm}
\label{tab:deep}
\begin{tabular}{c|c|cccc}
\toprule[1pt]
Dataset & \tabincell{c}{Shift Level\\($a$, $b$)} & ERM & EIWERM & RIWERM & one-step\\
\midrule
 & (2, 4) & 81.71(0.17) & 84.02(0.18) & 84.12(0.06) & \textbf{85.07(0.08)} \\
Fashion-MNIST & (2, 5) & 72.52(0.54) & 76.68(0.27) & 77.43(0.29) & \textbf{78.83(0.20)} \\
 & (2, 6) & 60.10(0.34) & 65.73(0.34) & 66.73(0.55) & \textbf{69.23(0.25)} \\
\midrule
 & (2, 4) & 77.09(0.18) & 80.92(0.32) & 81.17(0.24) & \textbf{82.45(0.12)} \\
Kuzushiji-MNIST & (2, 5) & 65.06(0.26) & 71.02(0.50) & 72.16(0.19) & \textbf{74.03(0.16)} \\
 & (2, 6) & 51.24(0.30) & 58.78(0.38) & 60.14(0.93) & \textbf{62.70(0.55)} \\
\bottomrule[1pt]
\end{tabular}
\end{table}
 

\section{Dynamic Importance Weighting}
\label{sec:diw}
So far, we have discussed covariate shift as the primal target for transfer learning. In this section, we consider the \emph{full-distribution shift} problem and introduce a novel method, dynamic importance weighting (DIW)~\cite{fang2020rethinking}, for making importance weighting work well for deep learning under distribution shift. 

\subsection{Motivation and Problem Setup}
\label{sec:diw_setup}
Let us consider the problem setup of \emph{full-distribution shift}, where the training data $\left\{\left(\bx_i^\tr,  y_i^\tr\right)\right\}_{i=1}^{n_\tr}$ are drawn from $p_\tr(\bx,y)$, the test data $\big\{\big(\bx_j^\te,  y_j^\te\big)\big\}_{j=1}^{n_\te}$ are drawn from $p_\te(\bx,y)$, and $p_\tr(\bx,y)\neq p_\te(\bx,y)$. 
Then, importance weighting (IW) under \emph{full-distribution shift} can be formulated as
\begin{align}
    \label{eq:IW-identity}
    \bE_{(\bx^\te, y^\te)\sim p_\te(\bx, y)}\left[\ell(f(\bx^\te), y^\te)\right] = \bE_{(\bx^\tr, y^\tr)\sim p_\tr(\bx, y)}\left[w(\bx^\tr, y^\tr)\ell(f(\bx^\tr), y^\tr)\right],
\end{align}
where the \emph{importance} is $w(\bx,y)=p_\te(\bx,y)/p_\tr(\bx,y)$.
Here we assume that we have a tiny set of validation data from $p_\te(\bx,y)$. For brevity, we abbreviate \emph{importance estimation} in Section~\ref{subsec:iwerm} as \emph{weight estimation}~{(WE)} and abbreviate \emph{importance-weighted ERM} as \emph{weighted classification}~{(WC)}.



IW works very well if the form of data is simple (e.g., some linear model suffices), and it has been the common practice of non-deep learning under distribution shift \cite{sugiyama2012density}. However, IW cannot work well if the form of data is complex \cite{byrd2018effect}.
Recall that for the $k$-class classification problem with input domain $\cX\subset\bR^{d_\mathrm{in}}$ and output domain $\cY\coloneqq[k]$,
$w(\bx, y)$ processes $(d_\mathrm{in}+1)$-dimensional or $(d_\mathrm{in}+k)$-dimensional input depending on how $y$ is encoded and $f(\bx)$ processes $d_\mathrm{in}$-dimensional input, and consequently WE is not necessarily easier than WC.
Hence, when a deep model is used in WC, more \emph{expressive power} may be needed in WE.

Here we improve IW for deep learning under distribution shift.
Nevertheless, WE and WC are different tasks with different goals, and it is difficult to boost the expressive power of WE for three reasons:
\begin{itemize}
    \item some WE methods are \emph{model-free} (e.g., \textit{KMM} in Section~\ref{subsubsec: kmm}), i.e., they assign weights to data without a model of $w$;
    \item other WE methods are \emph{model-based} and also \emph{model-independent} (e.g., \textit{LSIF} and \textit{uLSIF} in Section~\ref{subsubsec: lsif}), but the optimizations are constrained due to $\bE_{p_\tr(\bx,y)}[w(\bx,y)]=\bE_{p_\te(\bx,y)}[1]=1$ and incompatible with stochastic solvers;
    \item most powerful deep models nowadays are hard to train with the WE optimizations since they are \emph{designed for classification}, even if we ignore the constraint or satisfy it within each mini-batch.
\end{itemize}
Therefore, it sounds better to boost the expressive power by an external \emph{feature extractor}~(FE).
For instance, we may rely on $f$ that is a deep model chosen for the classification problem to be solved.
Going along this way, we encounter the \emph{circular dependency} in Figure~\ref{fig:circle}:
originally we need $w$ to train $f$; now we need a trained $f$ to estimate $w$.
It becomes a chicken-or-egg causality dilemma.

We think of two possible ways to solve the circular dependency, one \emph{pipelined} and one \emph{end-to-end}.
The pipelined solution pretrains a deep classifier (DC) from unweighted training data, and creates an FE from this DC;
then, WE is carried out on the data transformed by the FE.
Since the weights cannot change, we call this method \emph{static importance weighting}~(SIW), as illustrated in the top diagram of Figure~\ref{fig:siwvsdiw}.
Here, the DC is biased to the training data, and so is the FE, which could be empirically confirmed~\cite{fang2020rethinking}.
As a result, this naive pipelined solution is only slightly better than no FE unfortunately.

To overcome the bias of SIW, we propose \emph{dynamic importance weighting}~(DIW) \cite{fang2020rethinking} as an end-to-end solution; see the bottom diagram of Figure~\ref{fig:siwvsdiw}.
DIW iterates between WE (on the transformed data) and WC (for updating the DC and FE) and combines them in a seamless manner. This end-to-end solution can gradually \emph{improve the importance weight} and \emph{reduce the bias of $f$}.

\begin{figure}[t]
    \begin{minipage}[c]{0.34\textwidth}
        \vskip-1pt%
        \includegraphics[width=\textwidth]{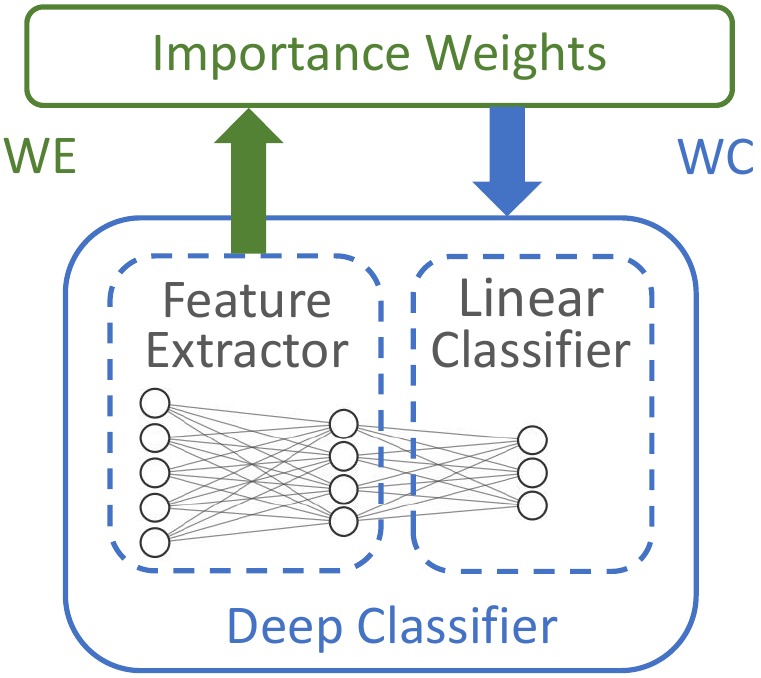}
        {\footnotesize%
        Blue arrow depicts WC depending on WE; green arrow depicts WE depending on WC---this makes a circle.
        }
        \caption{Circular dependency.}
        \label{fig:circle}
    \end{minipage}\hspace{1em}%
    \begin{minipage}[c]{0.66\textwidth}
        \includegraphics[width=\textwidth]{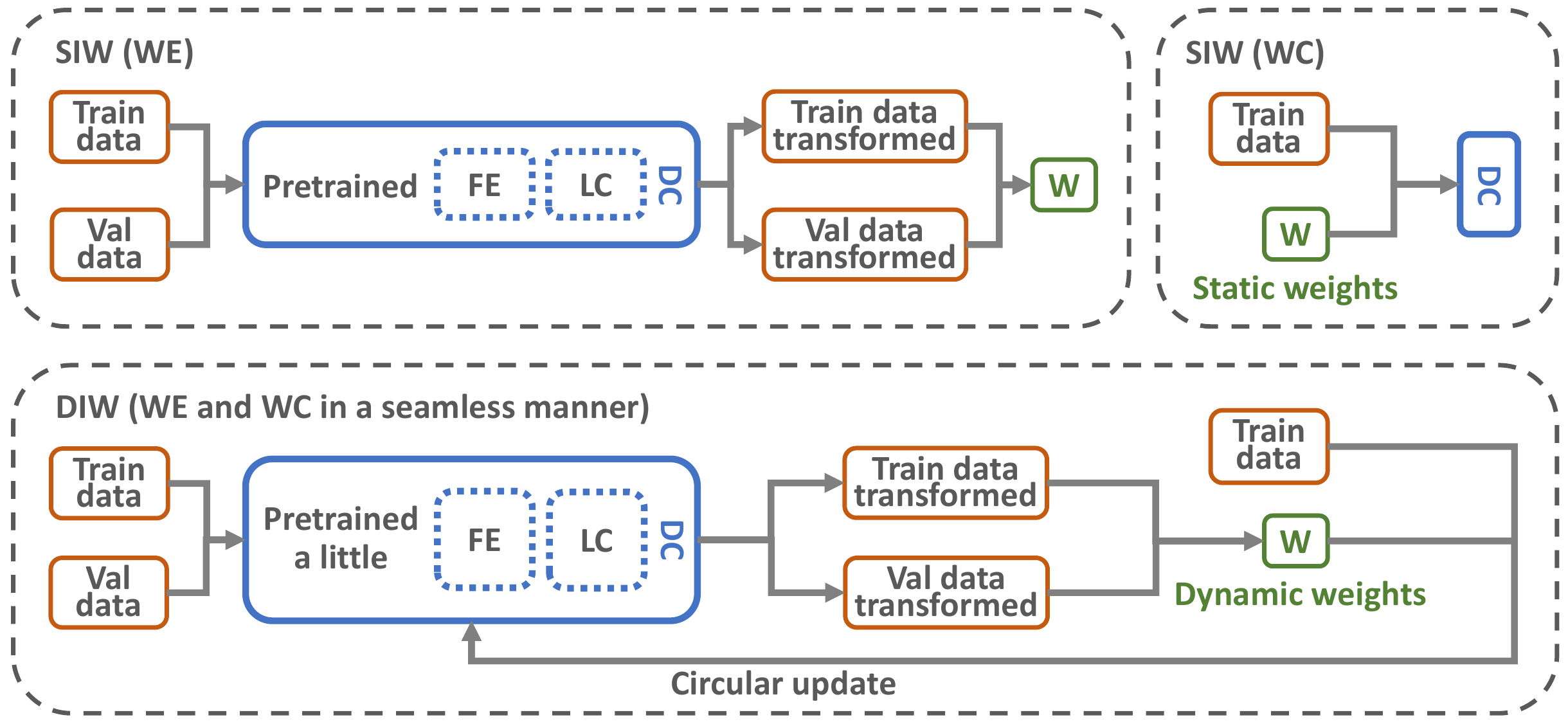}
        {\footnotesize%
            SIW/DIW stands for static/dynamic importance weighting;
            FE is short for feature extractor, and LC/DC is for linear/deep classifier;
            W is a set of weights.
            Circular update is employed to solve circular dependency.
        }
        \caption{Illustrations of SIW and DIW.}
        \label{fig:siwvsdiw}
    \end{minipage}
\end{figure}

\subsection{Dynamic Importance Weighting}
Here we introduce the details of the proposed method, dynamic importance weighting (DIW), including the weighted classification, non-linear transformation of data, practical choices of the transformation, and the distribution matching in DIW.

\subsubsection{Weighted Classification}
In this section, we assume that our classifier to be trained is a deep network parameterized by $\theta$, denoted by $\mf:\cX\to\bR^k$.
Let $\ell:\bR^k\times\cY\to\bR_{+}$ be a \emph{surrogate loss function} for $k$-class classification, e.g., the \emph{softmax cross-entropy loss}.
The classification risk of $\mf$ is defined as
\begin{align}
\label{eq:risk}%
R(\mf) = \bE_{p_\te(\bx,y)}[\ell(\mf(\bx),y)],
\end{align}
which is the performance measure we would like to optimize.
According to Eq.~\eqref{eq:IW-identity}, if $w(\bx,y)$ is given or $\mathcal{W}=\{w_i=w(\bx_i^\tr,y_i^\tr)\}_{i=1}^{n_\tr}$ is given, $R(\mf)$ can be approximated by
\begin{align}
\label{eq:risk-emp}%
\displaystyle
\widehat{R}(\mf) = \frac{1}{n_\tr}\sum_{i=1}^{n_\tr} w_i\ell(\mf(\bx_i^\tr),y_i^\tr),
\end{align}
which is the objective of WC.
With the \emph{optimal weights}, the weighted empirical risk in Eq.~\eqref{eq:risk-emp} is an \emph{unbiased estimator} of the risk in Eq.~\eqref{eq:risk}, and hence the trained classifier as the minimizer of $\widehat{R}(\mf)$ should converge to the minimizer of $R(\mf)$ as $n_\tr$ approaches infinity \cite{shimodaira2000improving,sugiyama2007covariate,huang2007correcting,sugiyama2008direct,sugiyama2008direct1,kanamori2009least}.

\subsubsection{Non-linear Transformation of Data}
Now, the issue is how to estimate the function $w$ or the set $\mathcal{W}$.
As discussed earlier in Section~\ref{sec:diw_setup}, we should boost the expressive power externally but not internally.
This means we should apply a \emph{non-linear transformation} of data rather than directly modeling $w(\bx,y)$ or $p_\tr(\bx,y)$ and $p_\te(\bx,y)$ by deep networks.
Let $\pi:\cX\times\cY\to\bR^{d_\mathrm{r}}$ or $\pi:\cX\times\cY\to\bR^{d_\mathrm{r}-1}\times\cY$ be a transformation, where $d_\mathrm{r}$ is the reduced dimension and $d_\mathrm{r}\ll d$;
let $\bz=\pi(\bx,y)$ be the transformed random variable, whose source of randomness is $(\bx,y)$ exclusively.
By applying $\pi$, we expect that WE on $\bz$ will be much easier than WE on $(\bx,y)$.
The feasibility of applying $\pi$ is justified below.

\begin{theorem}\cite{fang2020rethinking}%
    \label{thm:transform}%
~For a fixed, deterministic and invertible transformation $\pi:(\bx,y)\mapsto\bz$, let $p_\tr(\bz)$ and $p_\te(\bz)$ be the probability density functions induced by $p_\tr(\bx,y)$, $p_\te(\bx,y)$, and $\pi$.
    Then,
    \begin{align}
    \label{eq:transform}%
    w(\bx,y)
    = \frac{p_\te(\bx,y)}{p_\tr(\bx,y)}
    = \frac{p_\te(\bz)}{p_\tr(\bz)}
    = w(\bz).
    \end{align}
\end{theorem}
\begin{proof}
    Let $F_\tr(\bx,y)$, $F_\te(\bx,y)$, $F_\tr(\bz)$ as well as $F_\te(\bz)$ be the corresponding cumulative distribution functions.
    Then the fundamental theorem of calculus,
    and three properties of $\pi$ namely $\pi$ is fixed, deterministic, and invertible, it holds that
    \begin{align}
    \label{eq:equiv-tr}%
    p_\tr(\bx,y)\dif\bx = \dif F_\tr(\bx,y) &= \dif F_\tr(\bz) = p_\tr(\bz)\dif\bz,\\
    \label{eq:equiv-te}%
    p_\te(\bx,y)\dif\bx = \dif F_\te(\bx,y) &= \dif F_\te(\bz) = p_\te(\bz)\dif\bz,
    \end{align}
    where $\dif$ denotes the differential operator, and
    \begin{align*}
    \displaystyle
    \dif F_{*}(\bx,y) = \frac{\partial}{\partial\bx}
    ( \sum_{y'\le y}\int_{\bx'\le\bx}p_{*}(\bx',y')\dif\bx'
    -\sum_{y'<y}\int_{\bx'\le\bx}p_{*}(\bx',y')\dif\bx' )
    \cdot \dif\bx,
    \end{align*}
    where $_*$ indicates $\tr$ or $\te$. 
    For simplicity, the continuous random variable $\bx$ and the discrete random variable $y$ are considered separately.
    Dividing Eq.~\eqref{eq:equiv-te} by Eq.~\eqref{eq:equiv-tr} proves Eq.~\eqref{eq:transform}.
\end{proof}

Theorem~\ref{thm:transform} requires that $\pi$ satisfies three properties: we cannot guarantee $\dif F_\tr(\bz)=p_\tr(\bz)\dif\bz$ if $\pi$ is not fixed or $\dif F_\tr(\bx,y)=\dif F_\tr(\bz)$ if $\pi$ is neither deterministic nor invertible.
As a result, when $\mathcal{W}$ is computed in WE, $\mf$ is regarded as fixed, and it could be switched to the \emph{evaluation mode} from the \emph{training mode} to avoid the randomness due to dropout \cite{srivastava2014dropout} or similar randomized algorithms.
The invertibility of $\pi$ is non-trivial: it assumes that $\cX\times\cY$ is generated by a manifold $\cM\subset\bR^{d_\mathrm{m}}$ with an intrinsic dimension $d_\mathrm{m}\le d_\mathrm{r}$, and $\pi^{-1}$ recovers the generating function from $\cM$ to $\cX\times\cY$.
If $\pi$ is from parts of $\mf$, $\mf$ must be a reasonably good classifier so that $\pi$ compresses $\cX\times\cY$ back to $\cM$.
This finding is the circular dependency in Figure~\ref{fig:circle}, which is the major theoretical contribution.

\subsubsection{Practical Choices of $\pi$}
It seems obvious that $\pi$ can be $\mf$ as a whole or without its topmost layer.
However, the latter drops $y$ and corresponds to assuming
\begin{align}
\label{eq:transform-marginal}%
\displaystyle
p_\tr(y\mid\bx) = p_\te(y\mid\bx)
\Longrightarrow \frac{p_\te(\bx,y)}{p_\tr(\bx,y)}
= \frac{p_\te(\bx)\cdot p_\te(y\mid\bx)}{p_\tr(\bx)\cdot p_\tr(y\mid\bx)}
= \frac{p_\te(\bx)}{p_\tr(\bx)}
= \frac{p_\te(\bz)}{p_\tr(\bz)},
\end{align}
which is only possible under \emph{covariate shift} \cite{pan2009survey,shimodaira2000improving,sugiyama2008direct,sugiyama2008direct1}.
It is conceptually a bad idea to attach $y$ to the latent representation of $\bx$, since the distance metric on $\cY$ is completely different.
A better idea to take the information of $y$ into account may consist of three steps.
First, estimate $p_\te(y)/p_\tr(y)$;
second, partition $\{(\bx_i^\tr,y_i^\tr)\}_{i=1}^{n_\tr}$ and $\{(\bx_j^\te,y_j^\te)\}_{j=1}^{n_\te}$ according to $y$;
third, invoke WE $k$ times on $k$ partitions separately based on the following identity: let $w_y=p_\te(y)/p_\tr(y)$, then
\begin{align}
\label{eq:transform-class-conditional}%
\displaystyle
\frac{p_\te(\bx,y)}{p_\tr(\bx,y)}
= \frac{p_\te(y)\cdot p_\te(\bx\mid y)}{p_\tr(y)\cdot p_\tr(\bx\mid y)}
= w_y \cdot \frac{p_\te(\bx\mid y)}{p_\tr(\bx\mid y)}
= w_y \cdot \frac{p_\te(\bz\mid y)}{p_\tr(\bz\mid y)}.
\end{align}
That being said, in a small mini-batch, invoking WE $k$ times on even smaller partitions might be remarkably less reliable than invoking it once on the whole mini-batch.

To this end, we propose an alternative choice $\pi:(\bx,y)\mapsto\ell(\mf(\bx),y)$ that is motivated as follows.
In practice, we are not sure about the existence of $\cM$, we cannot check whether $d_\mathrm{m}\le d_\mathrm{r}$ when $\cM$ indeed exists, or it is computationally hard to confirm that $\pi$ is invertible.
Consequently, Eqs.~\eqref{eq:transform-marginal} and~\eqref{eq:transform-class-conditional} may not hold or only hold approximately.
As a matter of fact, Eq.~\eqref{eq:IW-identity} also only holds approximately after replacing the expectations with empirical averages, and then it may be too much to stick to the optimal solution $w(\bx,y)$.
According to Eq.~\eqref{eq:IW-identity}, there exists $w(\bx,y)$ such that for all possible
$h:\cX\times\cY\to\bR$,
\begin{align*}
\displaystyle
\frac{1}{n_\te}\sum_{j=1}^{n_\te} h(\bx_j^\te,y_j^\te)
& \approx \bE_{p_\te(\bx,y)}[h(\bx,y)] \\
& \approx \bE_{p_\tr(\bx,y)}[w(\bx,y)h(\bx,y)] 
 \approx \frac{1}{n_\tr}\sum_{i=1}^{n_\tr} \hat{w}_i h(\bx_i^\tr,y_i^\tr),
\end{align*}
where $\hat{w}_i$ is an estimated importance weight for $i=1,\ldots,n_\tr$.
This goal, \emph{IW for everything}, is too general and its only solution is $\hat{w}_i=w(\bx_i^\tr, y_i^\tr)$;
nonetheless, it is more than needed---\emph{IW for classification} is our current goal.

Specifically, the goal of DIW is to find a set of weights $\hat{\mathcal{W}}=\{\hat{w}_i\}_{i=1}^{n_\tr}$ such that for $\ell(\mf(\bx),y)$,
\begin{align}
\label{eq:DIW-goal}
\displaystyle
\frac{1}{n_\te}\sum_{j=1}^{n_\te} \ell(\mf(\bx_j^\te),y_j^\te) \big|_{\theta=\theta_t}
\approx \frac{1}{n_\tr}\sum_{i=1}^{n_\tr} \hat{w}_i\ell(\mf(\bx_i^\tr),y_i^\tr) \big|_{\theta=\theta_t},
\end{align}
where the left- and right-hand sides are conditioned on $\theta=\theta_t$, and $\theta_t$ holds model parameters at a certain time point $t$ of training.
After $\hat{\mathcal{W}}$ is found, $\theta_t$ will be updated to $\theta_{t+1}$, and the current $\mf$ will move to the next $\mf$;
then, we need to find a new set of weights satisfying Eq.~\eqref{eq:DIW-goal} again.
Compared with the general goal of IW, the goal of DIW is special and easy to achieve, and then there may be many different solutions, any of which can be used to replace $\mathcal{W}=\{w_i\}_{i=1}^{n_\tr}$ in $\widehat{R}(\mf)$.
The above argument elaborates the motivation of $\pi:(\bx,y)\mapsto\ell(\mf(\bx),y)$.
This is possible thanks to the \emph{dynamic nature of weights} in DIW, which is the major methodological contribution.

\subsubsection{Distribution Matching}
Finally, we perform distribution matching between the set of transformed training data $\{\bz_i^\tr\}_{i=1}^{n_\tr}$ and the set of transformed validation data $\{\bz_j^\te\}_{j=1}^{n_\te}$.
Let $\cH$ be a Hilbert space of real-valued functions on $\bR^{d_\mathrm{r}}$ with an inner product $\langle\cdot,\cdot\rangle_\cH$, or $\cH$ be a \emph{reproducing kernel Hilbert space}, where $k:(\bz,\bz')\mapsto\langle\phi_\mathrm{k}(\bz),\phi_\mathrm{k}(\bz')\rangle_\cH$ is the reproducing kernel of $\cH$ and $\phi_\mathrm{k}:\bR^{d_\mathrm{r}}\to\cH$ is the kernel-induced feature map \cite{scholkopf2001learning}. Then, we perform \emph{kernel mean matching} \cite{huang2007correcting} in Section~\ref{subsubsec: kmm} as follows.

Let $\mu_\tr=\bE_{w(\bz)\cdot p_\tr(\bx,y)}[\phi_\mathrm{k}(\bz)]$ and $\mu_\te=\bE_{p_\te(\bx,y)}[\phi_\mathrm{k}(\bz)]$ be the kernel embeddings of $p_\tr\cdot w$ and $p_\te$ in $\cH$, the discrepancy between the two can be 
approximated by
\begin{align}
\| \mu_\tr - \mu_\te \|_{\cH}^2
& \approx \| \frac{1}{n_\tr}\sum_{i=1}^{n_\tr}w_i\phi_\mathrm{k}(\bz_i^\tr) - \frac{1}{n_\te}\sum_{j=1}^{n_\te}\phi_\mathrm{k}(\bz_j^\te) \|_{\cH}^2 \nonumber\\
& \propto \bw^ \T\bK\bw - 2\bk^ \T\bw + \const,
\label{eq:MMD}
\end{align}
where $\bw\in\bR^{n_\tr}$ is the weight vector, $\bK\in\bR^{n_\tr\times n_\tr}$ is a kernel matrix such that $\bK_{ij}=k(\bz_i^\tr,\bz_j^\tr)$, and $\bk\in\bR^{n_\tr}$ is a vector such that $\bk_i=\frac{n_\tr}{n_\te}\sum_{j=1}^{n_\te}k(\bz_i^\tr,\bz_j^\te)$.
In practice, Eq.~\eqref{eq:MMD} is minimized subject to $0\le w_i\le B$ and $|\frac{1}{n_\tr}\sum_{i=1}^{n_\tr}w_i-1|\le\epsilon$ where $B>0$ and $\epsilon>0$ are hyperparameters as the upper bound of weights and the slack variable of $\frac{1}{n_\tr}\sum_{i=1}^{n_\tr}w_i=1$.
Eq.~\eqref{eq:MMD} is the objective of WE.
The whole DIW is shown in Algorithm~\ref{alg:DIW}, which is our major algorithmic contribution.
\begin{algorithm}[t]
    \caption{Dynamic importance weighting (in a mini-batch).}
    \label{alg:DIW}
    \begin{algorithmic}
        \REQUIRE a training mini-batch $\Str$,
            a teidation mini-batch $\Ste$,
            the current model $\boldsymbol{f}_{\theta_t}$
    \end{algorithmic}
    \begin{minipage}[t]{0.55\textwidth}\textbf{Hidden-layer-output transformation version:}\end{minipage}
    \begin{minipage}[t]{0.45\textwidth}\textbf{Loss-value transformation version:}\end{minipage}
    \begin{minipage}[t]{0.55\textwidth}
        \begin{algorithmic}[1]
            \STATE \texttt{forward} the input parts of $\Str$ \& $\Ste$
            \STATE \texttt{retrieve} the hidden-layer outputs $\Ztr$ \& $\Zte$
            \STATE \texttt{partition} $\Ztr$ \& $\Zte$ into $\{\Ztr_y\}_{y=1}^k$ \& $\{\Zte_y\}_{y=1}^k$
            \FOR{$y=1,\ldots,k$}
            \STATE \texttt{match} $\Ztr_y$ \& $\Zte_y$ to obtain $\hat{\mathcal{W}}_y$
            \STATE \texttt{multiply} all $w_i\in\hat{\mathcal{W}}_y$ by $w_y$
            \ENDFOR
            \STATE \texttt{compute} the loss values of $\Str$ as $\Ltr$
            \STATE \texttt{weight} the empirical risk $\widehat{R}(\mf)$ by $\{\hat{\mathcal{W}}_y\}_{y=1}^k$
            \STATE \texttt{backward} $\widehat{R}(\mf)$ and \texttt{update} $\theta$
        \end{algorithmic}
    \end{minipage}
    \begin{minipage}[t]{0.45\textwidth}
        \begin{algorithmic}[1]
            \STATE \texttt{forward} the input parts of $\Str$ \& $\Ste$
            \STATE \texttt{compute} the loss values as $\Ltr$ \& $\Lte$
            \STATE \texttt{match} $\Ltr$ \& $\Lte$ to obtain $\hat{\mathcal{W}}$
            \STATE \texttt{weight} the empirical risk $\widehat{R}(\mf)$ by $\hat{\mathcal{W}}$
            \STATE \texttt{backward} $\widehat{R}(\mf)$ and \texttt{update} $\theta$
        \end{algorithmic}
    \end{minipage}
\end{algorithm}

\subsection{Experiments}
In this section, we verify the effectiveness of DIW.%
\footnote{Our implementation of DIW is available at \url{https://github.com/TongtongFANG/DIW}.}
We first compare DIW based on loss-value transformation with baseline methods under label noise and class-prior shift.
We then report the results of ablation studies.

\begin{figure}[t]
    \centering
    \begin{minipage}[c]{0.02\textwidth}~\end{minipage}%
    \begin{minipage}[c]{0.326\textwidth}\centering\small 0.3 pair \end{minipage}%
    \begin{minipage}[c]{0.326\textwidth}\centering\small 0.4 symmetric \end{minipage}%
    \begin{minipage}[c]{0.326\textwidth}\centering\small 0.5 symmetric \end{minipage}\\
    \begin{minipage}[c]{0.02\textwidth}\small \rotatebox{90}{Fashion-MNIST} \end{minipage}%
    \begin{minipage}[c]{0.98\textwidth}
        \includegraphics[width=0.333\textwidth]{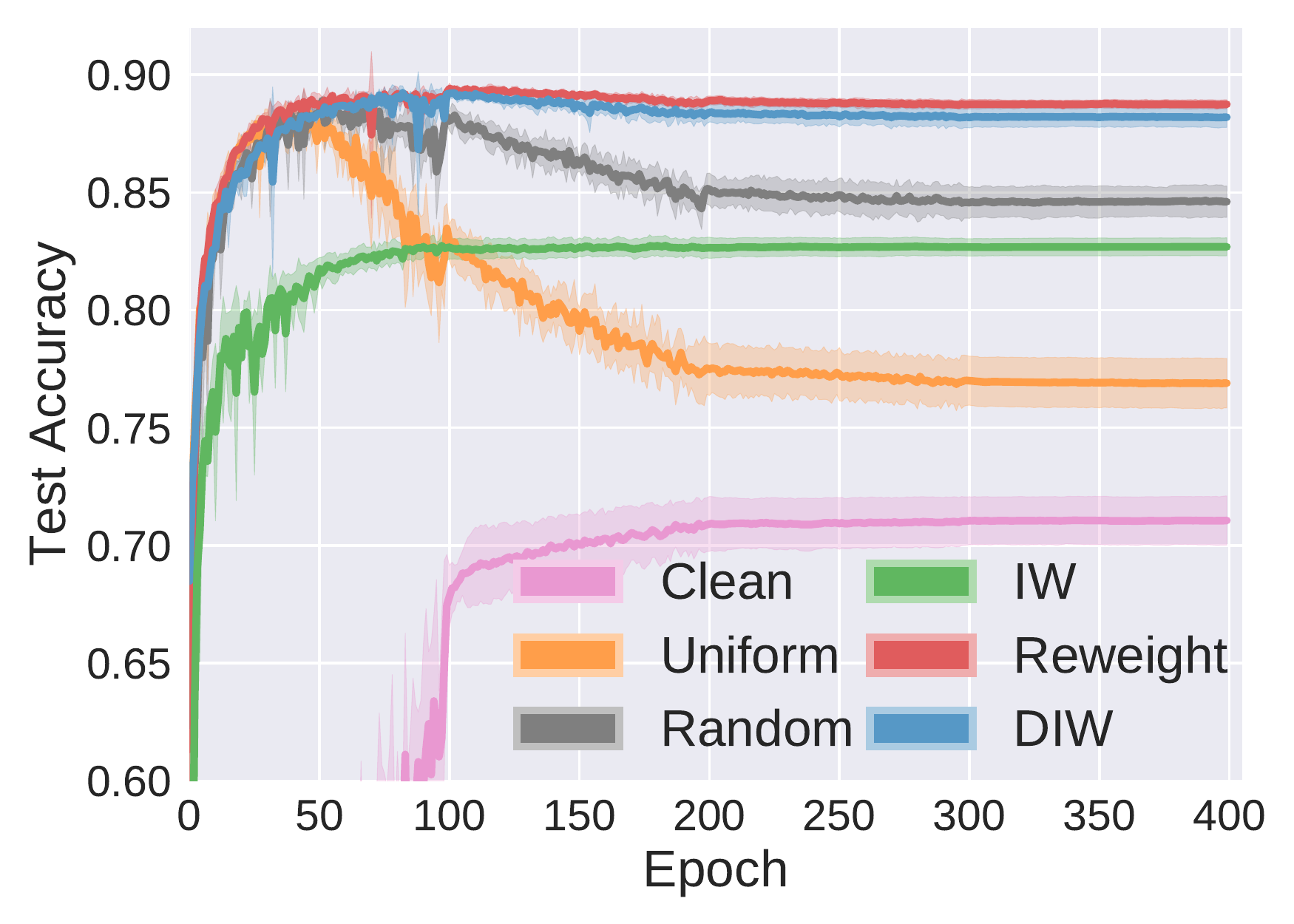}%
        \includegraphics[width=0.333\textwidth]{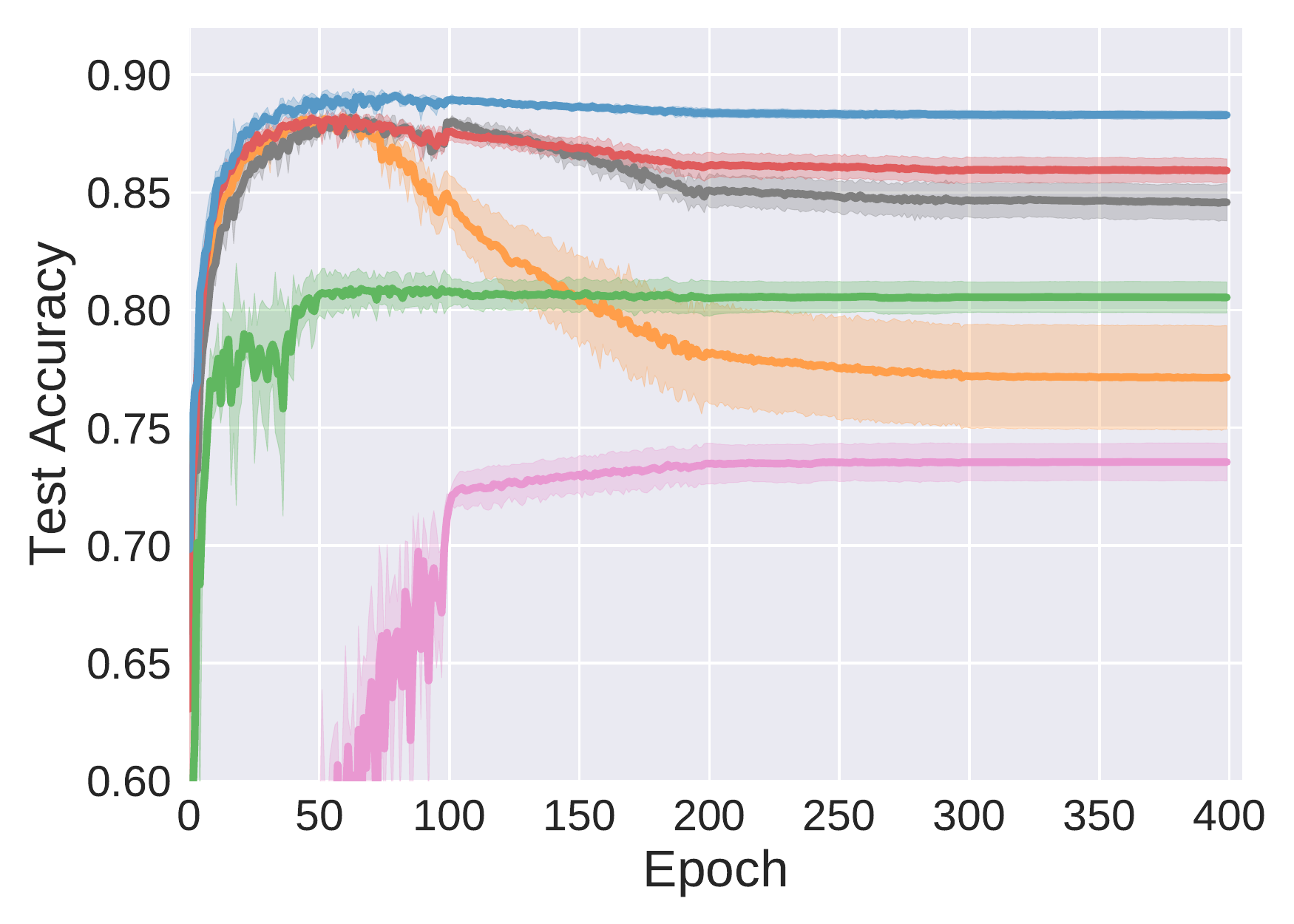}%
        \includegraphics[width=0.333\textwidth]{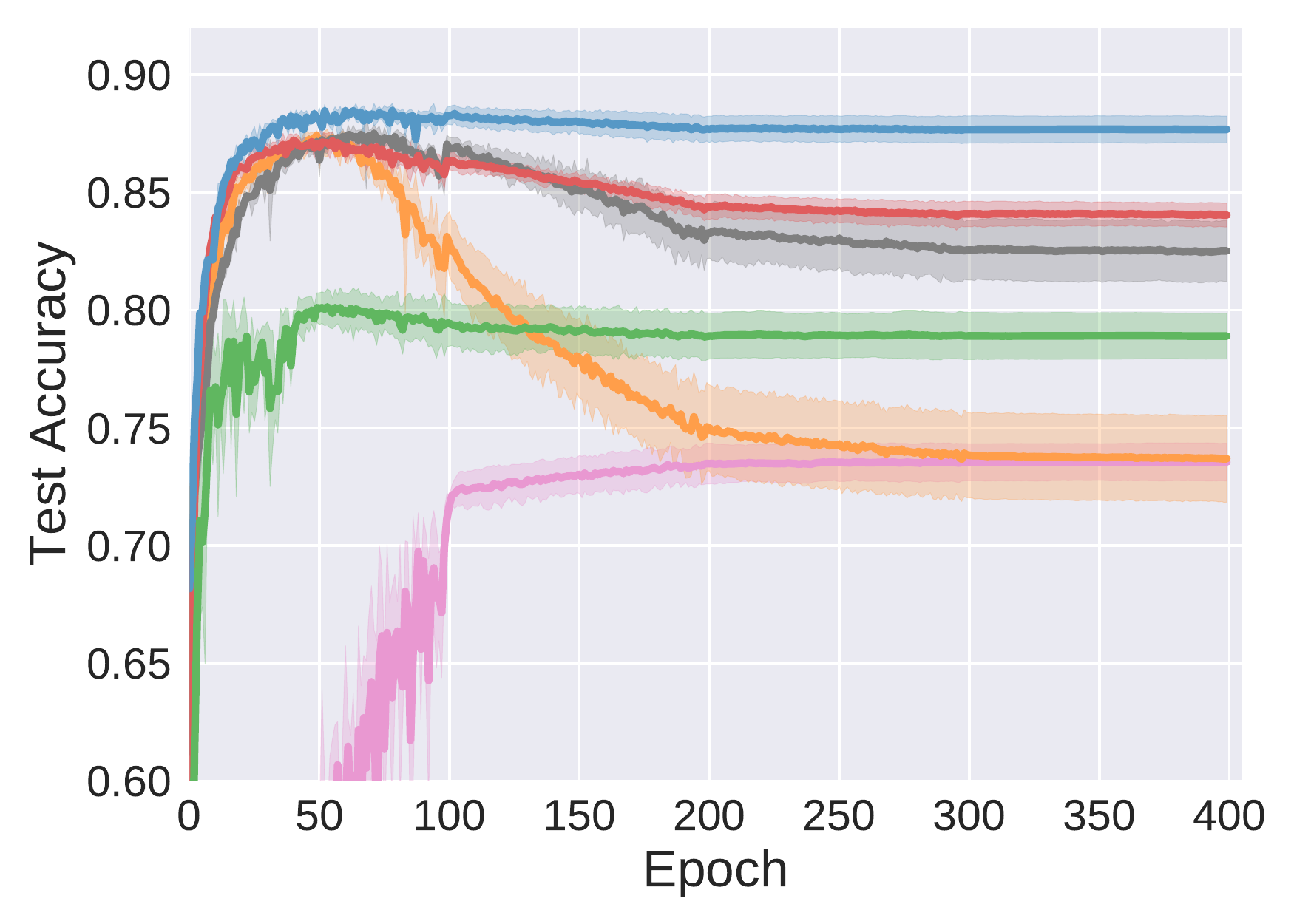}%
    \end{minipage}\\
    \begin{minipage}[c]{0.02\textwidth}\small \rotatebox{90}{CIFAR-10} \end{minipage}%
    \begin{minipage}[c]{0.98\textwidth}
        \includegraphics[width=0.333\textwidth]{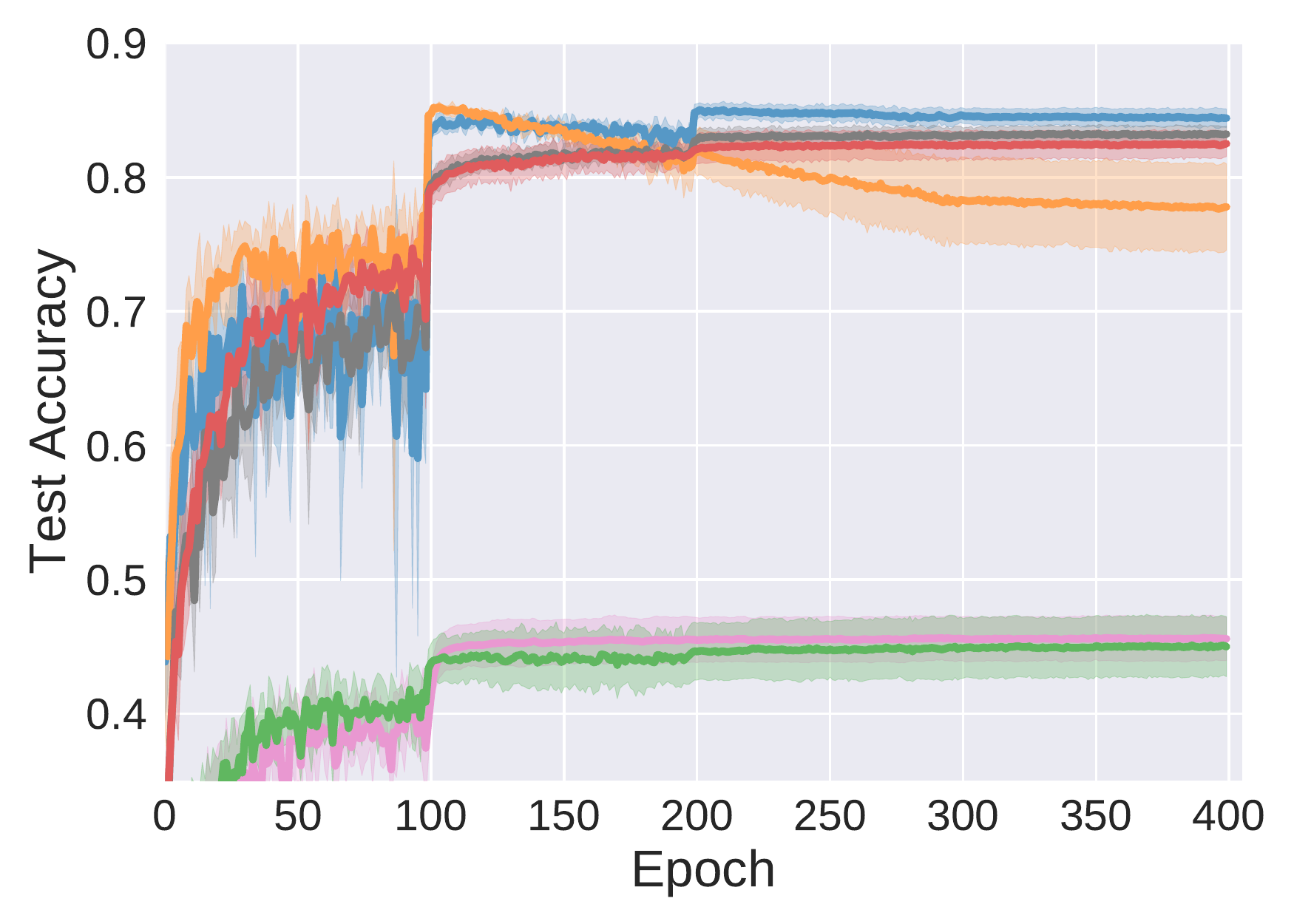}%
        \includegraphics[width=0.333\textwidth]{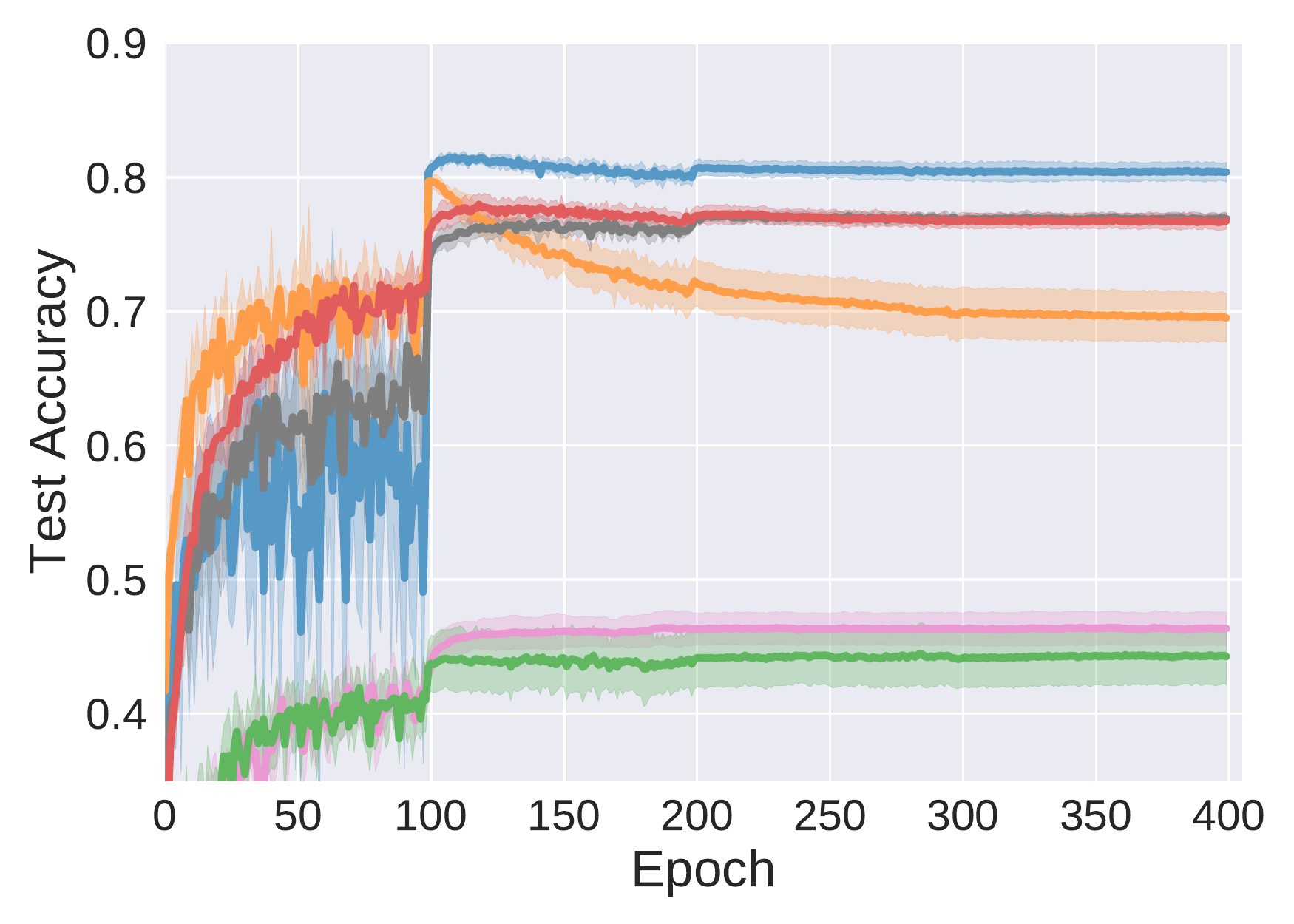}%
        \includegraphics[width=0.333\textwidth]{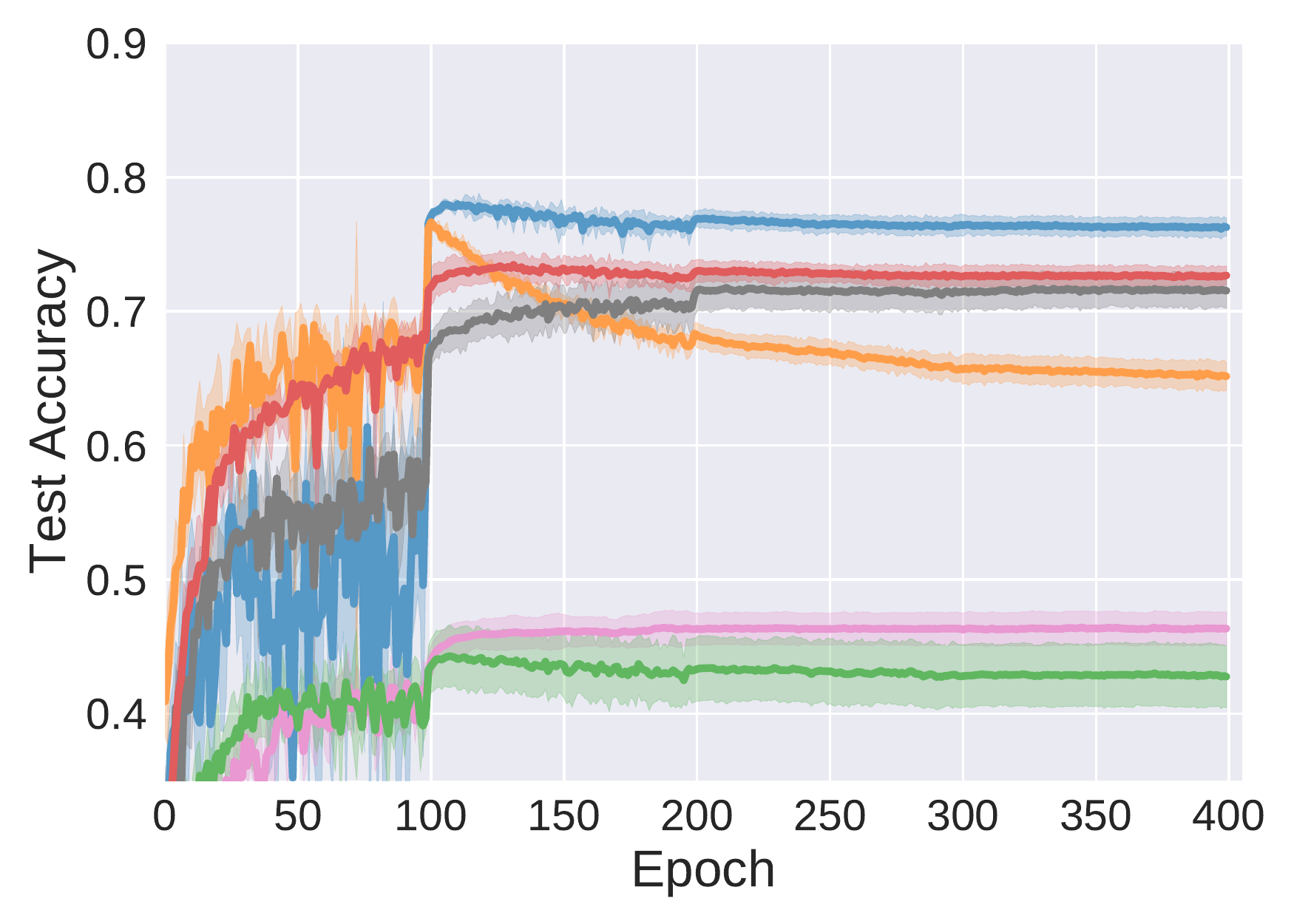}%
    \end{minipage}\\
    \begin{minipage}[c]{0.02\textwidth}\small \rotatebox{90}{CIFAR-100} \end{minipage}%
    \begin{minipage}[c]{0.98\textwidth}
        \includegraphics[width=0.333\textwidth]{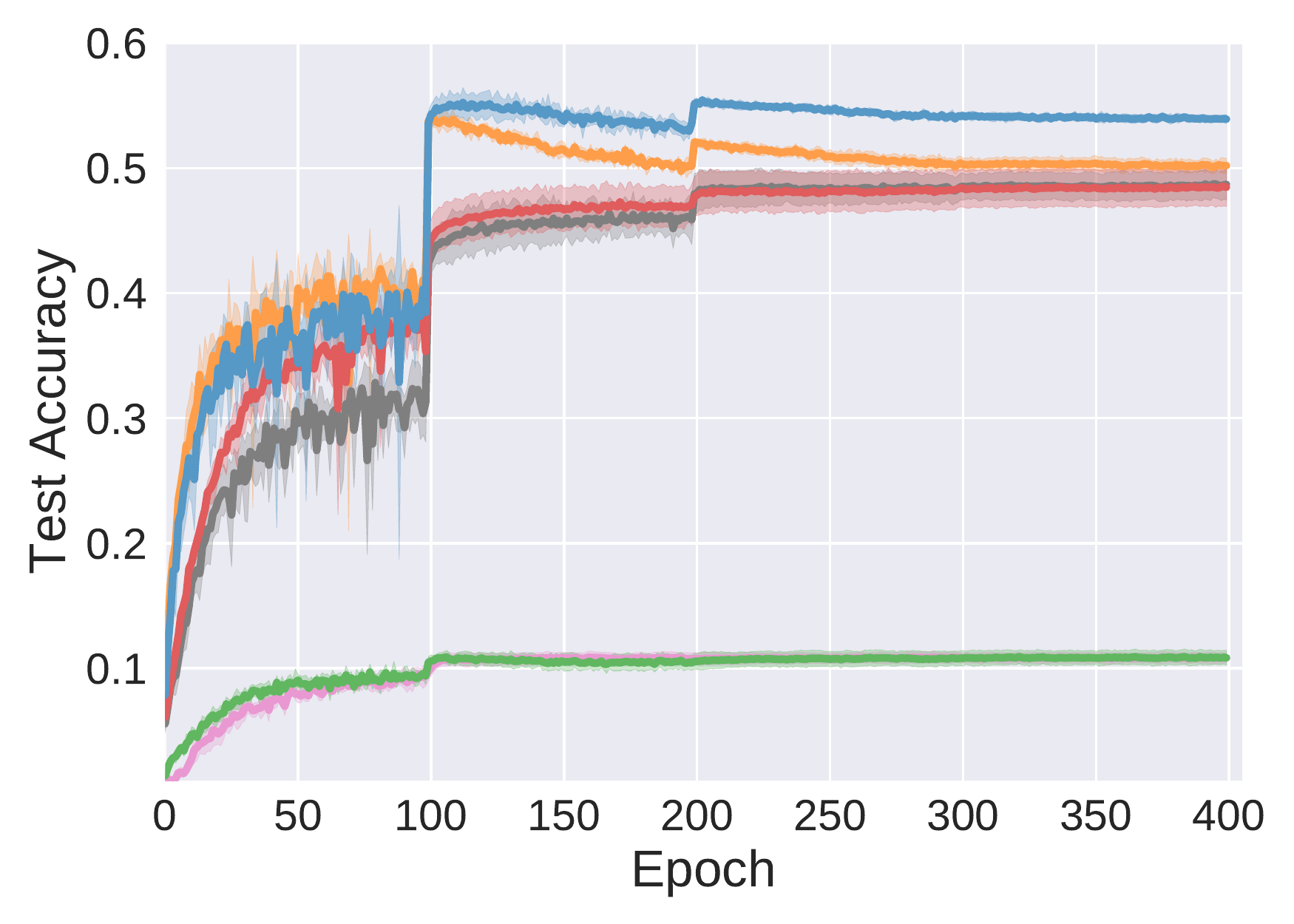}%
        \includegraphics[width=0.333\textwidth]{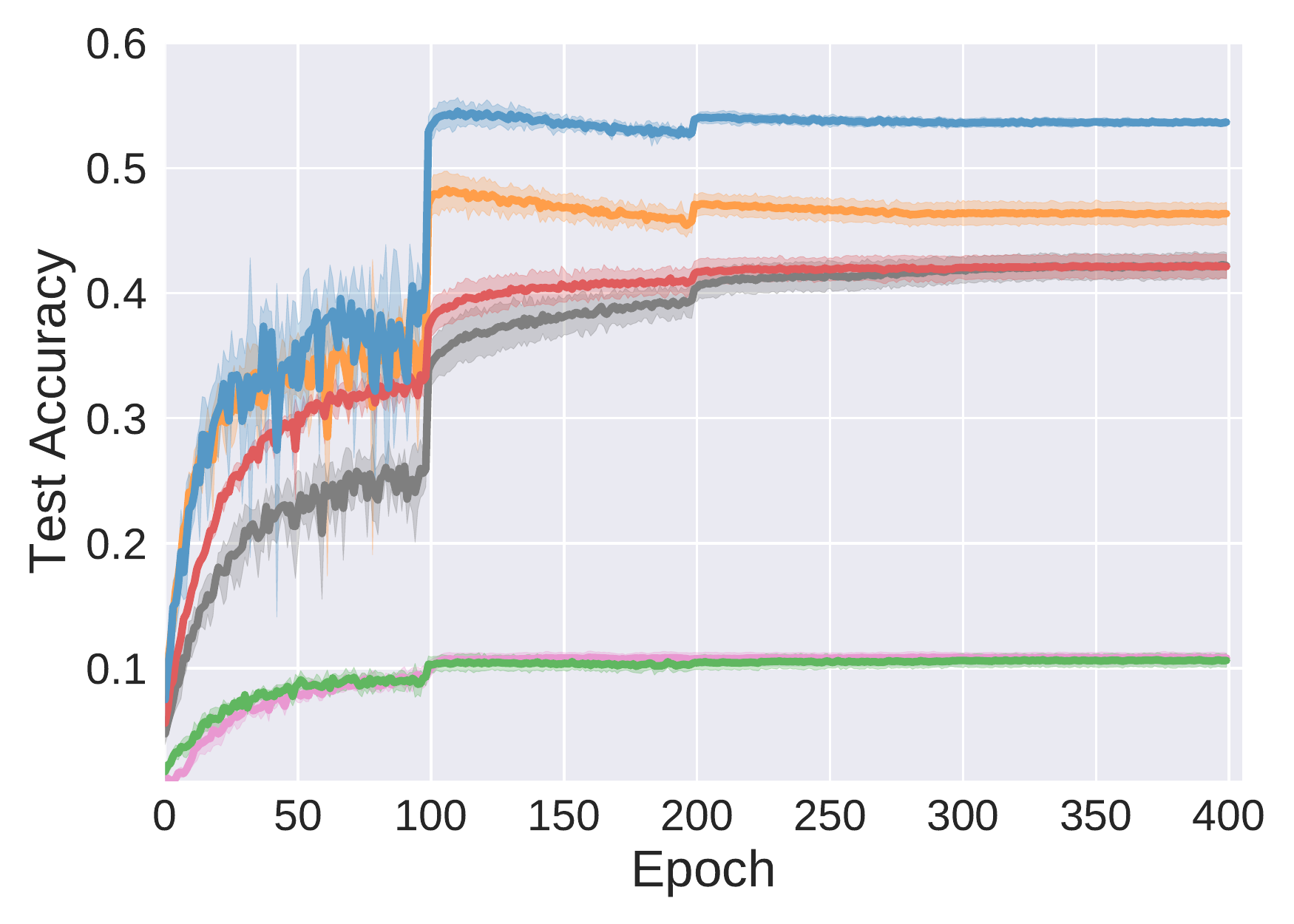}%
        \includegraphics[width=0.333\textwidth]{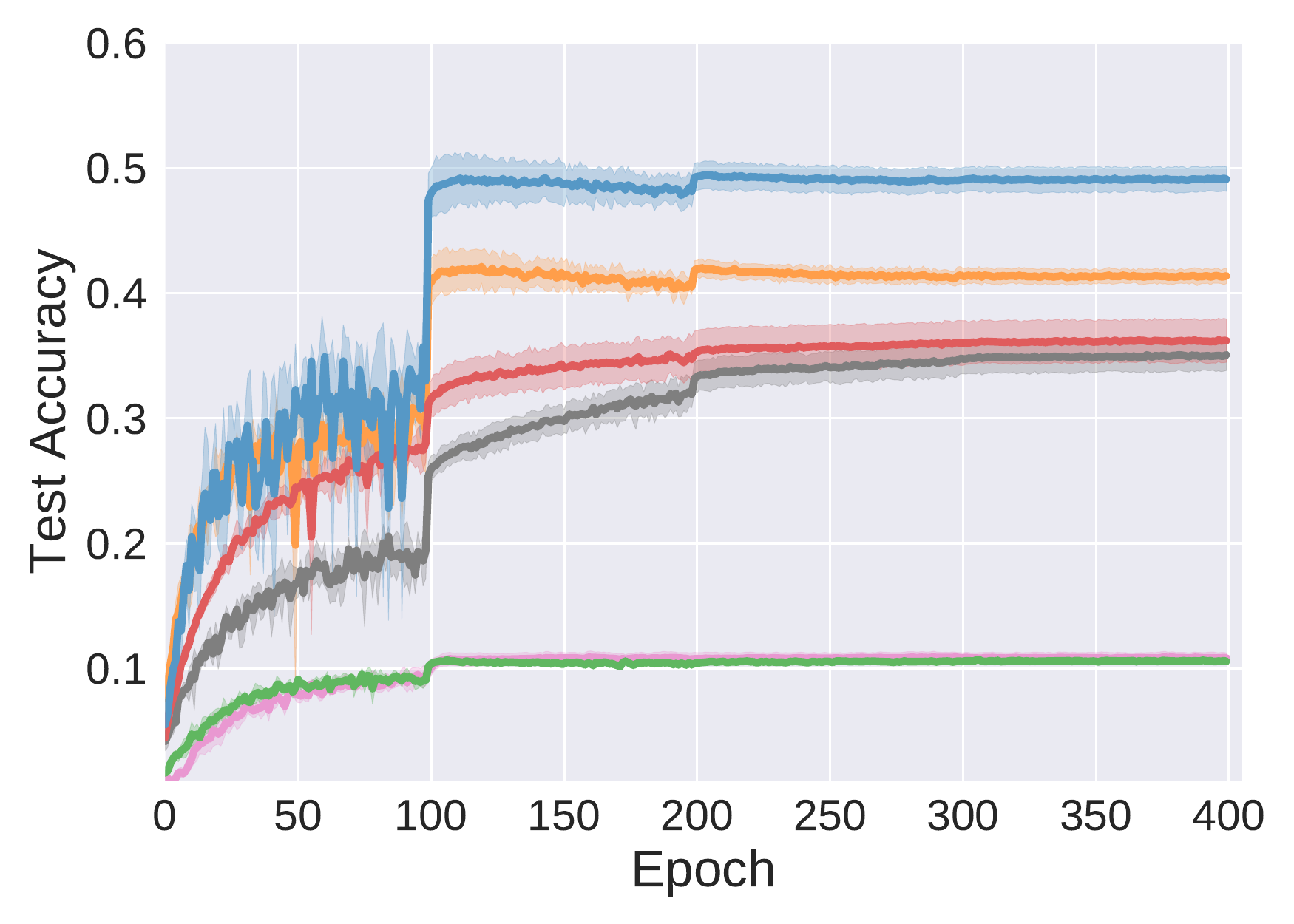}%
    \end{minipage}
    \caption{Experimental results on Fashion-MNIST and CIFAR-10/100 under label noise (5 trials).}
    \label{fig:performance}
\end{figure}

\subsubsection{Baselines}
There are five baseline methods involved in our experiments:
\begin{itemize}
    \item \emph{Clean} discards the training data and uses the validation data for training;
    \item \emph{Uniform} does not weight the training data, i.e., the weights are all ones;
    \item \emph{Random} draws random weights following the \emph{rectified Gaussian distribution};
    \item \emph{IW} is kernel mean matching without any non-linear transformation \cite{huang2007correcting};
    \item \emph{Reweight} is learning to reweight \cite{ren2018learning}.
\end{itemize}
All baselines were implemented with PyTorch.%
\footnote{We implemented Reweight to ensure same random samplings of data and initialization of models.}
Note that in each mini-batch, DIW computes $\mathcal{W}$ and then updates $\mf$, while Reweight updates $\mf$ and then updates $\mathcal{W}$.
Moreover, Reweight updates $\mathcal{W}$ in epoch one, while DIW pretrains $\mf$ in epoch one to equally go over all the training data once.

\subsubsection{Setup}
The experiments were based on three widely used benchmark datasets \emph{Fashion-MNIST} \cite{xiao2017}, \emph{CIFAR-10}, and \emph{CIFAR-100} \cite{krizhevsky2009learning}.
For the set of validation data,
\begin{itemize}
    \item 1,000 random clean data in total were used in the label-noise experiments;
    \item 10 random data per class were used in the class-prior-shift experiments.
\end{itemize}
The validation data were included in the training data, as required by Reweight. Then,
\begin{itemize}
    \item for Fashion-MNIST, LeNet-5 \cite{lecun1998gradient} was trained by SGD \cite{robbins1951stochastic};
    \item for CIFAR-10/100, ResNet-32 \cite{he2016deep} was trained by Adam \cite{kingma15iclr}.
\end{itemize}
For fair comparisons, we normalized $\mathcal{W}$ to make the average weight $1$ in each mini-batch.
For clear comparisons, there was no data augmentation.
More details can be found in \cite{fang2020rethinking}.

\subsubsection{Label-noise Experiments}
Two major class-conditional noise models were considered:
\begin{itemize}
    \item \emph{pair flip} \cite{han2018co}, where a label $j$, if it gets mislabeled, must flip to class $(j\bmod k+1)$;
    \item \emph{symmetric flip} \cite{van2015learning}, where a label may flip to all other classes with equal probability.
\end{itemize}
We set the noise rate as 0.3 for pair flip and 0.4 or 0.5 for symmetric flip.
The experimental results are reported in Figure~\ref{fig:performance}.
We can see that DIW outperforms the baselines.
As the noise rate increases, DIW stays reasonably robust and the baselines tend to overfit the noisy labels.

\begin{figure}[t]
    \centering
    \begin{minipage}[c]{0.02\textwidth}~\end{minipage}%
    \begin{minipage}[c]{0.326\textwidth}\centering\small IW \end{minipage}%
    \begin{minipage}[c]{0.326\textwidth}\centering\small Reweight \end{minipage}%
    \begin{minipage}[c]{0.326\textwidth}\centering\small DIW \end{minipage}\\
    \begin{minipage}[c]{0.02\textwidth}\small \rotatebox{90}{Box plots} \end{minipage}%
    \begin{minipage}[c]{0.98\textwidth}
        \includegraphics[width=0.33\textwidth]{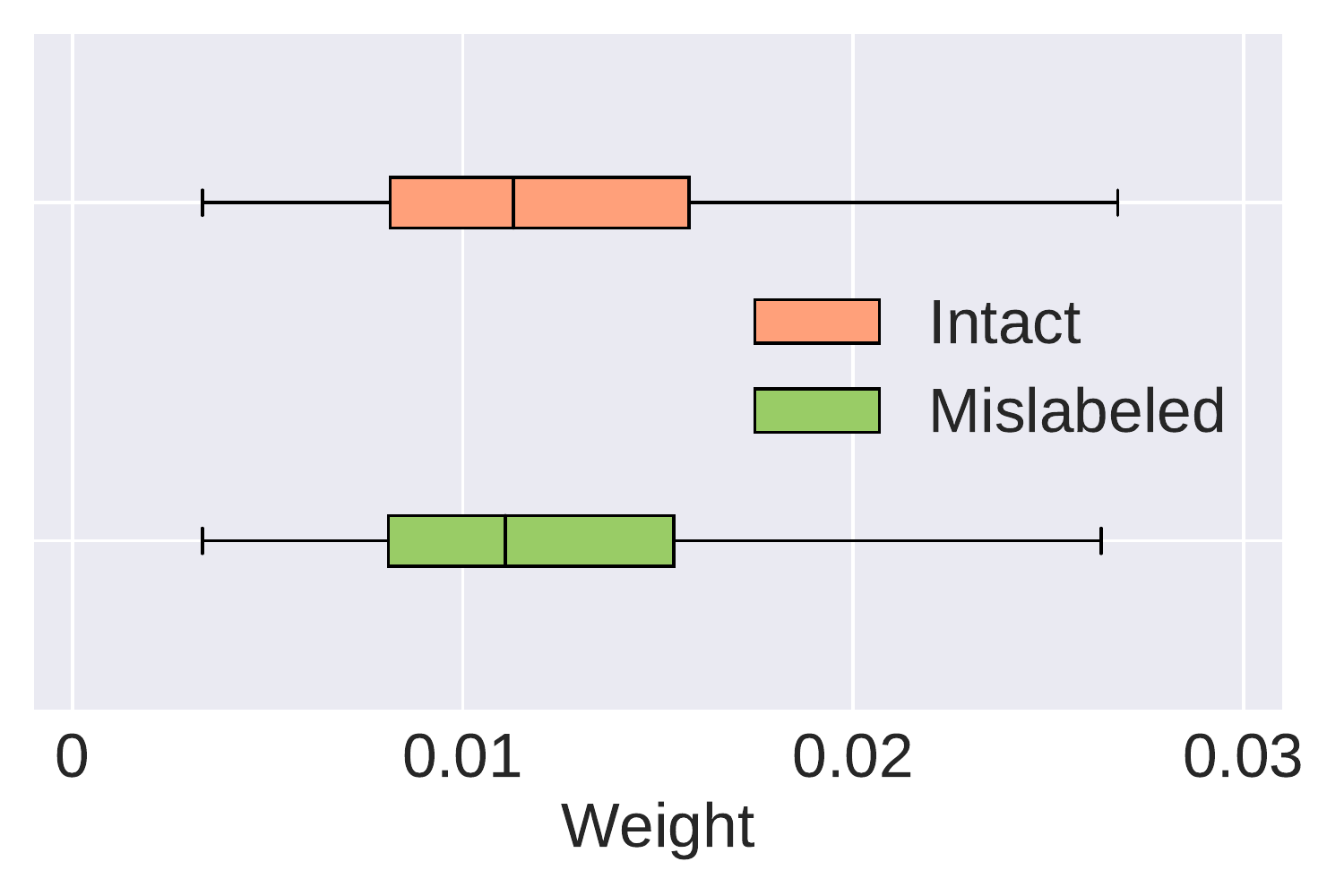}%
        \includegraphics[width=0.33\textwidth]{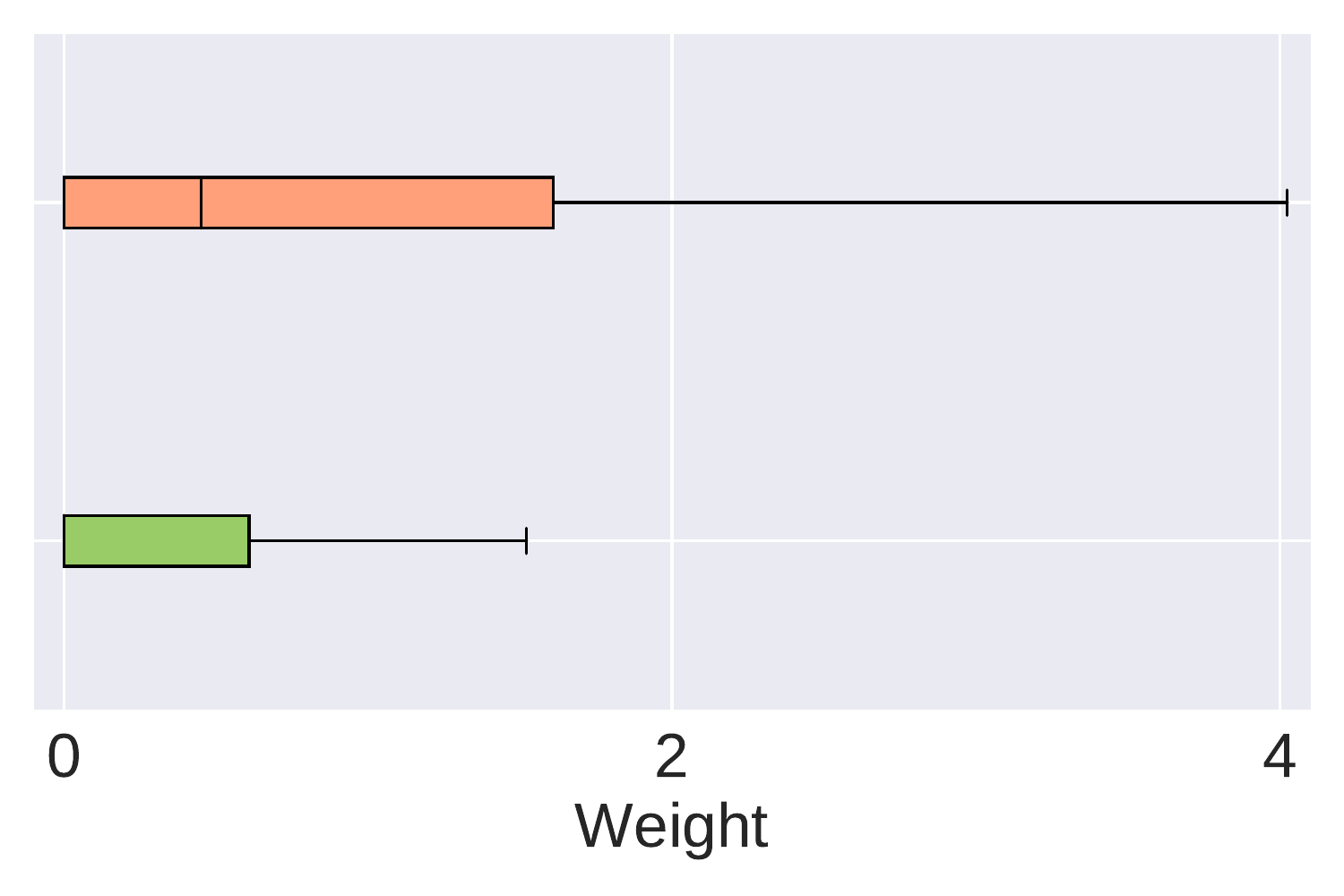}
        \includegraphics[width=0.33\textwidth]{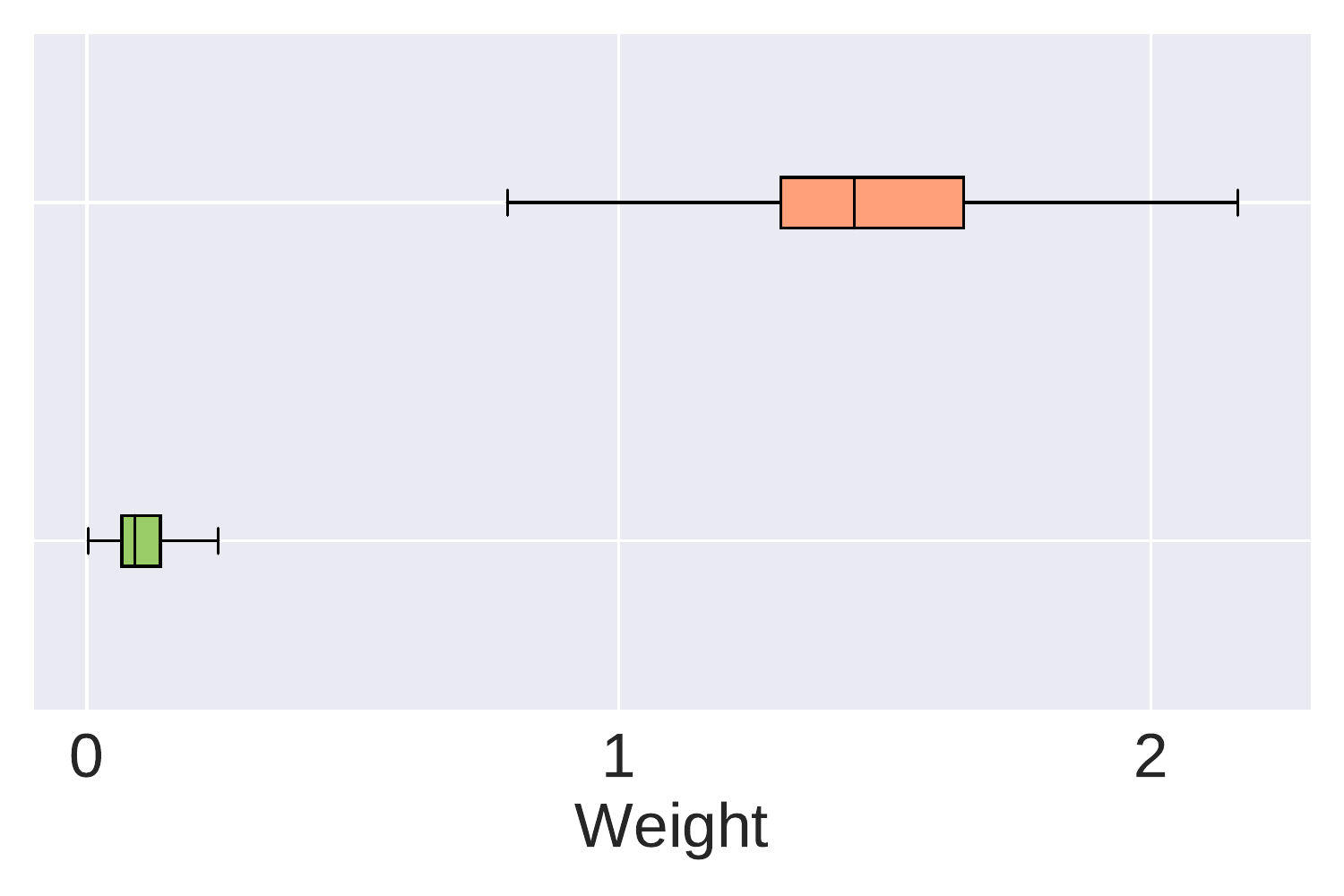}
    \end{minipage}\\
    \begin{minipage}[c]{0.02\textwidth}\small \rotatebox{90}{Histogram plots} \end{minipage}%
    \begin{minipage}[c]{0.98\textwidth}
        \includegraphics[width=0.33\textwidth]{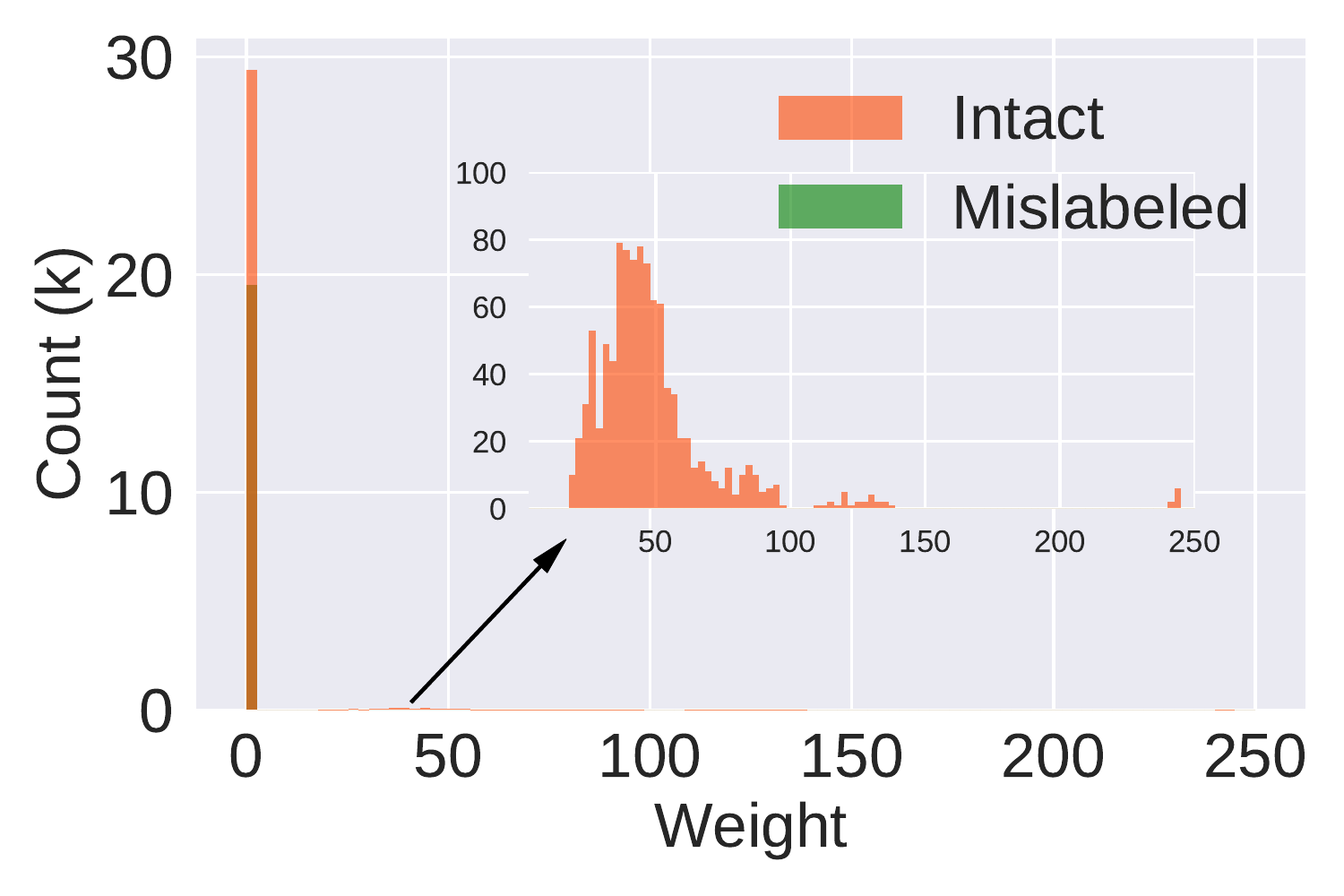}%
        \includegraphics[width=0.33\textwidth]{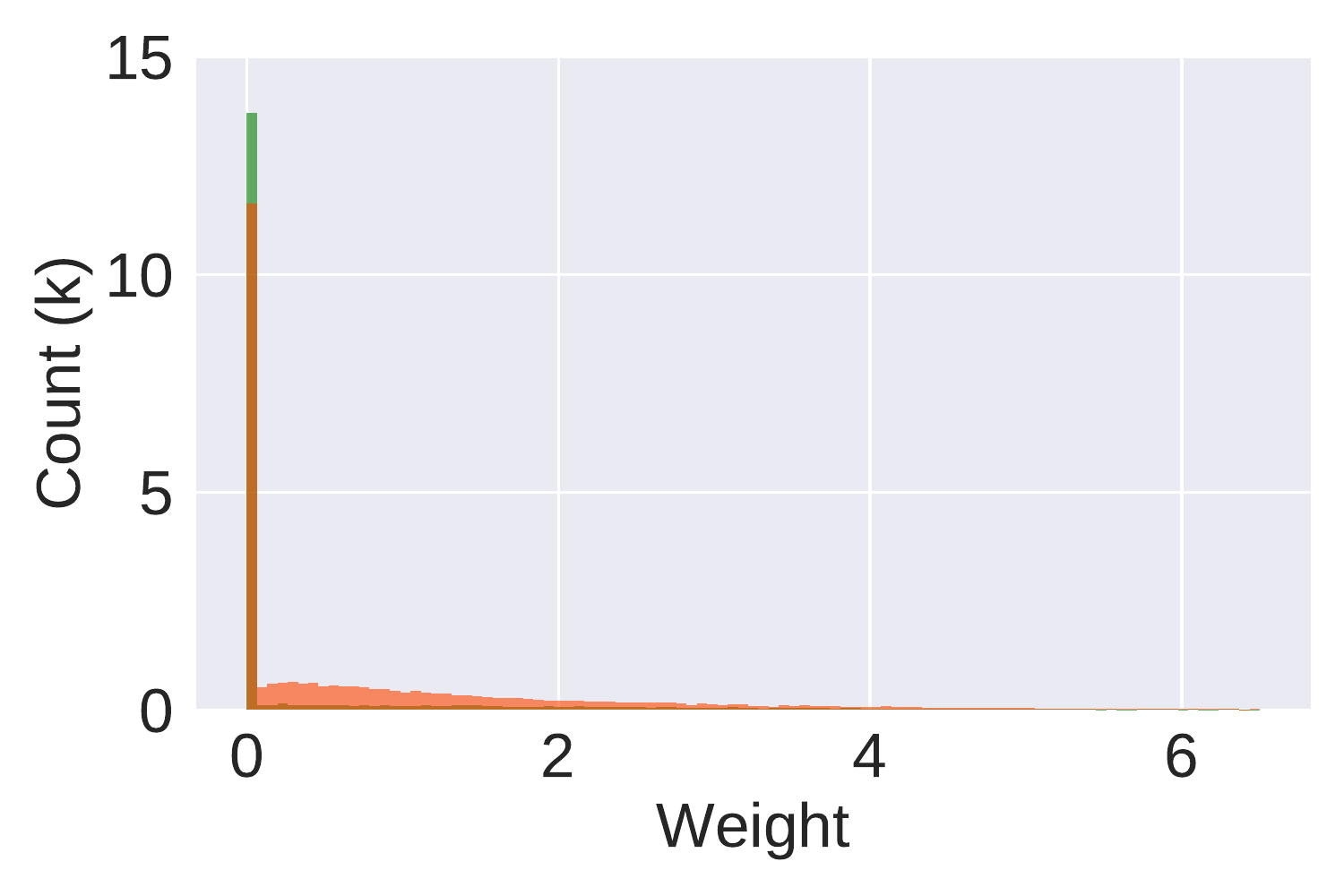}
        \includegraphics[width=0.33\textwidth]{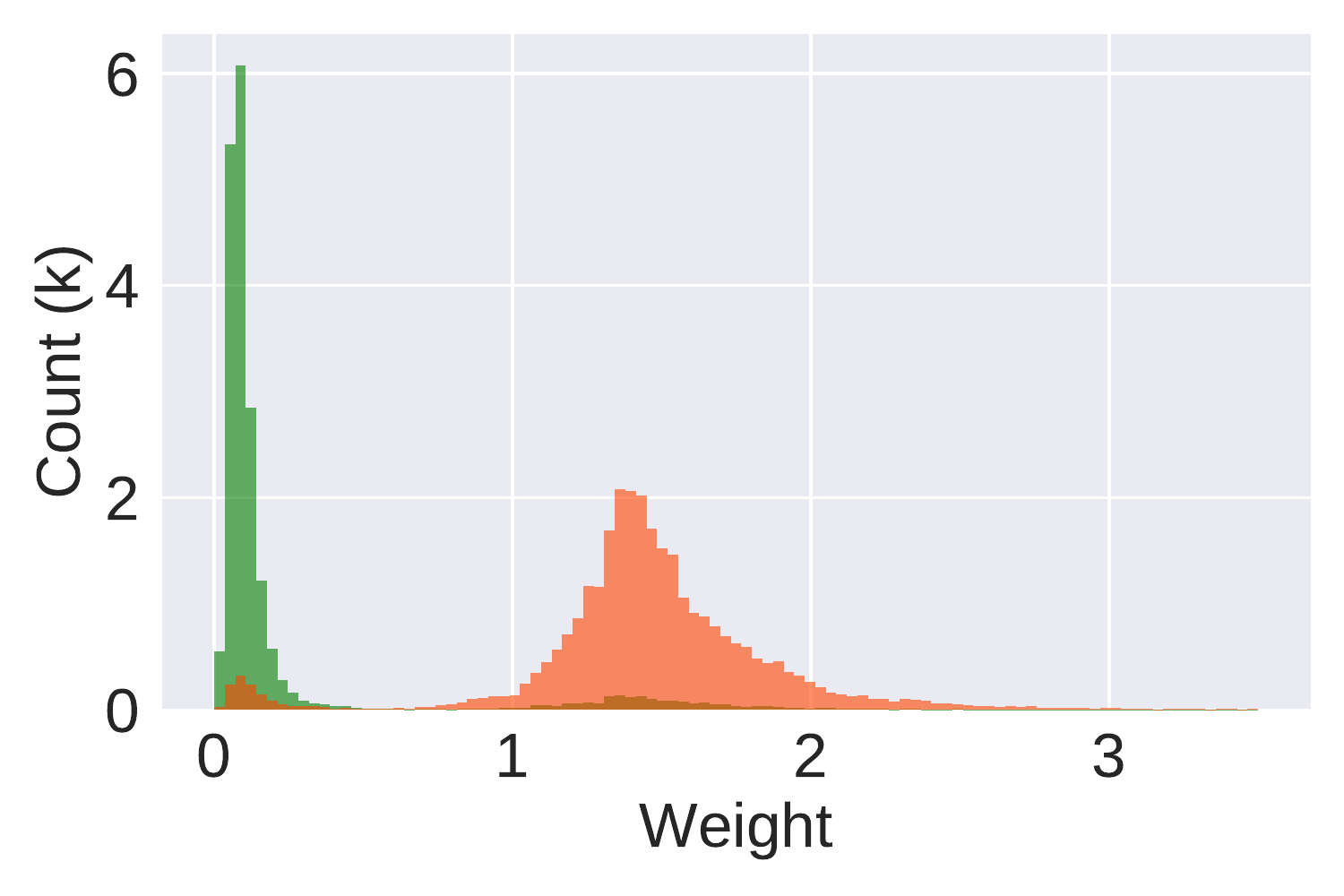}
    \end{minipage}
    \caption{Statistics of weight distributions on CIFAR-10 under 0.4 symmetric flip.}
    \label{fig:weightdist}
\end{figure}

To better understand how DIW contributes to learning robust models, we take a look at the learned weights in the final epoch.
As shown in Figure~\ref{fig:weightdist}, DIW can successfully identify intact/mislabeled training data and automatically up-/down-weight them, while others cannot effectively identify them.
This confirms that DIW can improve the weights and thus reduce the bias of the model.

\subsubsection{Class-prior-shift Experiments}
We imposed class-prior shift on Fashion-MNIST following \cite{buda2018systematic}:
\begin{itemize}
    \item the classes were divided into majority classes and minority classes, where the fraction of the minority classes is $\mu<1$;
    \item the training data were drawn from every majority class using a sample size, and from every minority class using another sample size, where the ratio of these two sample sizes was $\rho>1$;
    \item the test data were evenly sampled form all classes.
\end{itemize}
We fixed $\mu=0.2$ and tried $\rho=100$ and $\rho=200$.
A new baseline \emph{Truth} was added for reference, where the true weights were used, i.e., $1-\mu+\mu/\rho$ and $\mu+\rho-\mu\rho$ for the majority/minority classes.

The experimental results are reported in Table~\ref{tab:ib}, where we can see that DIW again outperforms the baselines.
Table~\ref{tab: dist} contains the \emph{mean absolute error}~(MAE) and \emph{root mean square error}~(RMSE) from the weights learned by IW, Reweight, and DIW to the true weights, as the unit test under class-prior shift.
The results confirm that the weights learned by DIW are closer to the true weights.

\begin{table}[t]
    \begin{minipage}[t]{0.46\textwidth}
        \centering\small
        \caption{Mean accuracy (standard deviation) in percentage on Fashion-MNIST under class-prior shift (5 trials). Best and comparable methods (paired \textit{t}-test at significance level 5\%) are highlighted in bold.}
        \label{tab:ib}
        \begin{tabular}{ccc}
            \toprule
            & $\rho=100$ & $\rho=200$\\
            \midrule
            Clean & \makecell{63.38 (2.59)} & \makecell{63.38 (2.59)} \\
            Uniform & \makecell{\textbf{83.48 (1.26)}} & \makecell{79.12 (1.18)} \\
            Random & \makecell{\textbf{83.11 (1.70)}} & \makecell{79.38 (0.96)} \\
            IW & \makecell{\textbf{83.45} \textbf{(1.10)}} & \makecell{\textbf{80.25} \textbf{(2.23)}} \\
            Reweight & \makecell{81.96 (1.74)} & \makecell{\textbf{79.37 (2.38)}} \\
            DIW & \makecell{\textbf{84.02 (1.82)}} & \makecell{\textbf{81.37 (0.95)}} \\
            \midrule
            Truth & \makecell{\textbf{83.29 (1.11)}} & \makecell{\textbf{80.22 (2.13)}} \\
            \bottomrule
        \end{tabular}
    \end{minipage}\hfill
    \begin{minipage}[t]{0.50\textwidth}
        \centering\small
        \caption{Mean distance (standard deviation) from the learned weights to the true weights on Fashion-MNIST under class-prior shift (5 trials). Best and comparable methods (paired \textit{t}-test at significance level 5\%) are highlighted in bold.}
        \label{tab: dist}
        \begin{tabular}{ccc}
            \toprule
            $\rho=100$ & MAE & RMSE\\
            \midrule
            IW & \makecell{1.10 (0.03)} & \makecell{10.19 (0.33)}\\
            Reweight & \makecell{1.66 (0.02)} & \makecell{5.65 (0.20)}\\
            DIW & \makecell{\textbf{0.45} \textbf{(0.02)}} & \makecell{\textbf{3.19} \textbf{(0.07)}}\\
            \midrule
            $\rho=200$ & MAE & RMSE\\
            \midrule
            IW & \makecell{1.03 (0.04)} & \makecell{9.99 (0.38)}\\
            Reweight & \makecell{1.64 (0.05)} & \makecell{6.07 (0.86)}\\
            DIW & \makecell{\textbf{0.46} \textbf{(0.06)}} & \makecell{\textbf{3.67} \textbf{(0.13)}}\\
            \bottomrule
        \end{tabular}
    \end{minipage}
\end{table}

\subsubsection{Ablation Study}
As shown in Figure~\ref{fig:siwvsdiw}, DIW comprises many options, which means that DIW can have a complicated algorithm design.
Starting from IW,
\begin{itemize}
    \item introducing feature extractor~(FE) yields SIW;
    \item based on SIW, updating $\mathcal{W}$ yields DIW1;
    \item based on DIW1, updating FE yields DIW2;
    \item based on DIW2, pretraining FE yields DIW3.
\end{itemize}
We compared them under label noise and report the results in Table~\ref{tab: ablation}, where the ``-F'' or ``-L'' suffix means using the hidden-layer-output or loss-value transformation.
In general, we can observe
\begin{itemize}
    \item SIWs improve upon IW due to the introduction of FE;
    \item DIWs improve upon SIWs due to the dynamic nature of $\mathcal{W}$ in DIWs;
    \item for DIWs with a pretrained FE (i.e., DIW1 and DIW3), updating the FE during training is usually better than fixing it throughout training;
    \item for DIWs whose FE is updated (i.e., DIW2 and DIW3), ``-F'' methods perform better when FE is pretrained, while ``-L'' methods do not necessarily need to pretrain FE.
\end{itemize}
Therefore, DIW2-L is more recommended, which was indeed used in the previous experiments.

\begin{table}
\begin{center}
    \small
    \caption{Mean accuracy (standard deviation) in percentage on Fashion-MNIST (F-MNIST for short) and CIFAR-10/100 under label noise (5 trials). Best and comparable methods (paired \textit{t}-test at significance level 5\%) are highlighted in bold. p/s is short for pair/symmetric flip.}
    \label{tab: ablation}
    \resizebox{\textwidth}{!}{
    \begin{tabular}{ccccccccccc}
        \toprule
        & Noise & IW & SIW-F & SIW-L & DIW1-F & DIW2-F & DIW3-F & DIW1-L & DIW2-L & DIW3-L\\
        \midrule
        \multirow{5}{*}{\rotatebox{90}{F-MNIST}} & 0.3 p &  \makecell{82.69 \\ (0.38)} &  \makecell{82.41 \\ (0.46)} & \makecell{85.46 \\ (0.29)} & \makecell{87.60 \\ (0.07)} & \makecell{\textbf{87.67} \\ \textbf{(0.37)}} & \makecell{87.54 \\ (0.25)} & \makecell{87.04 \\ (0.51)} & \makecell{\textbf{88.19} \\ \textbf{(0.43)}} & \makecell{86.68 \\ (1.42)} \\
        & 0.4 s &  \makecell{80.54 \\ (0.66)} &  \makecell{82.36 \\ (0.65)} & \makecell{\textbf{88.68} \\ \textbf{(0.23)}} & \makecell{87.45 \\ (0.22)} & \makecell{87.04 \\ (0.30)}  & \makecell{88.29 \\ (0.16)} & \makecell{\textbf{88.98} \\ \textbf{(0.19)}} & \makecell{88.29 \\ (0.18)} & \makecell{87.89 \\ (0.43)} \\
        & 0.5 s &  \makecell{78.90 \\ (0.97)} &  \makecell{81.29 \\ (0.68)} & \makecell{\textbf{87.49} \\ \textbf{(0.23)}} & \makecell{\textbf{87.27} \\ \textbf{(0.38)}} & \makecell{86.41 \\ (0.36)}  & \makecell{87.28 \\ (0.18)} & \makecell{\textbf{87.70} \\ \textbf{(0.15)}} & \makecell{\textbf{87.67} \\ \textbf{(0.57)}} & \makecell{\textbf{86.74} \\ \textbf{(1.19)}}\\
        \midrule
        \multirow{5}{*}{\rotatebox{90}{CIFAR-10}} & 0.3 p &  \makecell{45.02 \\ (2.25)} &  \makecell{74.61 \\ (0.51)} & \makecell{80.45 \\ (0.89)} & \makecell{82.75 \\ (0.57)} &  \makecell{81.19 \\ (0.81)} & \makecell{81.76 \\ (0.70)} & \makecell{81.73 \\ (0.54)} & \makecell{\textbf{84.44} \\ \textbf{(0.70)}} & \makecell{\textbf{83.80} \\ \textbf{(0.93)}} \\
        & 0.4 s &  \makecell{44.31 \\ (2.14)} &  \makecell{65.58 \\ (0.82)} & \makecell{76.39 \\ (0.72)} & \makecell{78.23 \\ (0.69)} &  \makecell{77.48 \\ (0.60)}  & \makecell{78.75 \\ (0.45)} &\makecell{75.27 \\ (1.37)} & \makecell{\textbf{80.40} \\ \textbf{(0.69)}} & \makecell{\textbf{80.10} \\ \textbf{(0.58)}} \\
        & 0.5 s &  \makecell{42.84 \\ (2.35)} &  \makecell{62.81 \\ (1.29)} & \makecell{71.47 \\ (1.47)} & \makecell{74.20 \\ (0.81)} &  \makecell{73.98 \\ (1.29)}  & \makecell{\textbf{76.38} \\ \textbf{(0.53)}} &\makecell{69.67 \\ (1.73)} & \makecell{\textbf{76.26} \\ \textbf{(0.73)}} & \makecell{\textbf{76.86} \\ \textbf{(0.44)}}\\
        \midrule
        \multirow{5}{*}{\rotatebox{90}{CIFAR-100$^*$}} & 0.3 p &  \makecell{10.85 \\ (0.59)} &  \makecell{10.44 \\ (0.63)} & \makecell{45.43 \\ (0.71)} & -- & -- & -- & \makecell{51.90 \\ (1.11)} & \makecell{\textbf{53.94} \\ \textbf{(0.29)}} & \makecell{\textbf{54.01} \\ \textbf{(0.93)}} \\
        & 0.4 s &  \makecell{10.61 \\ (0.53)} &  \makecell{11.70 \\ (0.48)} & \makecell{47.40 \\ (0.34)} & -- & -- & -- & \makecell{50.99 \\ (0.16)}& \makecell{\textbf{53.66} \\ \textbf{(0.28)}}& \makecell{\textbf{53.07} \\ \textbf{(0.32)}} \\
        & 0.5 s &  \makecell{10.58 \\ (0.17)} &  \makecell{13.26 \\ (0.69)} & \makecell{41.74 \\ (1.68)} & -- & -- & -- & \makecell{46.25 \\ (0.60)} & \makecell{\textbf{49.13} \\ \textbf{(0.98)}} & \makecell{\textbf{49.11} \\ \textbf{(0.90)}}\\
        \bottomrule
    \end{tabular}}\\
\end{center}
{\small$^*$ Note: ``-F'' methods for DIW are not applicable on CIFAR-100, since there are too few data in a class in a mini-batch.}
\end{table}

\begin{figure}[!t]
    \centering
    \subcaptionbox{IW}{\includegraphics[width=0.19\textwidth]{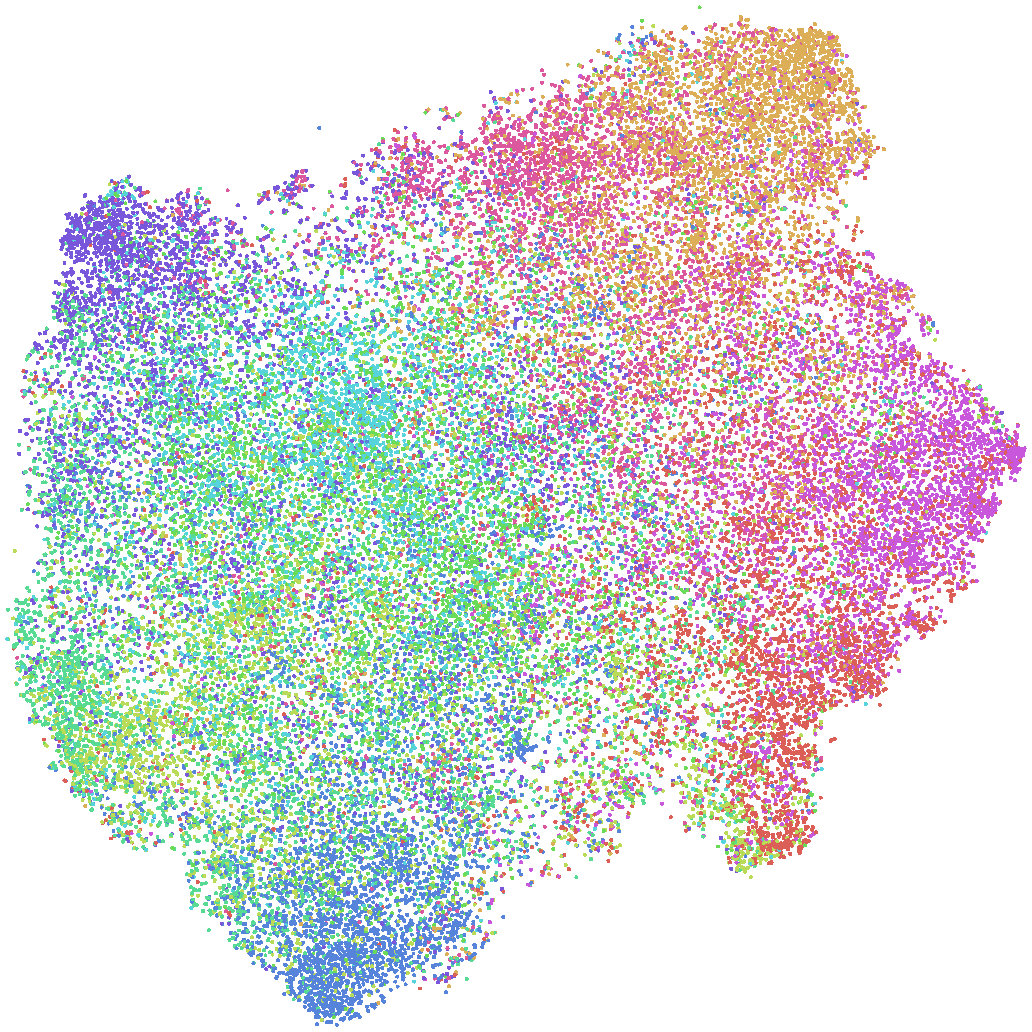}}\hfill
    \subcaptionbox{SIW-F}{\includegraphics[width=0.19\textwidth]{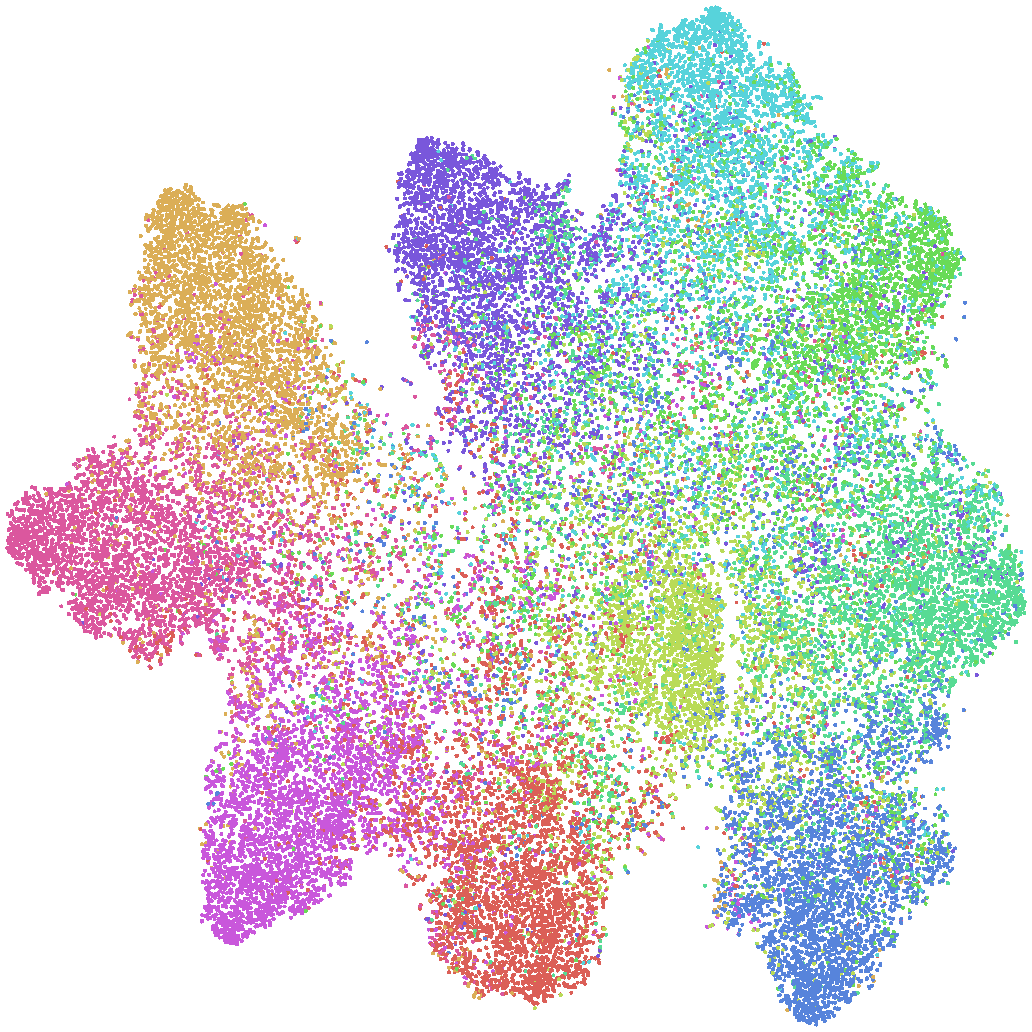}}\hfill
    \subcaptionbox{SIW-L}{\includegraphics[width=0.19\textwidth]{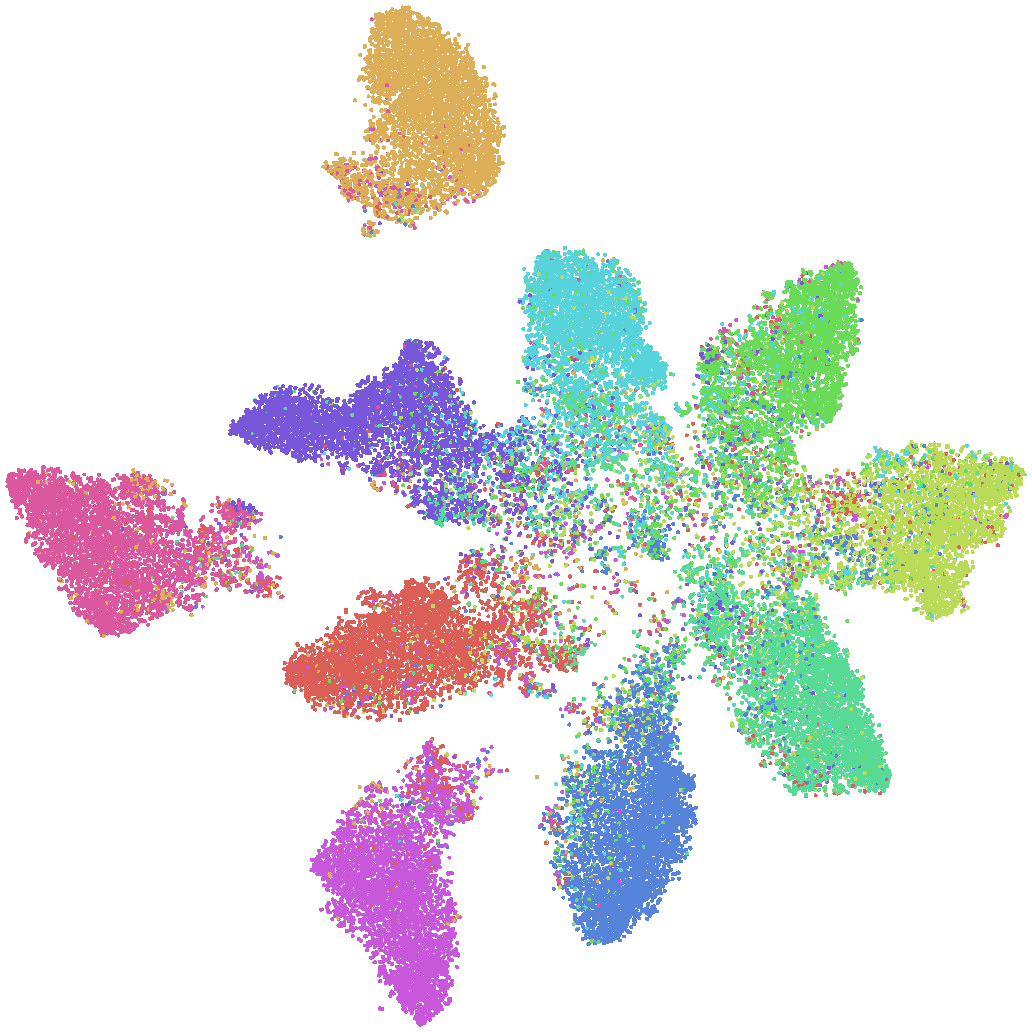}}\hfill
    \subcaptionbox{Reweight}{\includegraphics[width=0.19\textwidth]{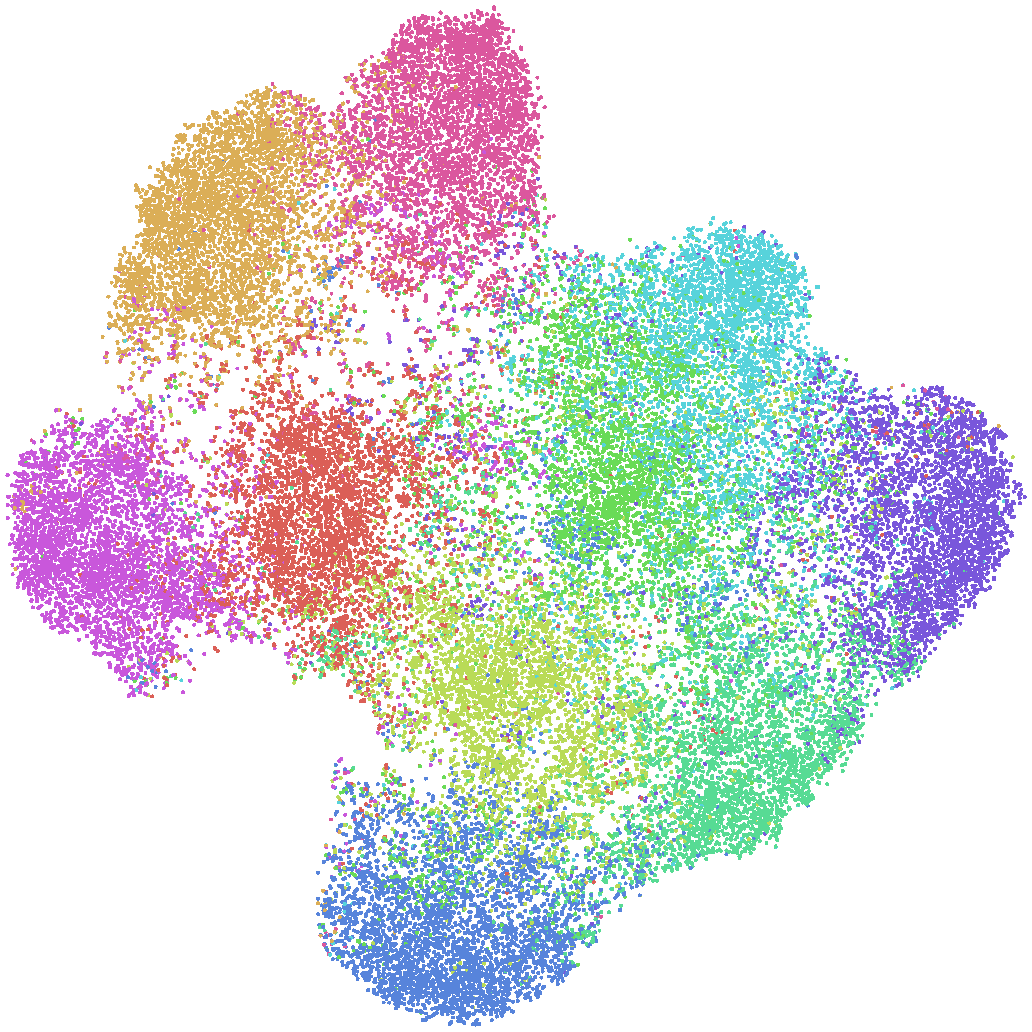}}\hfill
    \subcaptionbox{DIW1-F}{\includegraphics[width=0.19\textwidth]{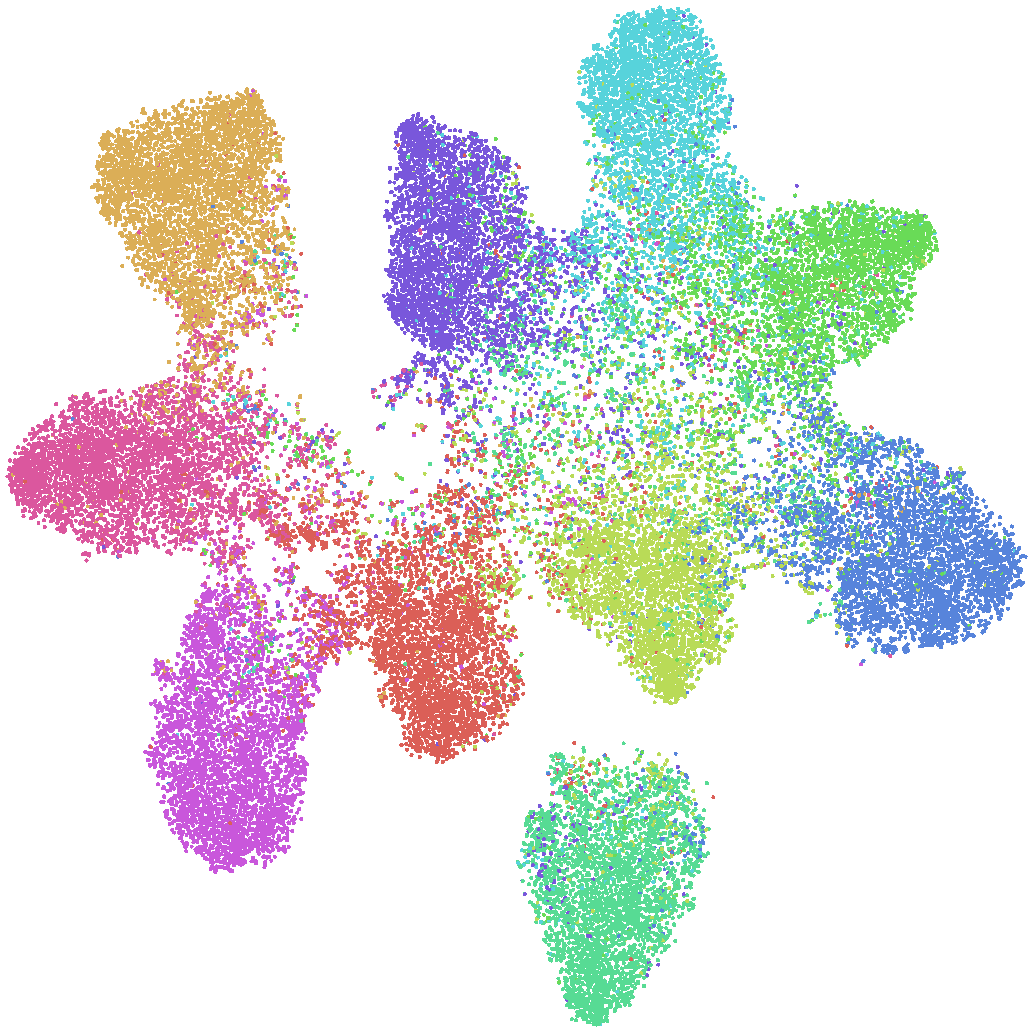}}\\
    \subcaptionbox{DIW2-F}{\includegraphics[width=0.19\textwidth]{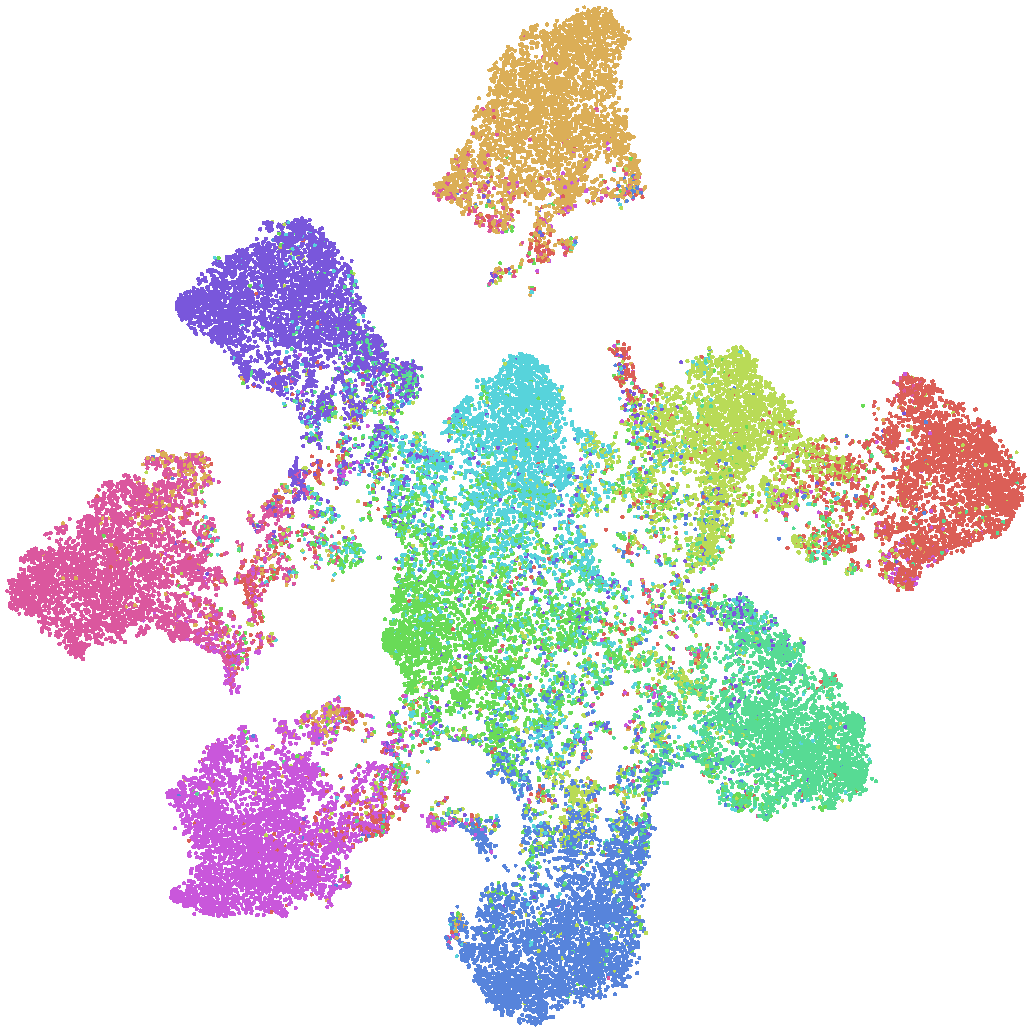}}\hfill
    \subcaptionbox{DIW3-F}{\includegraphics[width=0.19\textwidth]{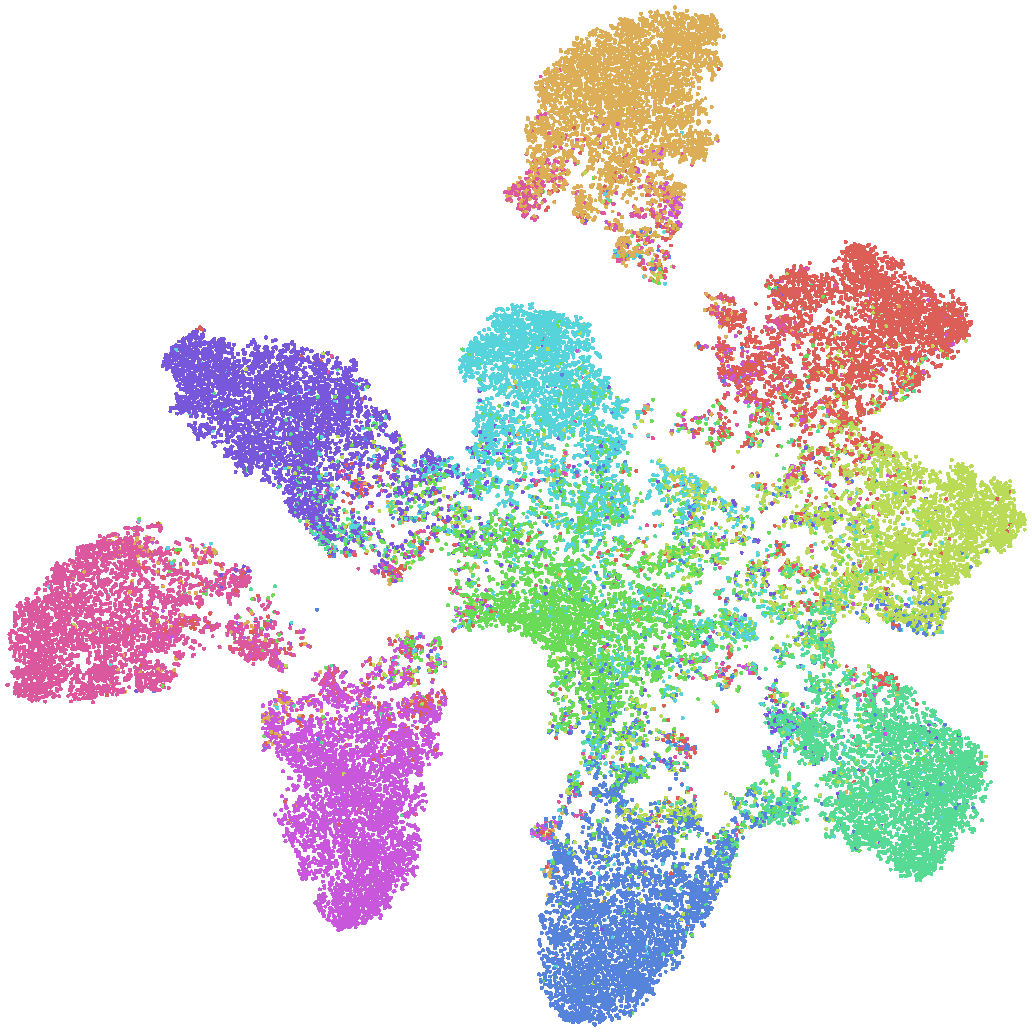}}\hfill
    \subcaptionbox{DIW1-L}{\includegraphics[width=0.19\textwidth]{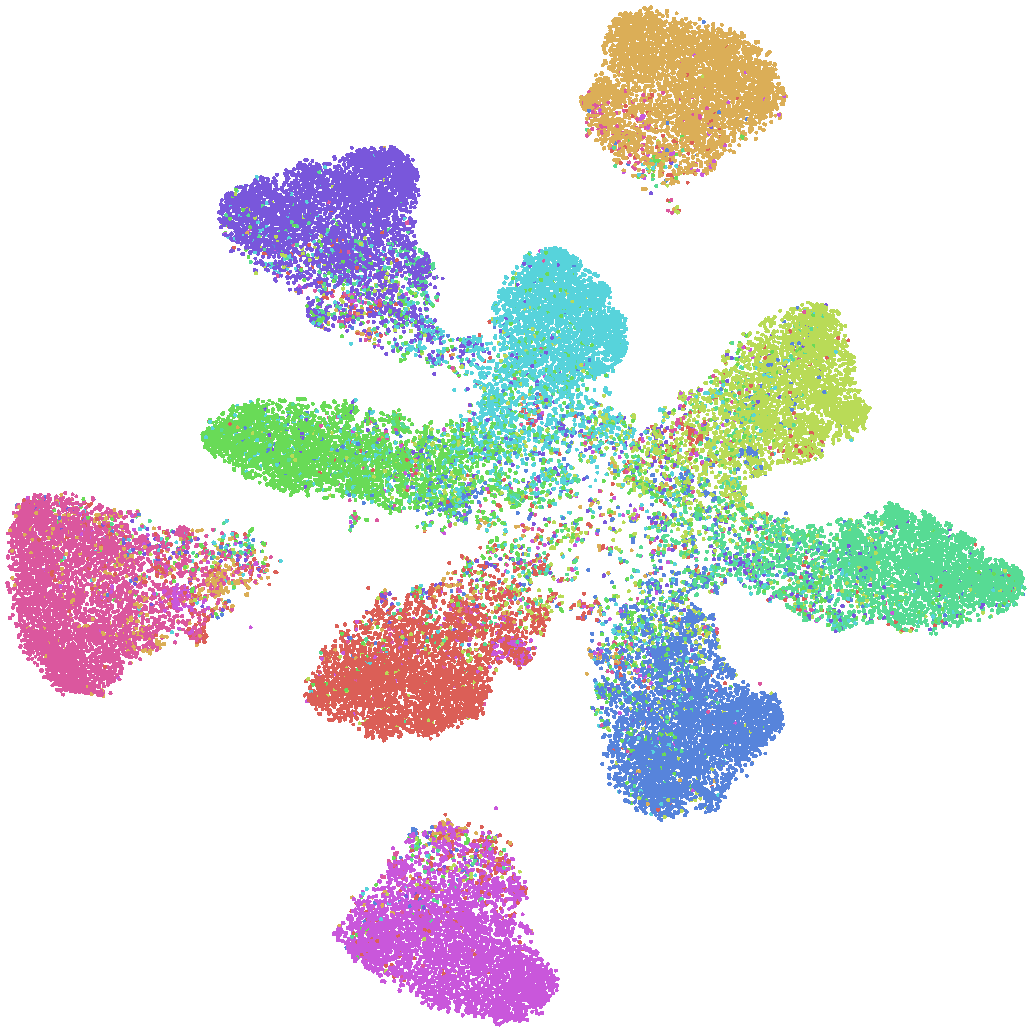}}\hfill
    \subcaptionbox{DIW2-L}{\includegraphics[width=0.19\textwidth]{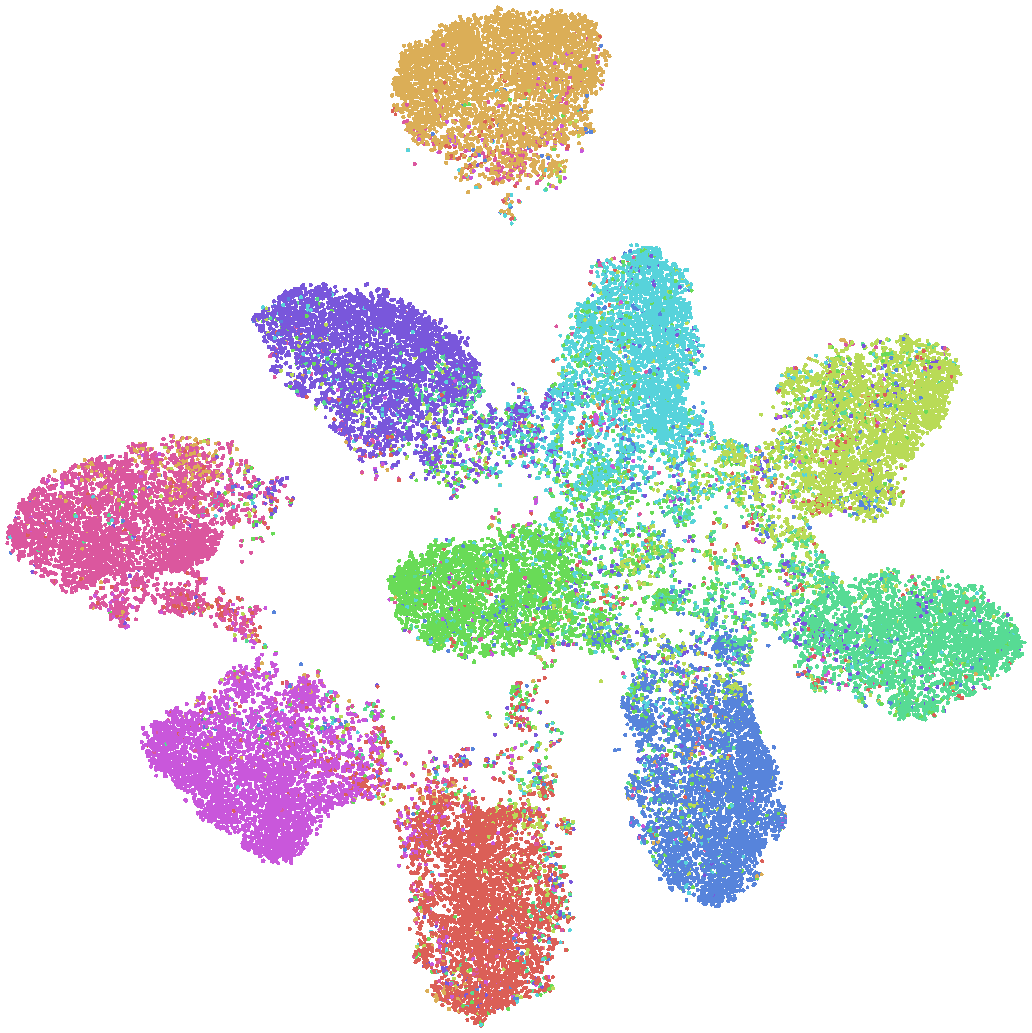}}\hfill
    \subcaptionbox{DIW3-L}{\includegraphics[width=0.19\textwidth]{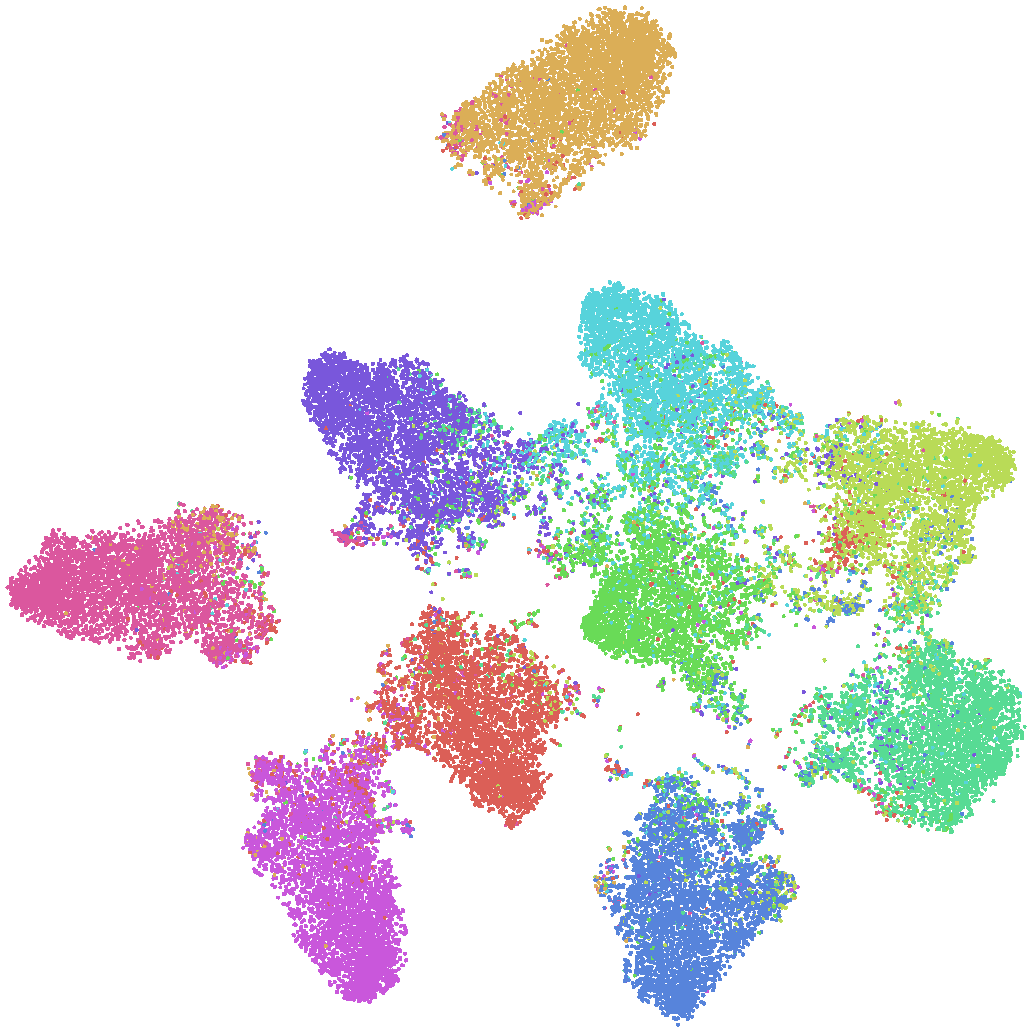}}
    \caption{Visualizations of embedded data on noisy CIFAR-10 (colors mean ground-truth labels).}
    \label{fig: tsne}
\end{figure}

\begin{figure}[!t]
    \centering
    \begin{minipage}[c]{0.02\textwidth}~\end{minipage}%
    \begin{minipage}[c]{0.326\textwidth}\centering\small 0.3 pair \end{minipage}%
    \begin{minipage}[c]{0.326\textwidth}\centering\small 0.4 symmetric \end{minipage}%
    \begin{minipage}[c]{0.326\textwidth}\centering\small 0.5 symmetric \end{minipage}\\
    \begin{minipage}[c]{0.02\textwidth}\small \rotatebox{90}{CIFAR-10} \end{minipage}%
    \begin{minipage}[c]{0.98\textwidth}
        \includegraphics[width=0.333\textwidth]{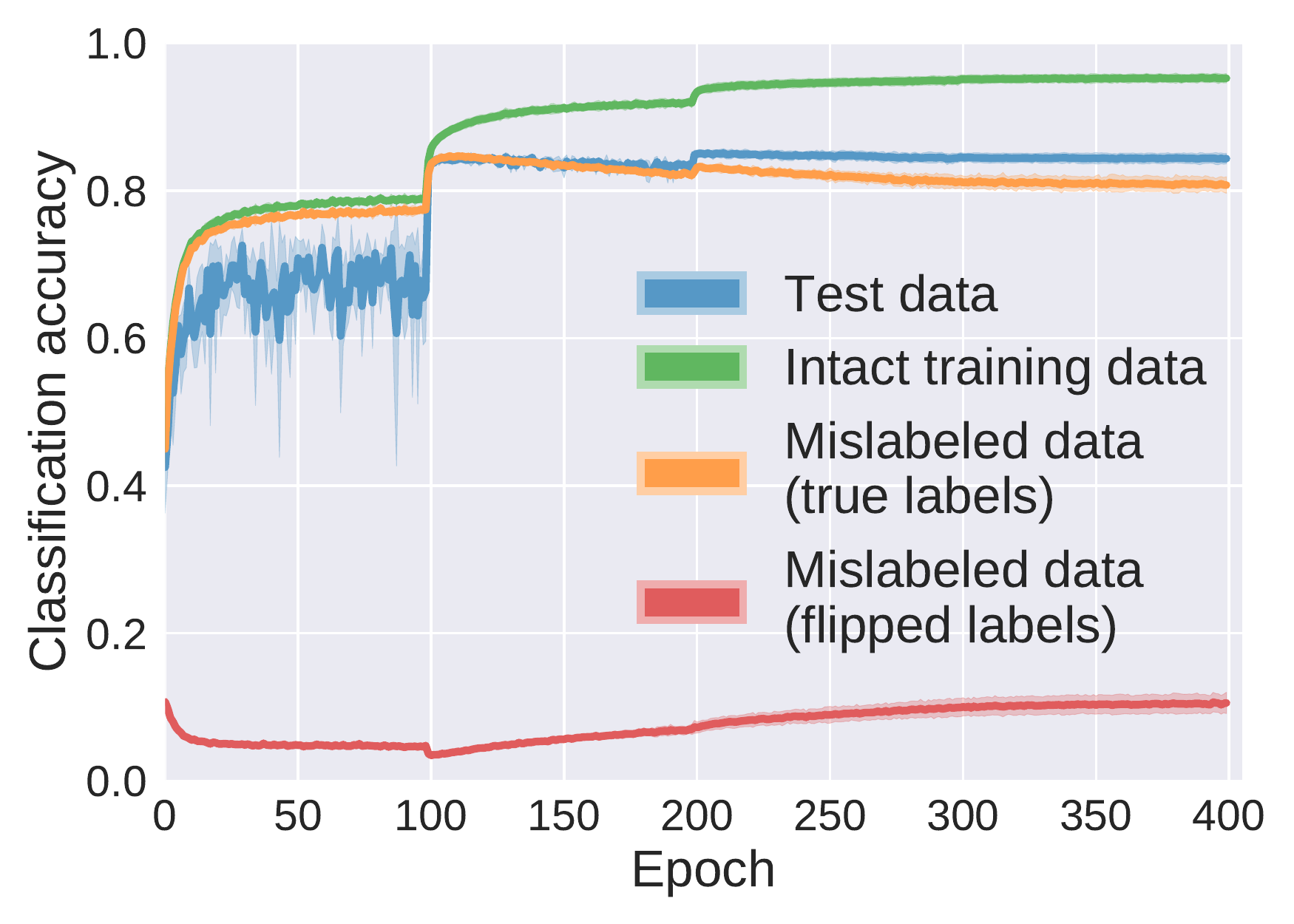}%
        \includegraphics[width=0.333\textwidth]{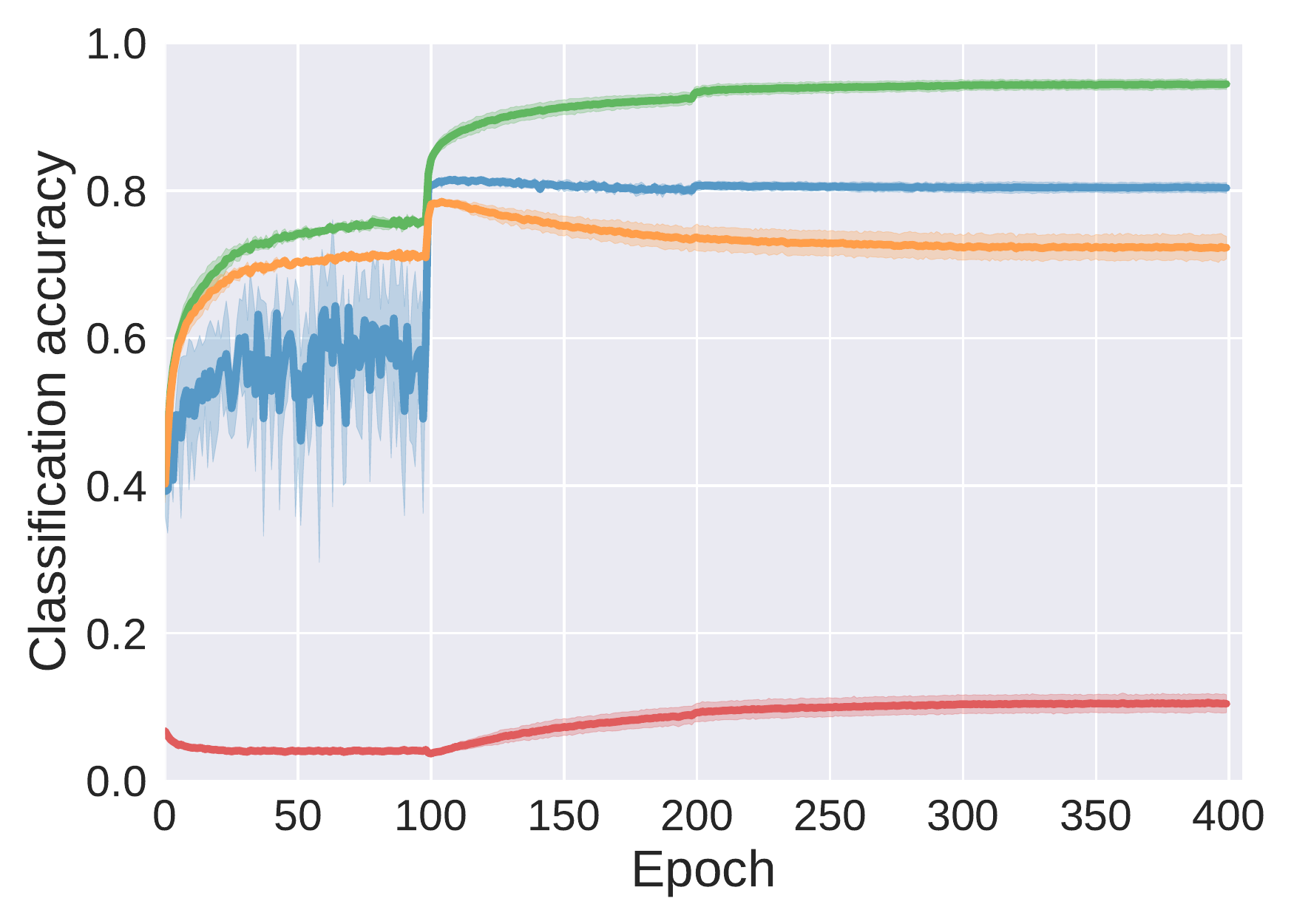}%
        \includegraphics[width=0.333\textwidth]{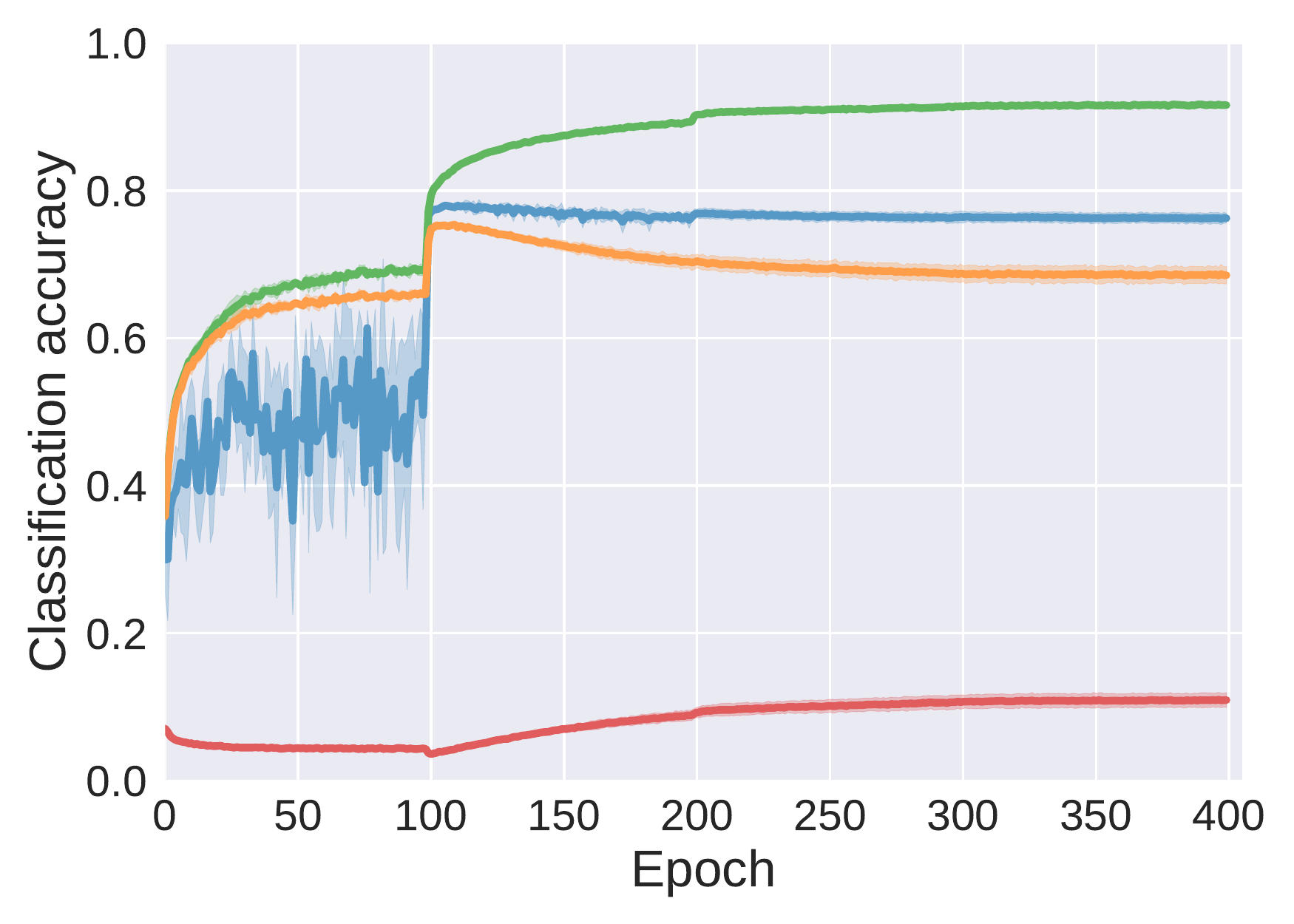}%
    \end{minipage}\\
    \begin{minipage}[c]{0.02\textwidth}\small \rotatebox{90}{CIFAR-100} \end{minipage}%
    \begin{minipage}[c]{0.98\textwidth}
        \includegraphics[width=0.333\textwidth]{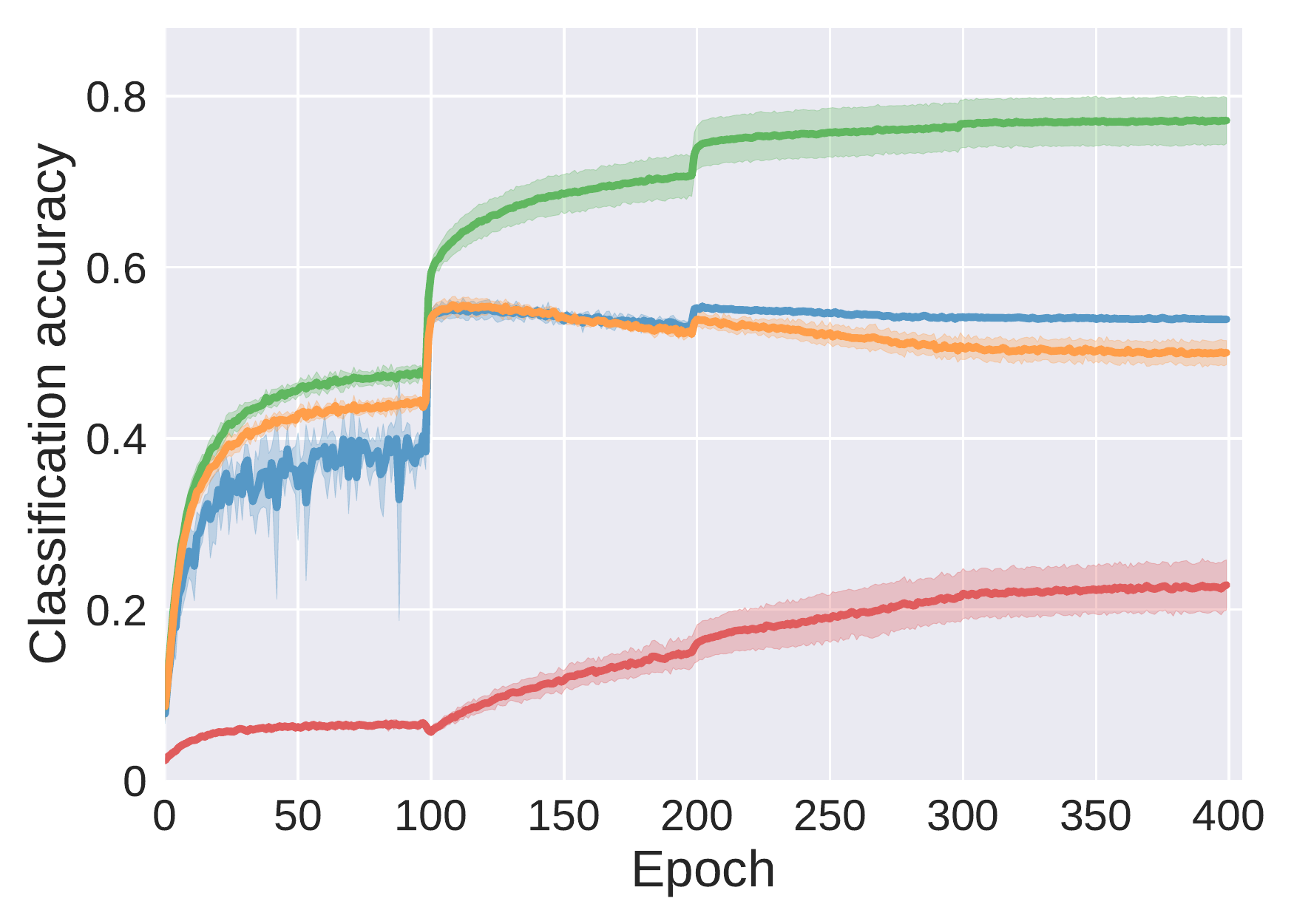}%
        \includegraphics[width=0.333\textwidth]{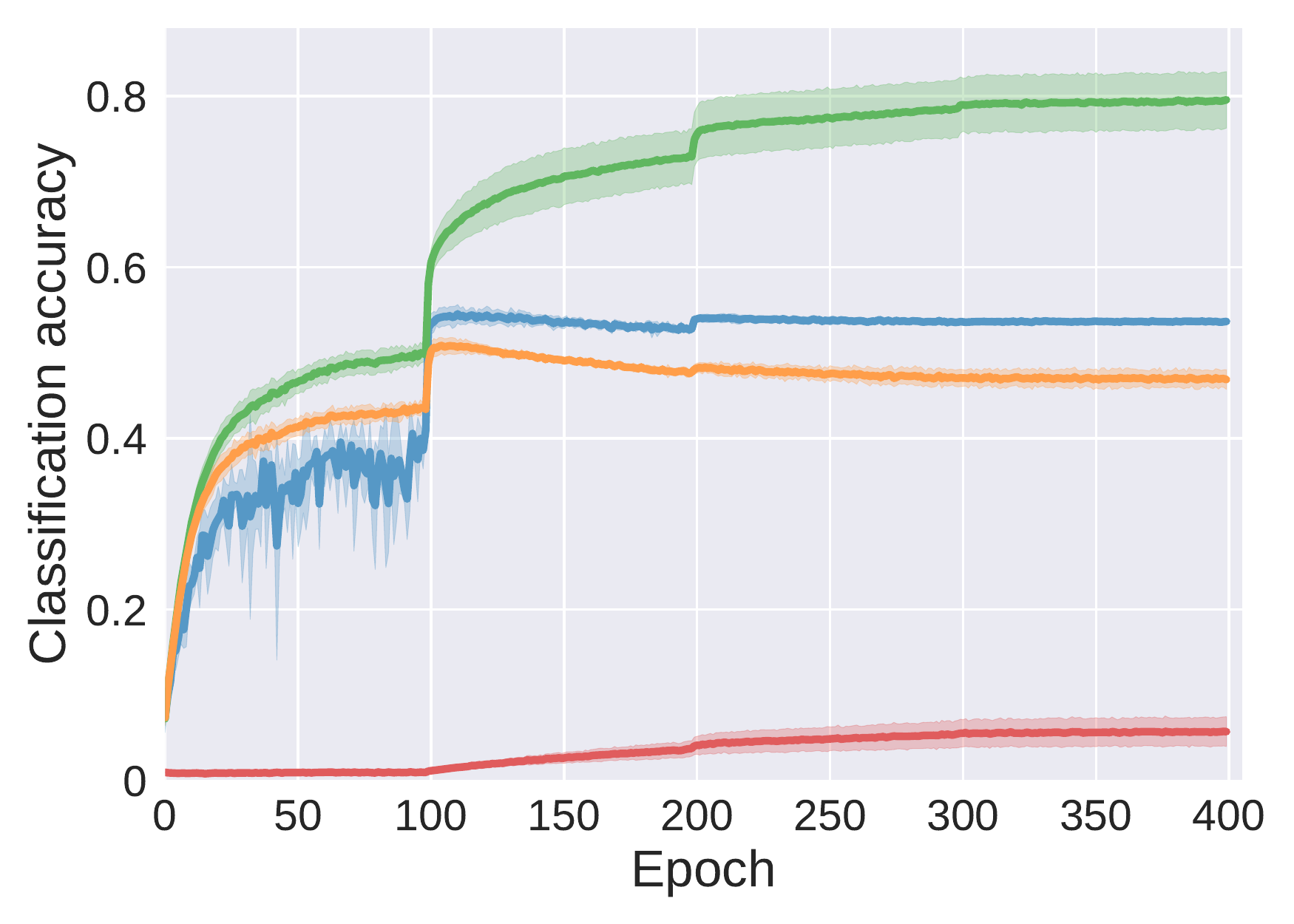}%
        \includegraphics[width=0.333\textwidth]{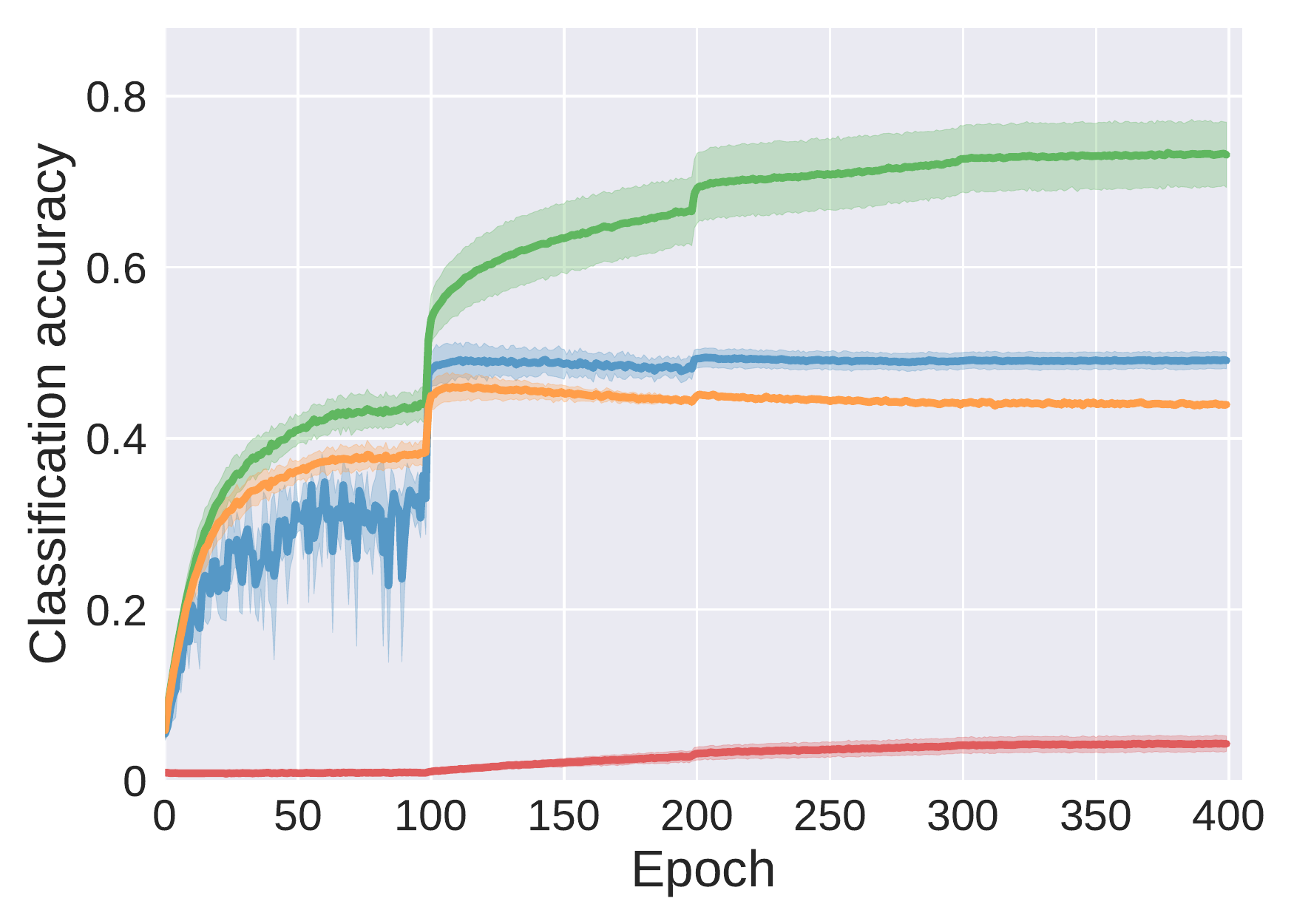}%
    \end{minipage}
    \caption{Denoising effect of DIW2-L on CIFAR-10/100 under label noise (5 trials).}
    \label{fig: diw_on_noisy}
\end{figure}

Furthermore, we trained models on CIFAR-10 under 0.4 symmetric flip, projected 64-dimensional last-layer representations of training data by \emph{t-distributed stochastic neighbor embedding}~(t-SNE) \cite{maaten2008visualizing}, and visualized the embedded data in Figure~\ref{fig: tsne}.
We can see that DIWs have more concentrated clusters of the embedded data, which implies the superiority of DIWs over IW and SIWs.

Finally, we analyzed the denoising effect of DIW2-L on CIFAR-10/100 in Figure~\ref{fig: diw_on_noisy}, by the curves of the training accuracy on the intact data, mislabeled data (evaluated by the flipped and ground-truth labels), and the test accuracy.
According to Figure~\ref{fig: diw_on_noisy}, DIW2-L can simultaneously fit the intact data and denoise the mislabeled data, so that for the mislabeled data, the flipped labels given for training correspond to much lower accuracy than the ground-truth labels withheld for training.
\section{Causal Mechanism Transfer}
\label{sec:causal}


\newtheorem{assumption}{Assumption}

\newcommand{\headline}[1]{\subsubsection{#1}}
\newcommand{\RSF}{{\bf{F}}}
\newcommand{\bj}{{\bf j}}
\newcommand{\bbN}{\mathbb{N}}
\newcommand{\cbarhR}{\check{R}}
\newcommand{\hRSF}{\hat{\RSF}}
\newcommand{\Var}{\mathbb{V}}
\newcommand{\Qsp}{\mathcal{Q}}
\newcommand{\NLICA}{\mathrm{ICA}}
\newcommand{\Datakk}[1]{\cD_{#1}}
\newcommand{\Datak}{\Datakk{k}}
\newcommand{\DataTar}{\cD_\te}
\newcommand{\augSi}{\bar{\bf{s}}_{\bj}}
\newcommand{\augZi}{\bar \bz_{\bj}}
\newcommand{\Unif}{\mathrm{Unif}}
\providecommand{\annot}[2]{\underbrace{#1}_{\text{#2}}}
\newcommand{\Rademacher}{\mathfrak{R}}
\newcommand{\lossBound}{B_\ell}
\newcommand{\qBound}{B_q}
\newcommand{\bztrikk}[1]{\bz^{\tr}_{#1,i}}
\newcommand{\bztrik}{\bztrikk{k}}
\newcommand{\bztej}{\bz^\te_j}
\newcommand{\runtrikk}[1]{_{i=1}^{n_{\tr,#1}}}
\newcommand{\runtrik}{\runtrikk{k}}
\newcommand{\runtej}{_{j=1}^{n_\te}}
\newcommand{\Stej}{S^\te_{j}}
\newcommand{\Strik}{S^\tr_{k,i}}
\newcommand{\logisticLoss}{\ell_\mathrm{log}}
\newcommand{\vPerm}{\varsigma}

In the \emph{full-distribution shift} scenario, what can we use as a reasonable basis for transfer learning? In other words, what is a plausible \emph{transfer assumption} that allows intricate shifts in the joint distributions between the source domains and the target domain?
In this section, we explore the possibility of a causality-based foundation for transfer learning.
Our motivation is simple: if we have different domains with the same causal mechanism that generates the data, it is plausible to think that one can exploit such a commonality of domains for transfer learning.
However, the natural question is: how can we exploit such a commonality?
We introduce the result of investigating this question, a novel method and foundation of transfer learning called \emph{causal mechanism transfer} \cite{TeshimaFewshot2020}.

\subsection{Motivation and Problem Setup}

The problem setup we consider is that of the \emph{independent component shift} scenario, a setting that emerges when the target and source distributions have a common causal mechanism behind them.

Such a common mechanism may be more conceivable in applications involving structured table data such as medical records \cite{YadavMining2018}. For example, in medical record analysis for disease risk prediction, it can be reasonable to assume that there is a pathological mechanism that is common across regions or generations, but the data distributions may vary due to the difference in cultures or lifestyles. Such a hidden structure (pathological mechanism, in this case), once estimated, may provide portable knowledge to enable transfer learning, allowing one to obtain accurate predictors in under-investigated regions or for new generations.

\subsubsection{Base Problem Setup}
In this section, we consider a \emph{multi-source few-shot supervised regression} transfer learning problem.
Therefore, the output space is \(\cY\coloneqq\bR\) (see Section~\ref{sec:formulate:osl}). For notational brevity, we denote $\bx$ and $y$ jointly as \(\bz = (\bx, y) = (\bz^1, \ldots, \bz^D)\) as well as \(\cZ = \bR^D\), where \(D = d_\mathrm{in} + 1\).

Then, with some abuse of notation, we treat the loss function \(\ell\) as a function over \(\cF \times \cZ\) and write the true risk and empirical risk in \eqref{eq: risk} and \eqref{eq:empirical-risk} as
\begin{align*}
R(f) &= \bE_{(\bz^\te)\sim p_\te(\bz)}\ell(f, \bz^\te), \\
\widehat{R}(f) &= \frac{1}{n_\tr}\sum_{i=1}^{n_\tr}\ell(f, \bz_i^\tr).
\end{align*}

Instead of a single source domain \(p_\tr(\bz)\), we assume that there are \(K \geq 2\) source domains, and we have access to independent samples from these domains as \(\{\Datak\}_{k \in [K]}\), where each \(\Datak = \{\bztrik\}\runtrik\) is a set of i.i.d.~samples from \(p_{\tr, k}(\bz)\).

We assume that we can access small data from the target domain, \(\DataTar = \{\bztej\}\runtej\), which is a set of i.i.d.~samples from \(p_\te(\bz)\).
We assume \(n_\te, n_{\tr, k} \geq D\) for simplicity.

The goal is to learn a predictor $f\colon\cX\rightarrow\cY$ that minimizes the risk \(R(f)\).
Under this base problem setup, the question is what is the relation between the source domains and the target domain.

\begin{figure}[!t]
\begin{minipage}[c]{1.0\linewidth}
\begin{minipage}[c]{0.3\linewidth}
\includegraphics[keepaspectratio, width=0.975\textwidth]{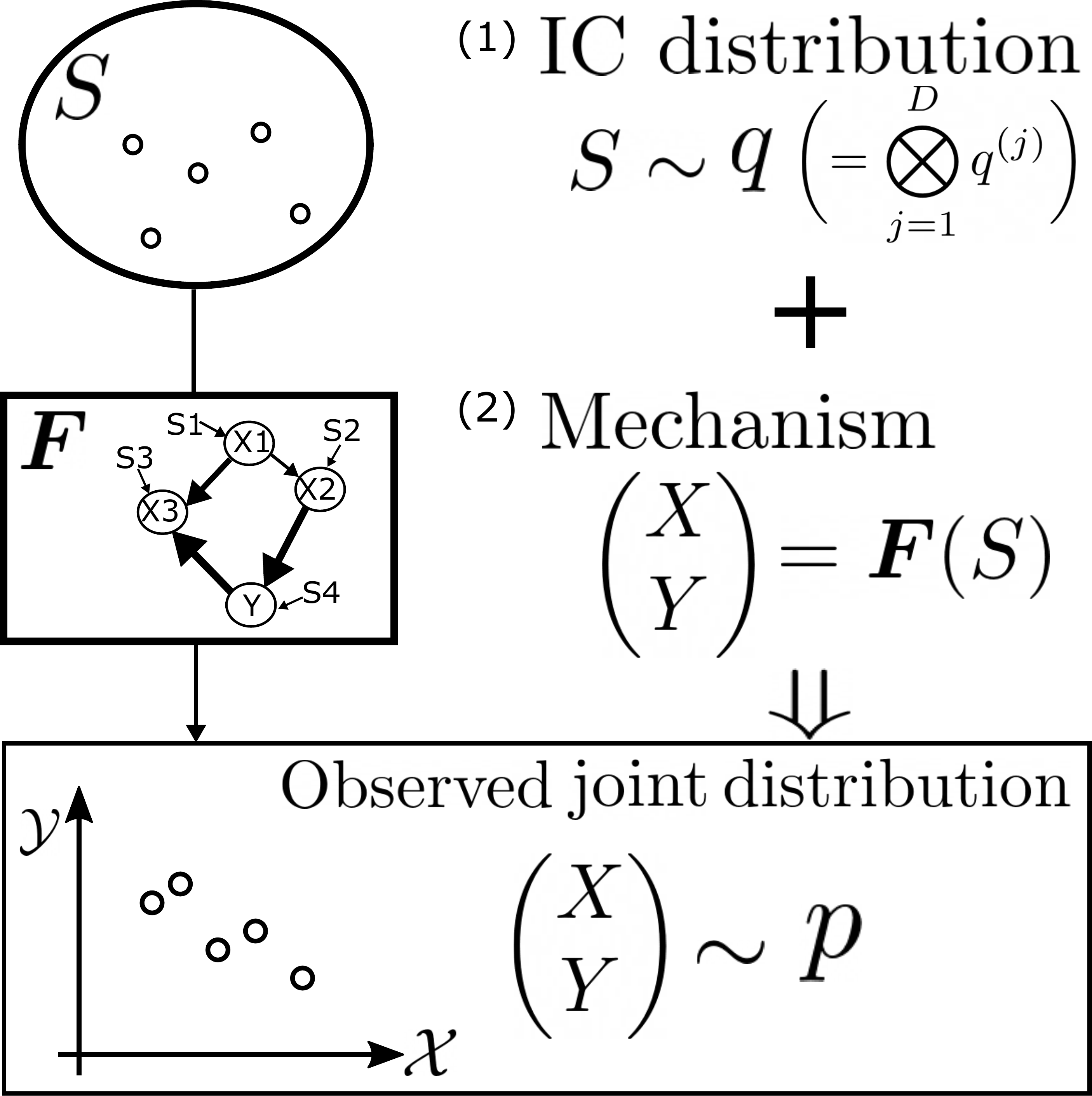}
\end{minipage}\hspace*{\fill}
\begin{minipage}[c]{0.65\linewidth}
\caption{
Nonparametric generative model of nonlinear independent component analysis.
Our meta-distributional transfer assumption is built on the model,
where there exists an invertible function \(\RSF\) representing the mechanism
to generate labeled data \((X, Y)\) from the independent components (ICs), \(S\), sampled from \(q\).
As a result, each pair \((\RSF, q)\) defines a joint density \(p\).
}
\label{fig:causal:schematic-illustration-1}
\end{minipage}\hspace*{\fill}
\end{minipage}
\end{figure}

\begin{figure}[!t]
\begin{minipage}[c]{1.0\linewidth}
\begin{minipage}[c]{1.0\linewidth}
\fbox{\includegraphics[keepaspectratio, width=0.975\textwidth]{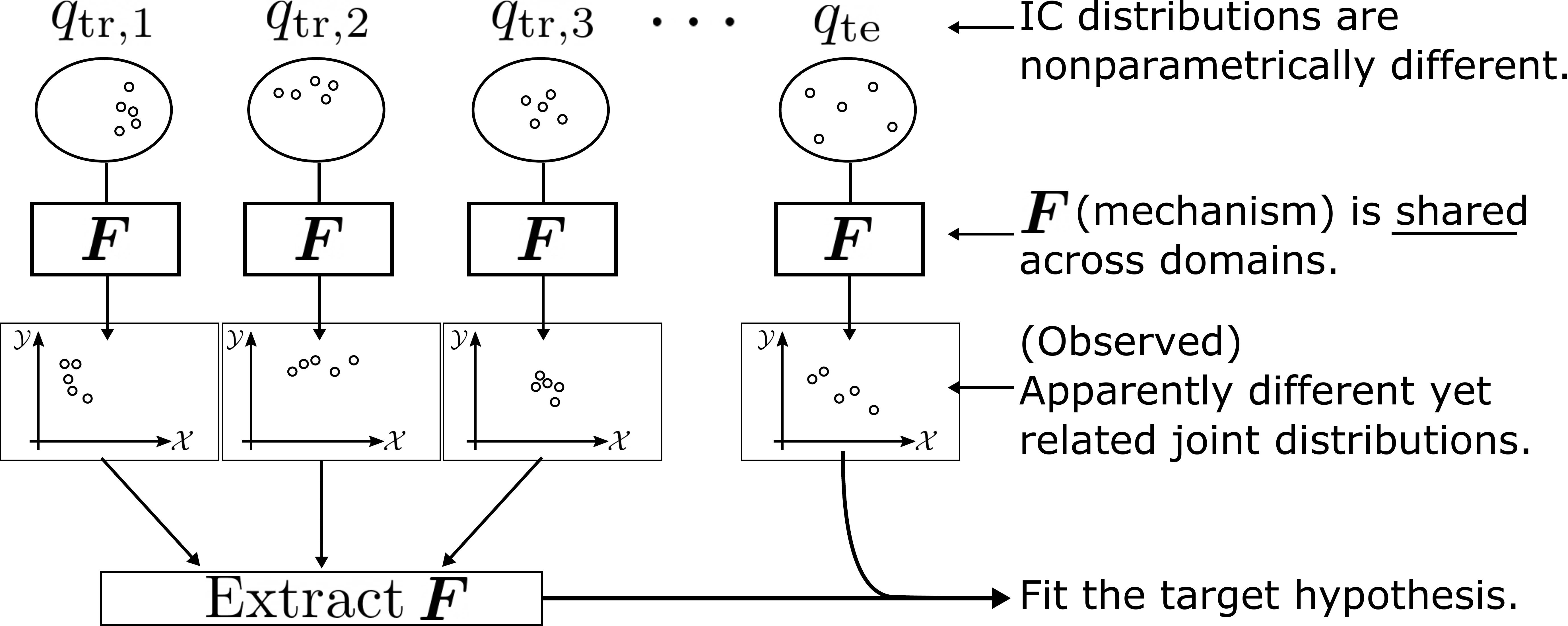}}
\vspace*{0.5\topsep}
\end{minipage}
\begin{minipage}[c]{1.0\linewidth}
\caption{
Our assumption of common generative mechanism.
By capturing the common data generation mechanism, we enable domain adaptation among seemingly very different distributions without relying on parametric assumptions.
}
\label{fig:causal:schematic-illustration-2}
\end{minipage}
\end{minipage}
\end{figure}

\subsubsection{Key Assumption}

Our key transfer assumption is that all domains follow nonlinear \emph{independent component analysis} (ICA; e.g., \cite{HyvarinenIndependent2001}) models with identical mixing functions (Figure~\ref{fig:causal:schematic-illustration-2}).
Formally, the assumption is stated as follows (illustrated in Figures~\ref{fig:causal:schematic-illustration-1} and \ref{fig:causal:schematic-illustration-2}):

\begin{assumption}[Main assumption \cite{TeshimaFewshot2020}]
Let \(\Qsp\) be the set of the density (with respect to the Lebesgue measure) of independent distributions over \(\cZ\).
We assume that there exists a set of independent component (IC) densities \(q_\te, q_{\tr, k} \in \Qsp (k \in [K])\), and a smooth invertible function \(\RSF: \bR^D \to \bR^D\) (the \emph{transformation} or \emph{mixing map}) such that \(\bztrik \sim p_{\tr, k}\) is generated by first sampling \(\Strik \sim q_{\tr, k}\) and later transforming it by
\begin{equation}\label{paper:eq:nonlinear-mixing}
\bztrik = \RSF(\Strik),
\end{equation}
and similarly \(\bztej = \RSF(\Stej), \Stej \sim q_\te\) for \(p_\te\).
\end{assumption}

A salient example of generative models which results in \eqref{paper:eq:nonlinear-mixing} is \emph{structural equation models} (SEMs; \cite{PearlCausality2009,PetersElements2017}).
More precisely, the generative model of \eqref{paper:eq:nonlinear-mixing} corresponds to the \emph{reduced form} \cite{ReissStructural2007} of a \emph{Markovian} SEM \cite{PearlCausality2009}, i.e., a form where the structural equations to determine \(Z\) from \((Z, S)\) are solved so that \(Z\) is expressed as a function of \(S\).
Such a conversion is always possible because a Markovian SEM induces an \emph{acyclic} causal graph \cite{PearlCausality2009}, and hence the structural equations can be solved by elimination of variables.
This interpretation of reduced-form SEMs as \eqref{paper:eq:nonlinear-mixing} has been exploited in methods of \emph{causal discovery}, e.g., in the linear non-Gaussian additive-noise models and their successors \cite{KanoCausal2003,ShimizuLinear2006,MontiCausal2019}.
In the case of SEMs, the key assumption of our method translates into the invariance of the causal mechanisms (expressed by the structural equations) across domains, which enables an intuitive assessment of the assumption based on prior knowledge. For instance, if all domains have the same causal mechanism and are in the same intervention state (including an intervention-free case), the modeling choice is deemed plausible.
Note that we do not estimate the original structural equations in the proposed method, but we only require estimating the reduced form which is an easier problem compared to causal discovery, e.g., \cite{GlymourReview2019,ShimizuLinear2006,MontiCausal2019}.

Under the assumption of invariant \(\RSF\), since complex changes in \(q\) are allowed, intricate shifts in the apparent distribution \(p\) can be accommodated.

\subsection{Causal Mechanism Transfer}

Here, we detail the proposed method, causal mechanism transfer (Algorithm~\ref{alg:cmt}; \cite{TeshimaFewshot2020}).
The method proceeds in three steps: \emph{estimation}, \emph{inflation}, and \emph{synthesis}, which are visually summarized in Figure~\ref{fig:causal:schematic-illustration-algorithm}.

\begin{figure}[!t]
\begin{minipage}[t]{1.0\linewidth}

\begin{minipage}[c]{1.0\linewidth}
\includegraphics[keepaspectratio,width=\textwidth,height=\textheight]{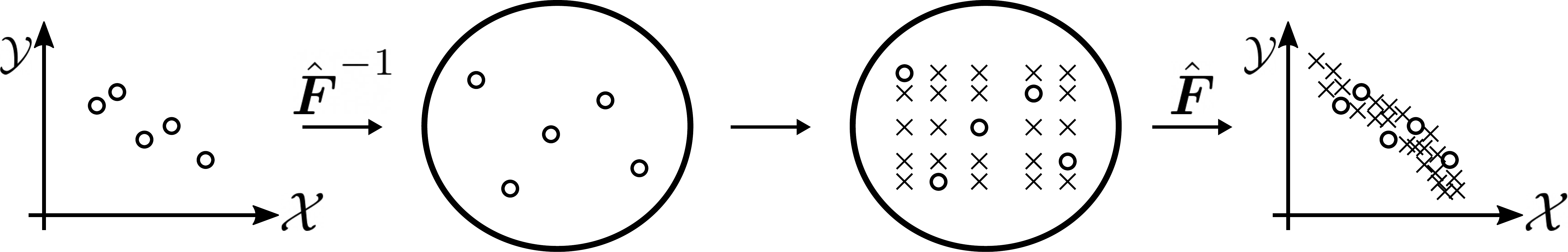}
\end{minipage}

\begin{minipage}[c]{1.0\linewidth}
\begin{minipage}[c]{0.2\linewidth}
\centering{}{\subcaption{Labeled target data}\label{fig:schema-1}}
\end{minipage}\hfill
\begin{minipage}[c]{0.2\linewidth}
\centering{}{\subcaption{Find IC}\label{fig:schema-2}}
\end{minipage}\hfill
\begin{minipage}[c]{0.2\linewidth}
\centering{}{\subcaption{Shuffle}\label{fig:schema-3}}
\end{minipage}\hfill
\begin{minipage}[c]{0.2\linewidth}
\centering{}{\subcaption{Pseudo target data}\label{fig:schema-4}}
\end{minipage}\hfill
\end{minipage}

\begin{minipage}[c]{1.0\linewidth}
\caption{
  Schematic illustration of proposed few-shot domain adaptation method after estimating the common mechanism \(\RSF\).
  With the estimated \(\hRSF\), the method augments the small target domain sample in a few steps to enhance statistical efficiency:
  (\subref{fig:schema-1}) The algorithm is given labeled target domain data.
  (\subref{fig:schema-2}) From labeled target domain data, extract the ICs.
  (\subref{fig:schema-3}) By shuffling the values, synthesize likely values of IC.
  (\subref{fig:schema-4}) From the synthesized IC, generate pseudo target data.
  The generated data is used to fit a predictor for the target domain.
  }
\label{fig:causal:schematic-illustration-algorithm}
\end{minipage}

\end{minipage}
\end{figure}

\subsubsection{Step 1: Estimation}
First, we estimate the common generative mechanism \(\RSF\), which is the sole connection between the source domains and the target domain.
The estimation can be realized by performing nonlinear ICA using the source domain data, namely via \emph{generalized contrastive learning} (GCL; \cite{HyvarinenNonlinear2019}).
GCL uses auxiliary information for training a certain binary classification function, \(r_{\hRSF,\varphi}\), equipped with a parametrized feature extractor \(\hRSF: \bR^D \to \bR^D\) and a set of functions \(\varphi = \{\varphi_j\}_{j=1}^D\), where each \(\varphi_j\) is a function from \(\bR \times \mathcal{U}\) to \(\bR\), and \(\mathcal{U}\) is some measurable space of auxiliary labels.
The auxiliary information we use in our problem setup is the domain indices, and hence \(\mathcal{U} = [K]\).
The classification function to be trained in GCL is \(r_{\hRSF,\varphi}(\bz, u) := \sum_{j=1}^D \varphi_j((\hRSF^{-1}(\bz))_j, u)\) consisting of \((\hRSF, \varphi)\), and the classification task of GCL is to classify \((\bz_{k}, k)\) as positive and \((\bz_{k}, k') (k' \neq k)\) as negative when \(\bz_{k} \in \Datak\).
This yields the following domain-contrastive learning criterion to estimate \(\RSF\):
\begin{equation*}\begin{split}
\argmin_{\hRSF, \{\varphi_j\}_{j=1}^D} \sum_{k=1}^K \frac{1}{n_{\tr,k}} \sum_{i=1}^{n_{\tr,k}} \biggl(&\logisticLoss\left(r_{\hRSF,\varphi}(\bztrik, k)\right) + \mathbb{E}_{k' \neq k}\logisticLoss\left(-r_{\hRSF,\varphi}(\bztrik, k')\right) \biggr),
\end{split}\end{equation*}
where \(\mathbb{E}_{k'\neq k}\) denotes the expectation with respect to \(k' \sim \Unif([K] \setminus \{k\})\) (``\(\Unif\)'' denotes the uniform distribution),
and \(\logisticLoss\) is the logistic loss \(\logisticLoss(m) := \log(1 + \exp(- m))\).
The trained feature extractor \(\hRSF\) is used as an estimator of \(\RSF\).
In experiments, \(\hRSF\) is implemented by invertible neural networks \cite{KingmaGlow2018}, \(\varphi_j (j \in [D])\) by multi-layer perceptron \cite{hastie2009elements}, and \(\mathbb{E}_{k'\neq k}\) is replaced by a random sampling renewed for every mini-batch.
Note that the invertible neural networks we use have been proven to be \emph{universal approximators} for smooth invertible maps, which adds a layer of theoretical justification to the modeling choice \cite{TeshimaCouplingbased2020}.

\subsubsection{Step 2: Inflation}
Second, the method uses the estimated \(\hRSF\) to perform data augmentation of the target domain data based on the knowledge transferred from the source domains.
The second step extracts and inflates the target domain ICs using the estimated \(\hRSF\).
We first extract the ICs of the target domain data by applying the inverse of \(\hRSF\) as
\begin{equation*}\begin{split}
\hat s_j = \hRSF^{-1}(\bztej).
\end{split}\end{equation*}
After the extraction, we inflate the set of IC values by taking all dimension-wise combinations of the estimated IC:
\begin{equation*}\begin{split}
\augSi = (\hat s_{j_1}^{(1)}, \ldots, \hat s_{j_D}^{(D)}), \quad {\bj} = (j_1, \ldots, j_D) \in [n_\te]^D,
\end{split}\end{equation*}
to obtain new plausible IC values \(\augSi\).
The intuitive motivation of this procedure stems from the independence of the IC distributions.
In our implementation, we used invertible neural networks \cite{KingmaGlow2018} to model the function \(\hRSF\) to enable the computation of the inverse \(\hRSF^{-1}\).

\subsubsection{Step 3: Synthesis}
The third step estimates the target risk \(R\) by the empirical distribution of the augmented data:
\begin{equation}\label{paper:eq:augmented-erm}\begin{split}
\cbarhR(f) := \frac{1}{n_\te^D}\sum_{\bj \in [n_\te]^D} \left[\ell(f, \hRSF(\augSi))\right],
\end{split}\end{equation}
and performs empirical risk minimization.
In experiments, we used a regularization term \(\Omega(\cdot)\) to control the complexity of \(\cF\) and selected
\begin{equation*}\begin{split}
\hat{f} \in \mathop{\rm arg~min}\limits_{f \in \cF} \left\{\cbarhR(f) + \Omega(f)\right\}.
\end{split}\end{equation*}
The generated hypothesis \(\hat{f}\) is then used to make predictions in the target domain.
In our experiments, we used \(\Omega(f) = \lambda \|f\|^2\), where \(\lambda > 0\) and the norm is that of the reproducing kernel Hilbert space (RKHS) which we took the subset \(\cF\) from.
Note that we may well subsample only a subset of combinations in \eqref{paper:eq:augmented-erm} to mitigate the computation cost similarly to \cite{ClemenconScalingup2016} and \cite{PapaSGD2015}.

\begin{algorithm}[t]
\caption{Causal mechanism transfer.}
\label{alg:cmt}
\begin{algorithmic}
\REQUIRE Source domain data sets \(\{\Datak\}_{k \in [K]}\), target domain data set \(\DataTar\),
nonlinear ICA algorithm \(\NLICA\), and a learning algorithm \(\mathcal{A}_{\cF}\) to fit the hypothesis class \(\cF\) of predictors.
\STATE // Step 1. Estimate the shared transformation.
\STATE \quad \(\hRSF \gets \NLICA(\Datakk{1}, \ldots, \Datakk{K})\)
\STATE // Step 2. Extract and shuffle target independent components
\STATE \quad \(\hat s_j \gets \hRSF^{-1}(\bztej), \quad (j = 1, \ldots, n_\te)\)
\STATE \quad \(\{\augSi\}_{\bj\in[n_\te]^D} \gets \mathrm{AllCombinations}(\{\hat s_j\}_{j=1}^{n_\te})\)
\STATE // Step 3. Synthesize target data and fit the predictor.
\STATE \quad \(\augZi \gets \hRSF(\augSi)\)
\STATE \quad \(\hat{f} \gets \mathcal{A}_{\cF}(\{\augZi\}_{\bj})\)
\ENSURE \(\hat{f}\): the predictor in the target domain.
\end{algorithmic}
\end{algorithm}

\subsection{Theoretical Insights}

Next, we introduce the theoretical analysis of the proposed method.
Since the proofs are involved and requires space, we omit the details here. We refer the interested readers to \cite{TeshimaFewshot2020} and its supplementary material.

\subsubsection{Complete-estimation Case: Minimum Variance Property}

First, we consider the case that \(\RSF\) has been estimated perfectly. While this is an idealistic case, the analysis provides us with the intuition that the proposed method helps the learner in terms of the \emph{variance} of the risk estimator.

\begin{theorem}[Minimum variance property of \(\cbarhR\)]\label{thm:causal:1}
Assume that \(\hRSF = \RSF\).
Then, for each \(f \in \cF\), the proposed risk estimator \(\cbarhR(f)\) is the uniformly minimum variance unbiased estimator of \(R(f)\), i.e., for any unbiased estimator \(\tilde R(f)\) of \(R(f)\),
\begin{equation*}\begin{split}
\forall q \in \Qsp, \quad \Var(\cbarhR(f)) \leq \Var(\tilde R(f)),
\end{split}\end{equation*}
where \(\Var\) denotes the variance, and also \(\bE_\te \cbarhR(f) = R(f)\) holds where \(\bE_\te\) denotes the expectation with respect to \(p_\te\).
\end{theorem}
The proof of Theorem~\ref{thm:causal:1} is obtained by rewriting \(R(f)\) as a \(D\)-variate \emph{regular statistical functional} and \(\cbarhR(f)\) as its corresponding generalized U-statistic \cite{LeeUStatistics1990}.
Theorem~\ref{thm:causal:1} implies that the proposed risk estimator can have superior statistical efficiency in terms of the variance over the ordinary empirical risk \eqref{eq:empirical-risk}, since \(\widehat{R}(f)\) is also an unbiased estimator of \(R(f)\).

\subsubsection{Incomplete-estimation Case: Excess Risk Bound}
\begin{figure}[!t]
\centering{}\includegraphics[keepaspectratio,width=0.6\textwidth]{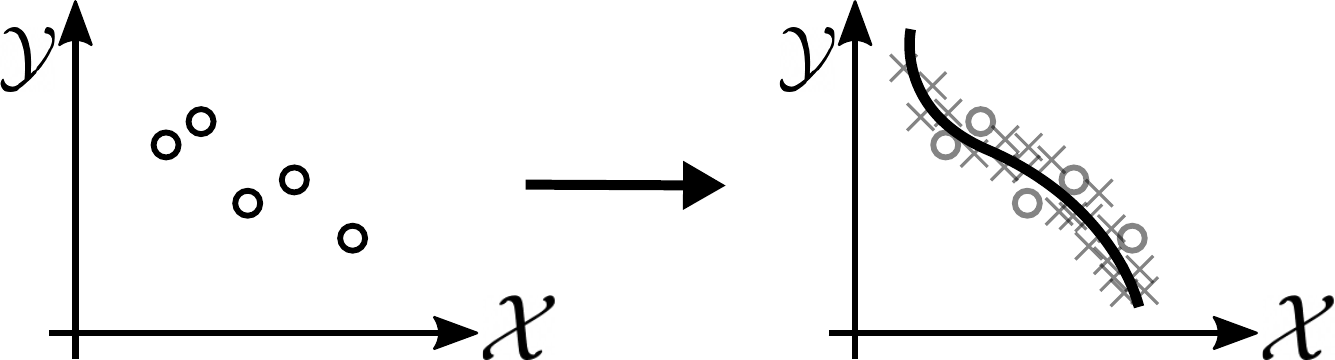}
\caption{\label{fig: sec6: inflated-data-fitting}Fitting the data inflated by the proposed method. If the inflated data appear at the appropriate locations, the increment of the data has the effect of apparent complexity reduction since one can fit a complex predictor with less fear of overfitting. On the other hand, if the estimation of \(\hRSF\) is poor, the fitting may be biased.}
\end{figure}

In real situations, one has to estimate \(\RSF\).
The following theorem characterizes the statistical gain and loss arising from the estimation error \(\RSF - \hRSF\).
The intuition is that the increased number of data points suppresses the possibility of overfitting because the hypothesis has to fit the majority of the inflated data (Figure~\ref{fig: sec6: inflated-data-fitting}) and not only the few original data, but the estimator \(\hRSF\) has to be accurate so that fitting the inflated data is meaningful.
Theorem~\ref{thm:causal:2} quantifies this consideration:
\begin{theorem}[Excess risk bound]\label{thm:causal:2}
Let \(\hat{f}\) be a minimizer of \eqref{paper:eq:augmented-erm}, and \(f^*\) be a minimizer of \(R\) (assuming both exist).
Under appropriate assumptions (see Theorem~3 in the supplementary material of \cite{TeshimaFewshot2020}),
for arbitrary \(\delta, \delta' \in (0, 1)\), we have with probability at least \(1 - (\delta + \delta')\),
\begin{equation*}\begin{split}
R(\hat{f}) - R(f^*) \leq &\annot{C\sum_{j=1}^D \|\RSF_j - \hRSF_j\|_{W^{1,1}}}{Approximation error} + \annot{4D\Rademacher(\cF) + 2D\lossBound\sqrt{\frac{\log 2/\delta}{2n_\te}}}{Estimation error} \\
&\quad+ \annot{\zeta_1(\delta', n_\te) + D\lossBound\qBound\zeta_2(\RSF - \hRSF)}{Higher order terms}.
\end{split}\end{equation*}
Here, \(\|\cdot\|_{W^{1,1}}\) is the \((1, 1)\)-Sobolev norm, and we define the effective Rademacher complexity \(\Rademacher(\cF)\) by
\begin{equation*}
\begin{split}
\Rademacher(\cF) := \frac{1}{n_\te} \bE_{\hat s}\bE_{\sigma}\left[\sup_{f \in \cF} \left|\sum_{j=1}^{n_\te} \sigma_j \bE_{S_2', \ldots, S_D'}[\tilde{\ell}(\hat s_j, S_2', \ldots, S_D')]\right|\right],
\end{split}\end{equation*}
where \(\{\sigma_j\}_{j=1}^{n_\te}\) are independent sign variables,
\(\bE_{\hat s}\) is the expectation with respect to \(\{\hat s_j\}_{j=1}^{n_\te}\),
the dummy variables \(S_2', \ldots, S_D'\) are i.i.d.\@ copies of \(\hat s_1\),
and \(\tilde{\ell}\) is defined by using the degree-\(D\) symmetric group \(\mathfrak{S}_D\) as
\begin{equation*}\begin{split}
\tilde{\ell}(s_1, \ldots, s_D) := \frac{1}{D!} \sum_{\vPerm \in \mathfrak{S}_D} \ell(f, \hRSF(s_{\vPerm(1)}^{(1)}, \ldots, s_{\vPerm(D)}^{(D)})),
\end{split}\end{equation*}
and \(\zeta_1(\delta', n)\) and \(\zeta_2(\RSF - \hRSF)\) are higher order terms.
The constants \(\qBound\) and \(\lossBound\) depend only on \(q\) and \(\ell\), respectively, while \(C\) depends only on \(\RSF, q, \ell\), and \(D\).
\end{theorem}

Details of the statement and the proof can be found in \cite{TeshimaFewshot2020}.
Note that the theorem is agnostic to how \(\hRSF\) is obtained, hence it applies to more general problem setups as long as \(\RSF\) can be estimated.
In Theorem~\ref{thm:causal:2}, the utility of the proposed method appears in the effective complexity measure.
The complexity is defined by a set of functions which are marginalized over all but one argument, resulting in mitigated dependence on the input dimensionality from exponential to linear (see Remark~3 in the supplementary material of \cite{TeshimaFewshot2020} for details).

\subsection{Experiments}
Finally, we provide the results of proof-of-concept experiments to demonstrate the effectiveness of the proposed approach.

\headline{Datasets}
We used the gasoline consumption data \cite[p.284, Example~9.5]{GreeneEconometric2012}, which is a panel data of gasoline usage in 18 of the OECD countries over 19 years.
We considered each country as a domain, and we disregarded the time-series structure and considered the data as i.i.d.\@ samples for each country in this proof-of-concept experiment.
The dataset contains four variables, all of which were log-transformed: the motor gasoline consumption per car (the predicted variable), per-capita income, the motor gasoline price, and the stock of cars per capita (the predictor variables) \cite{BaltagiGasoline1983}.
The dataset has been used in econometric analyses involving SEMs \cite{BaltagiEconometric2005}, conforming to our approach.

\headline{Compared Methods}
We compared the following transfer learning methods, all of which can be applied to regression problems.
Unless explicitly specified, the predictor class \(\cF\) was chosen to be kernel ridge regression (KRR; e.g., \cite{scholkopf2001learning}) with the same hyperparameter candidates as the proposed method.

\begin{itemize}
\item Naive baselines (\emph{SrcOnly}, \emph{TarOnly}, and \emph{S\&TV}): \emph{SrcOnly} (resp. \emph{TarOnly}) trains a predictor on the source domain data (resp. target training data) without any device.
\emph{SrcOnly} can be effective if the source domains and the target domain have highly similar distributions.
The \emph{S\&TV} baseline trains on both source and target domain data, but the LOOCV score is computed only from the target domain data.
\item \emph{TrAdaBoost}: Two-stage TrAdaBoost.R2; a boosting method tailored for few-shot regression transfer proposed in \cite{PardoeBoosting2010}.
It is an iterative method with early-stopping \cite{PardoeBoosting2010}, for which we used the leave-one-out cross-validation score on the target domain data as the criterion.
As suggested in \cite{PardoeBoosting2010}, we set the maximum number of outer loop iterations at \(30\).
The base predictor is the decision tree regressor with the maximum depth \(6\) \cite{hastie2009elements}.
Note that although TrAdaBoost does not have a clarified transfer assumption,
we compared the performance for reference.
\item \emph{IW}: Importance-weighted KRR using RuLSIF \cite{yamada2011relative} (see also Section~\ref{riwerm}).
The method directly estimates a relative joint density ratio function \(\frac{p_\te(\bz)}{\alpha p_\te(\bz) + (1 - \alpha) p_\tr(\bz)}\) for \(\alpha \in [0, 1)\),
where \(p_\tr\) is a hypothetical source distribution created by pooling all source domain data.
Following \cite{yamada2011relative}, we experimented on \(\alpha \in \{0, 0.5, 0.95\}\) and reported the results of \(0.5\) which performed the best among the three.
The regularization coefficient \(\lambda'\) was selected from \(\lambda' \in 2^{\{-10, \ldots, 10\}}\) using importance-weighted cross-validation \cite{sugiyama2012machine}.
\item \emph{GDM}: Generalized discrepancy minimization \cite{CortesAdaptation2019}.
This method performs instance-weighted training on the source domain data with the weights that minimize the \emph{generalized discrepancy} (via quadratic programming).
We selected the hyperparameters \(\lambda_r\) from \(2^{\{-10, \ldots, 10\}}\) as suggested in \cite{CortesAdaptation2019}.
The selection criterion is the performance of the trained predictor on the target training data as the method trains on the source domain data and the target unlabeled data.
\item \emph{Copula}: The non-parametric regular-vine copula method \cite{Lopez-pazSemisupervised2012}.
This method presumes to use a specific joint density estimator called regular-vine (R-vine) copulas.
Adaptation is realized in two steps: the first step estimates which components of the constructed R-vine model are different by performing two-sample tests based on maximum mean discrepancy \cite{Lopez-pazSemisupervised2012},
and the second step re-estimates the components in which a change is detected using only the target domain data.
\item \emph{LOO} (reference score): The leave-one-out cross-validated error estimate was also calculated for reference.
It is the average prediction error for a single held-out test point when the predictor is trained on the rest of the target domain data.
\end{itemize}

\textbf{Evaluation Procedure}
The prediction accuracy was measured by the mean squared error (MSE).
For each train-test split, we randomly selected one-third (6 points) of the target domain dataset as the training set and use the rest as the test set.
All experiments were repeated 10 times with different train-test splits of target domain data.

\textbf{Configuration of the Proposed Method}
We modeled \(\hRSF\) by an \(8\)-layer Glow neural network \cite{KingmaGlow2018}.
We modeled \(\varphi_j\) by a \(1\)-hidden-layer neural network with a varied number of hidden units, \(K\) output units, and the rectified linear unit activation \cite{LeCunDeep2015}.
We used its \(k\)-th output (\(k \in [K]\)) as the value for \(\varphi_j(\cdot, k)\).
For training, we use the Adam optimizer \cite{kingma15iclr} with fixed parameters \((\beta_1, \beta_2, \epsilon) = (0.9, 0.999, 10^{-8})\), fixed initial learning rate \(10^{-3}\), and the maximum number of epochs \(300\).
For further implementation details, see \cite{TeshimaFewshot2020}.
For each evaluation step, we took all combinations (with replacement) of the estimated ICs to synthesize target domain data.
After we synthesized the data, we filtered them by applying a novelty detection technique with respect to the union of source domain data, namely the \emph{one-class support vector machine} \cite{scholkopf2001learning} with the fixed parameter \(\nu=0.1\) and radial basis function (RBF; e.g., \cite{scholkopf2001learning}) kernel \(k(x, y) = \exp(- \|x - y\|^2/\gamma)\) with \(\gamma = D\).
This is because the estimated transform \(\hRSF\) is not expected to be trained well outside the union of the supports of the source distributions.
After performing the filtration, we combined the original target training data with the augmented data to ensure the original data points to be always included.

\textbf{Predictor Hypothesis Class \(\cF\)}
As the predictor model, we used the KRR with the RBF kernel.
The bandwidth \(\gamma\) was chosen by the median heuristic similarly to \cite{yamada2011relative} for simplicity.
Note that the choice of the predictor model is for the sake of comparison with the other methods tailored for KRR \cite{CortesAdaptation2019}, and that an arbitrary predictor hypothesis class and learning algorithm can be easily combined with the proposed approach.

\textbf{Hyperparameter Selection}
We performed grid-search for hyperparameter selection.
The number of hidden units for \(\varphi_j\) was chosen from \(\{10, 20\}\) and the coefficient of weight-decay from \(10^{\{-2, -1\}}\).
The \(\ell_2\)-regularization coefficient \(\lambda\) of KRR was chosen from \(\lambda \in 2^{\{-10, \ldots, 10\}}\) following \cite{CortesAdaptation2019}.
To perform hyperparameter selection as well as early-stopping, we recorded the leave-one-out cross-validation (LOOCV) mean-squared error on the target training data every \(20\) epochs and selected its minimizer.
The leave-one-out score was computed using the well-known closed-form formula instead of training the predictor for each split \cite{RifkinNotes2007}.
Note that we only used the original target domain data as the held-out set and not the synthesized data.

\subsubsection{Experimental Results} 
In Table~\ref{tbl:experiment:real-1}, we report the MSE scores normalized by that of \emph{LOO} to facilitate the comparison, similarly to \cite{CortesDomain2014}.

In many of the target domain choices, the naive baselines (\emph{SrcOnly} and \emph{S\&TV}) suffered from negative transfer, i.e., higher average MSEs than \emph{TarOnly} (in 12 out of 18 domains).
On the other hand, the proposed method performed better than \emph{TarOnly} or was more resistant to negative transfer than the other compared methods.
The performances of \emph{GDM}, \emph{Copula}, and \emph{IW} were often inferior even compared to the baseline performance of \emph{SrcAndTarValid}.
For \emph{GDM} and \emph{IW}, this was attributed to the fact that these methods presume the availability of abundant (unlabeled) target domain data, which was unavailable in the current problem setup.
For \emph{Copula}, the performance inferior to the naive baselines was possibly due to the restriction of the predictor model to its accompanied probability model \cite{Lopez-pazSemisupervised2012}.
\emph{TrAdaBoost} worked reasonably well for many but not all domains. For some domains, it suffered from negative transfer similarly to others, possibly because of the very small number of training data points. Note that the transfer assumption of TrAdaBoost has not been stated \cite{PardoeBoosting2010}, and it is not clear when the method is reliable.

The domains on which the baselines perform better than the proposed method can be explained by the following two cases: (i) easier domains allow naive baselines to perform well and (ii) some domains may have deviated \(\RSF\).
Case (i) implies that estimating \(\RSF\) is unnecessary, and hence the proposed method can be suboptimal (more likely for JPN, NLD, NOR, and SWE in Table~\ref{tbl:experiment:real-1}, where SrcOnly or S\&TV improved upon TrgOnly).
On the other hand, case (ii) implies that an approximation error was induced as in Theorem~\ref{thm:causal:2} (more likely for IRL and ITA in Table~\ref{tbl:experiment:real-1}). In this case, others also perform poorly, implying the difficulty of the problem instance.
In either case, in practice, one may well perform cross-validation to fall back into the baselines.
\begin{table*}[!ht]
\centering
\caption{\small
Results of the real-world data experiments for different choices of the target domain.
The evaluation score was MSE normalized by that of \emph{LOO} (the lower the better).
All experiments were repeated 10 times with different train-test splits of target domain data,
and the average performance is reported with the standard errors in the brackets.
The target column indicates abbreviated country names.
Bold-face indicates the best score (Prop: proposed method, TrAda: \emph{TrAdaBoost}, the numbers in the brackets of IW indicate the value of \(\alpha\)).
The proposed method often improved upon the baseline \emph{TarOnly} or was relatively more resistant to negative transfer, with notable improvements in \emph{DEU}, \emph{GBR}, and \emph{USA}.
}
\label{tbl:experiment:real-1}
\small{}
\begin{tabular}[t]{|*1{p{8mm}|}|*1{p{8mm}}||*4{p{10mm}|}*1{p{10mm}|}*5{p{10mm}|}}
  \hline
  Target & (LOO) & TarOnly & \textbf{Prop} & SrcOnly & S\&TV & TrAda & GDM & Copula & IW(.5)\\
  \hline
AUT & 1 & 5.88 (1.60) & \textbf{5.39 (1.86)} & 9.67 (0.57) & 9.84 (0.62) & 5.78 (2.15) & 31.56 (1.39) & 27.33 (0.77) & 34.06 (0.67)\\
  \hline
BEL & 1 & 10.70 (7.50) & \textbf{7.94 (2.19)} & 8.19 (0.68) & 9.48 (0.91) & 8.10 (1.88) & 89.10 (4.12) & 119.86 (2.64) & 105.68 (3.13)\\
\hline
CAN & 1 & 5.16 (1.36) & \textbf{3.84 (0.98)} & 157.74 (8.83) & 156.65 (10.69) & 51.94 (30.06) & 516.90 (4.45) & 406.91 (1.59) & 571.33 (1.60)\\
\hline
DNK & 1 & 3.26 (0.61) & \textbf{3.23 (0.63)} & 30.79 (0.93) & 28.12 (1.67) & 25.60 (13.11) & 16.84 (0.85) & 14.46 (0.79) & 21.83 (0.93)\\
\hline
FRA & 1 & 2.79 (1.10) & \textbf{1.92 (0.66)} & 4.67 (0.41) & 3.05 (0.11) & 52.65 (25.83) & 91.69 (1.34) & 156.29 (1.96) & 113.5 (1.15)\\
\hline
DEU & 1 & 16.99 (8.04) & \textbf{6.71 (1.23)} & 229.65 (9.13) & 210.59 (14.99) & 341.03 (157.80) & 739.29 (11.81) & 929.03 (4.85) & 807.88 (4.14)\\
\hline
GRC & 1 & 3.80 (2.21) & \textbf{3.55 (1.79)} & 5.30 (0.90) & 5.75 (0.68) & 11.78 (2.36) & 26.90 (1.89) & 23.05 (0.53) & 39.56 (1.70)\\
\hline
IRL & 1 & \textbf{3.05 (0.34)} & 4.35 (1.25) & 135.57 (5.64) & 12.34 (0.58) & 23.40 (17.50) & 3.84 (0.22) & 26.60 (0.59) & 5.79 (0.12)\\
\hline
ITA & 1 & \textbf{13.00 (4.15)} & 14.05 (4.81) & 35.29 (1.83) & 39.27 (2.52) & 87.34 (24.05) & 226.95 (11.14) & 343.10 (10.04) & 237.15 (6.46)\\
\hline
JPN & 1 & 10.55 (4.67) & 12.32 (4.95) & \textbf{8.10 (1.05)} & 8.38 (1.07) & 18.81 (4.59) & 95.58 (7.89) & 71.02 (5.08) & 129.3 (10.47)\\
\hline
NLD & 1 & 3.75 (0.80) & 3.87 (0.79) & \textbf{0.99 (0.06)} & 0.99 (0.05) & 9.45 (1.43) & 28.35 (1.62) & 29.53 (1.58) & 33.38 (1.63)\\
\hline
NOR & 1 & 2.70 (0.51) & 2.82 (0.73) & 1.86 (0.29) & \textbf{1.63 (0.11)} & 24.25 (12.50) & 23.36 (0.88) & 31.37 (1.17) & 27.09 (0.76)\\
\hline
ESP & 1 & 5.18 (1.05) & 6.09 (1.53) & 5.17 (1.14) & \textbf{4.29 (0.72)} & 14.85 (4.20) & 33.16 (6.99) & 152.59 (6.19) & 56.54 (2.16)\\
\hline
SWE & 1 & 6.44 (2.66) & 5.47 (2.63) & 2.48 (0.23) & \textbf{2.02 (0.21)} & 2.18 (0.25) & 15.53 (2.59) & 2706.85 (17.91) & 113.55 (1.72)\\
\hline
CHE & 1 & 3.51 (0.46) & \textbf{2.90 (0.37)} & 43.59 (1.77) & 7.48 (0.49) & 38.32 (9.03) & 8.43 (0.24) & 29.71 (0.53) & 9.33 (0.22)\\
\hline
TUR & 1 & 1.65 (0.47) & 1.06 (0.15) & 1.22 (0.18) & \textbf{0.91 (0.09)} & 2.19 (0.34) & 64.26 (5.71) & 142.84 (2.04) & 139.29 (2.41)\\
\hline
GBR & 1 & 5.95 (1.86) & \textbf{2.66 (0.57)} & 15.92 (1.02) & 10.05 (1.47) & 7.57 (5.10) & 50.04 (1.75) & 68.70 (1.25) & 69.19 (0.87)\\
\hline
USA & 1 & 4.98 (1.96) & \textbf{1.60 (0.42)} & 21.53 (3.30) & 12.28 (2.52) & 2.06 (0.47) & 308.69 (5.20) & 244.90 (1.82) & 393.45 (1.68)\\
\hline
\hline
\#Best & - & 2 & 10 & 2 & 4 & 0 & 0 & 0 & 0\\
\hline
\end{tabular}
\end{table*}



\section{Conclusion}
\label{sec:concl}
In this article, we rethought classical importance-weighting based approaches for transfer learning and explained their limitations on the two-step architecture: 
as data nowadays are becoming more and more complex and high-dimensional, the first importance estimation step can be very difficult and the resulted estimation error will propagate to the second weighted training step, degrading the prediction performance of the trained model.

To avoid the issue, we introduced two one-step solutions: one learns an importance estimator and a predictor jointly by minimizing an upper bound of the test risk,
and the other iterates weight estimation and weighted training with features for weight estimation extracted as hidden-layer outputs or loss values.
Furthermore, we presented a transfer learning method built upon the common data generation causal mechanism.
\clearpage

Future work will consider adapting the proposed methods to evolving domain shift in non-stationary environments \cite{liu2020learning,kumar2020understanding},
and study other meta-distributional concepts for capturing the intrinsic structure of the evolving domain data.
From the practical application viewpoint, we will employ the one-step importance-weighting techniques for deep reinforcement learning where the agents may face some visual changes \cite{gamrian2019transfer},
and use the novel causal mechanism transfer learning method in health data across different population groups \cite{gardner2020application}.

\subsubsection*{Acknowledgments}
NL, TF, and MS were supported by the Institute for AI and Beyond, UTokyo.
MS was also supported by JST AIP Acceleration Research Grant Number JPMJCR20U3, Japan.
TZ and TF were supported by the SPRING GX program at UTokyo.
TT was supported by RIKEN Junior Research Associate Program and Masason Foundation.

\bibliographystyle{ieeetr}
\bibliography{mybib.bib,bibs/sec6.bib,bibs/sec5.1.bib}

\end{document}